\documentclass[english]{article}
\usepackage[T1]{fontenc}
\usepackage[latin9]{inputenc}
\usepackage{geometry}
\usepackage{color}
\usepackage{babel}
\usepackage{float}
\usepackage{amsmath}
\usepackage{amsthm}
\usepackage{amssymb}
\usepackage{graphicx} 
\usepackage{algorithm}
\usepackage{tasks}
\settasks{style=itemize}
\usepackage{booktabs}
\usepackage{caption}
\usepackage{subcaption}
\usepackage{tikz}
\usetikzlibrary{arrows.meta,backgrounds,positioning}

\usepackage{bm}
\usepackage{amsmath, amsfonts, amssymb, amsthm}
\usepackage{mathrsfs, dsfont}
\usepackage[dvipsnames]{xcolor}

\usepackage{bbold}  % or try dsfont

\newcommand{\R}{\mathbb{R}}
\newcommand{\N}{\mathbb{N}}

\newcommand{\E}{\mathbb{E}}

\newcommand{\alphabet}[1][d]{{A}_{#1}}

\newcommand{\word}[1]{{\mathcolor{NavyBlue}{\mathbf{#1}}}}
\newcommand{\emptyword}{{\color{NavyBlue}\textup{\textbf{\o{}}}}}

\newcommand{\conpow}[1]{^{\otimes #1}}

\NewDocumentCommand{\sigX}{O{t} O{X}}{\mathbb{#2}_{#1}}
\NewDocumentCommand{\sig}{O{t} O{W}}{\widehat{\mathbb{#2}}_{#1}}
\NewDocumentCommand{\sigE}{O{t} O{W}}{\E[\sig[#1][#2]]}

\NewDocumentCommand{\bracketsigX}{O{t} O{X} m}{\left \langle #3, \sigX[#1][#2] \right \rangle} % not an actual error
\NewDocumentCommand{\bracketsig}{O{t} O{W} m}{\left \langle #3, \sig[#1][#2] \right \rangle}   % not an actual error
\NewDocumentCommand{\bracketsigtrunc}{O{M} O{t} O{W} m}{\left \langle #4, \sig[#2][#3]^{\leq #1} \right \rangle}   % not an actual error
\NewDocumentCommand{\bracketsigE}{O{t} O{W} m}{\left \langle #3, \sigE[#1][#2] \right \rangle} % not an actual error

\usepackage[round]{natbib}
\usepackage{enumitem}

\usepackage[unicode=true,pdfusetitle,
 bookmarks=true,bookmarksnumbered=false,bookmarksopen=false,
 breaklinks=false,pdfborder={0 0 0},pdfborderstyle={},backref=false,colorlinks=true]
 {hyperref}
\hypersetup{
 citecolor=blue, linkcolor=red}

\makeatletter

\newcommand{\noun}[1]{\textsc{#1}}
\numberwithin{equation}{section}
\theoremstyle{definition}
\newtheorem{defn}{\protect\definitionname}[section]
\theoremstyle{remark}
\newtheorem{rem}{\protect\remarkname}[section]
\theoremstyle{plain}
\newtheorem{prop}{Proposition}[section]
\theoremstyle{remark}

\theoremstyle{plain}
\newtheorem{thm}{Theorem}[section]
\theoremstyle{plain}

\theoremstyle{plain}
\newtheorem{lem}{\protect\lemmaname}[section]
\theoremstyle{plain}
\newtheorem{example}{Example}[section]

\makeatother

\providecommand{\corollaryname}{Corollary}
\providecommand{\definitionname}{Definition}
\providecommand{\lemmaname}{Lemma}
\providecommand{\notationname}{Notation}

\providecommand{\remarkname}{Remark}

\newenvironment{sqremark}{\begin{rem}}{\hfill $\blacksquare$ \end{rem}}

\newtheorem{assumption}{Assumption}

\usepackage{todonotes}
\usepackage{comment}
\usepackage{mathtools}
\usepackage{authblk}

\usepackage{algorithm}
\usepackage{algorithmic}
\title{Deep kernel hedging}

\author[1,2]{Jean-Loup Dupret\thanks{\emph{j.l.dupret@uva.nl}}}
\author[3]{Donatien Hainaut\thanks{\emph{donatien.hainaut@uclouvain.be}}}
\author[3]{Edouard Motte\thanks{\emph{Corresponding author, edouard.motte@uclouvain.be.}}}
\affil[1]{University of Amsterdam, Amsterdam School of Economics}
\affil[2]{ETH Zurich, Department of Mathematics, RiskLab}
\affil[3]{Universit\'e Catholique de Louvain, LIDAM--ISBA}
\begin{document}

\maketitle

\begin{abstract}
 We introduce a deep kernel hedging framework that combines the flexibility of deep learning with the structural inductive bias of kernel methods. The hedging functional is restricted to a reproducing kernel Hilbert space whose kernel is parameterized through a neural network embedding of the input features. The framework minimizes a regularized empirical risk under convex loss functions and can  accommodate path-dependent information through truncated time-augmented signature features.  We derive a generalized representer theorem for the joint hedging problem, reducing the empirical optimization to a finite-dimensional problem. To further reduce the computational cost associated with large kernel matrices, we develop a scalable  random Fourier feature approximation and establish convergence guarantees. The random Fourier parameters are sampled once and remain fixed throughout training, while the deep kernel adapts to  market data through the learned neural representation. We evaluate the performance of the proposed deep kernel approach on both synthetic and real data and compare it with standard kernel methods and classical deep hedging architectures. Numerical results indicate competitive and robust hedging performance, particularly in low-data regimes, which highlights the benefits of combining  expressive  neural representations with the  inductive bias of kernel methods.
\end{abstract}
\noun{Keywords:} {Hedging, Kernel methods, Deep kernel learning, Random Fourier features, Path-signatures}.

\section{Introduction}
The problem of hedging financial derivatives is a crucial topic in mathematical finance. Given a financial claim exposed to uncertain market fluctuations, the objective is to construct trading strategies that optimally reduce the associated financial risk. In idealized complete market settings, claims can be perfectly replicated by an appropriate trading strategy. However, realistic financial markets are typically incomplete due to the presence of stochastic volatility, liquidity effects or other sources of uncertainty. In such situations, perfect replication is generally impossible and hedging problems are naturally formulated as stochastic control problems aiming at minimizing a suitable loss functional associated with the profit and loss (P\&L) of the strategy. \\

The resulting optimization problems strongly depend on the modeling assumptions for the underlying risk factors and market dynamics. In particular, the optimal hedging strategy may vary significantly depending on whether the dynamics are Markovian or path-dependent, low-dimensional or high-dimensional. Except in a few specific settings with restrictive assumptions \citep{pham2009continuous}, explicit characterizations of optimal hedging strategies through stochastic control are rarely available. One well-known example where such a characterization is generally unavailable concerns the hedging under the real-world measure in stochastic volatility models \citep{jonsson2002partial,motte2024partial}. This has motivated the development of efficient numerical and data-driven approaches for solving such hedging problems. \\

Over the last few years, the rapid development of machine learning techniques has deeply impacted quantitative finance, particularly in the context of hedging problems. In this direction, the deep hedging framework introduced in \cite{buehler2019deep} proposes to parametrize hedging strategies using deep neural networks and to learn optimal trading policies directly from simulated or historical data without any modeling assumption on the chosen market dynamics. Since then, several extensions have been proposed, including robust and empirical formulations of deep hedging problems \citep{carbonneau2021equal,lutkebohmert2022robust,mikkila2023empirical,abi2025hedging,gnoatto2025deep,dupret2026deep}. Deep learning approaches offer remarkable flexibility and expressive power, especially in complex market environments where classical parametric models become inadequate. Nevertheless, neural network-based methods may suffer from instability, overfitting and substantial data requirements. This has motivated the exploration of alternative function approximation frameworks that combine flexibility with stronger structural properties in the presence of scarce data, in particular  kernel-based hedging techniques.\\

Kernel methods have emerged as a powerful alternative  in statistical learning theory. Based on the theory of reproducing kernel Hilbert spaces (RKHS), kernel methods provide flexible nonparametric models while preserving a rich functional analytic structure \citep{cristianini2000introduction,herbrich2001learning,scholkopf2002learning,shawe2004kernel,williams2006gaussian, kuurkova2008approximate}. Their mathematical tractability, convex optimization structure, and strong generalization properties have led to numerous successful applications across machine learning and statistics. In finance, kernel-based approaches have recently attracted increasing attention for pricing, hedging, portfolio optimization and risk-management problems \citep{chalup2008kernel,hu2020pricing,hainaut2025optimal,hainaut2026american}. In particular, several recent works have investigated kernel-based hedging strategies \citep{nian2018learning,ludkovski2021krighedge,cirone2025rough,futter2025kernel}, highlighting the ability of RKHS methods to efficiently capture nonlinear structures while having robustness and regularization properties. \\

Despite their appealing theoretical properties, classical kernel methods typically rely on fixed kernels whose representational capacity may be insufficient in highly complex financial environments. Deep kernel learning addresses this limitation by composing a standard kernel with a parametrized neural representation learned from data \citep{wilson2016deep,wilson2016stochastic,jiu2017nonlinear,ober2021promises}. Such approaches aim at combining the expressive power of deep architectures with the structured inductive bias and regularization properties of kernel methods. However, the deep kernel learning approaches considered in \citet{wilson2016deep,wilson2016stochastic,jiu2017nonlinear,ober2021promises} are typically formulated within a Gaussian-process (GP) framework, where the representation parameters and kernel hyperparameters are jointly learned by maximizing the GP marginal likelihood. Although statistically principled, this procedure can be computationally demanding because it involves repeatedly manipulating large kernel matrices, while the resulting nonconvex optimization may be sensitive to initialization and hyperparameter choices. \\

To retain the flexibility of deep kernel learning while improving its scalability, we introduce a novel deep kernel framework for empirical hedging together with an explicit random-feature approximation that enables efficient stochastic  gradient optimization. We consider more precisely hedging functionals belonging to an RKHS associated with a parameterized deep kernel of the form 
\[ K_\theta(x,x') =
K\bigl(\psi_\theta(x),\psi_\theta(x')\bigr),\qquad x,x' \in \mathcal{X}\subseteq \R^d, \]where \(\psi_\theta :\mathcal{X} \to \R^p\) is a learned neural representation and $K:\R^p \times \R^p \to \R$ is a standard base kernel. The inputs $x\in \mathcal{X}$ may consist of standard market-state variables or incorporate path-dependent information through truncated time-augmented signatures, thereby capturing richer temporal features of the underlying asset trajectories. We then formulate the corresponding empirical kernel hedging problem as an infinite-dimensional learning problem. Using a generalized representer theorem, we derive an equivalent finite-dimensional formulation that provides an explicit characterization of the optimal hedging strategy and establishes the well-posedness of the joint learning procedure. While this characterization is theoretically appealing, it still requires constructing and repeatedly updating large kernel matrices, which becomes computationally prohibitive as the number of training paths increases. To address this limitation, we introduce a scalable approximation scheme based on random Fourier features (RFF),   as introduced in \cite{rahimi2007random} and further analyzed in \cite{rahimi2008weighted, sutherland2015error}. The resulting finite-dimensional  RFF formulation substantially reduces the computational and memory requirements by enabling joint mini-batch training of the neural representation $\psi_\theta$ and the RFF output (linear) coefficients.  \\

Our construction is therefore naturally related to random-feature neural networks (RFNNs, also called \textit{extreme learning machine}, see \cite{huang2006extreme,sun2019random, gonon2023approximation}), in which the hidden neural weights and biases are randomly sampled and kept fixed, while only the final output linear layer is trained; see also the online RFNN extension of \cite{koppel2019parsimonious} and \cite{le2025parsimonious}. These RFNN methods allow for general activation functions and sampling distributions but do not learn an input representation $\psi_{\theta}$ before the random-feature layer.  In our framework, the random Fourier parameters are also sampled once before training and subsequently held fixed, but the resulting random features and kernel evolve throughout training through our learned latent representation $\psi_\theta$, which is optimized jointly with the output RFF coefficients. Other related deep kernel constructions based on RFFs have also been proposed in \cite{mehrkanoon2018deep,xie2019deep,mehrkanoon2019deep,fang2023end}. The approach of \cite{xie2019deep} stacks several RFF layers whose initially random Fourier parameters are subsequently trained by backpropagation, while \cite{mehrkanoon2018deep} and \cite{mehrkanoon2019deep} develop similar hybrid neural-kernel architectures with stacked RFF layers, deeper kernel blocks and pooling constructions for Least Squares Support Vector Machines (LS-SVM).  Other methods learn instead kernels directly from data through their random-feature representation \citep{sinha2016learning,zhen2020learning}; in particular, \cite{fang2023end} uses a neural generator to learn the spectral distribution from which the Fourier frequencies are sampled. In contrast, we sample the base Fourier frequencies and phases once and keep them fixed during training, thereby learning the deep kernel only through the representation map $\psi_\theta$. The RFF layer is thus used solely as a scalable finite-dimensional approximation of the deep kernel $K_\theta$ arising from our empirical hedging formulation. \\

Our main contribution is  therefore the development and analysis of a  deep kernel framework specifically built for empirical hedging. More precisely, we:

\begin{itemize}
\item introduce a deep kernel hedging architecture that combines a learned nonlinear representation of market information together with RKHS regularization, making it particularly well suited to complex financial environments with limited training data;

\item formulate the hedging problem over the RKHS induced by the learned deep kernel, derive an equivalent finite-dimensional representation through a generalized representer theorem, and establish the existence of an optimal deep kernel hedge;

\item develop a scalable random Fourier feature formulation
that enables joint mini-batch training of the neural
representation and output coefficients, along with convergence guarantees; 

\item provide an extensive numerical evaluation across different derivative products, synthetic and real datasets, and different hedging objectives, including quadratic loss and CVaR minimization. The experiments demonstrate competitive and robust hedging performance, particularly in low-data regimes where purely neural approaches suffer from instability or overfitting.

\end{itemize}

The remainder of the paper is organized as follows. In Section~\ref{sec:recall_deep_kernel}, we recall the main elements of RKHS theory and deep kernel learning. Section~\ref{sec:deep_kernel_hedging} presents the proposed deep kernel hedging formulation, derives the finite-dimensional representation of the optimization problem via a generalized representer theorem, and introduces a random Fourier feature approximation for efficient learning with convergence guarantees. Finally, Section~\ref{sec:num} contains the numerical experiments on both synthetic and real datasets, including comparisons with benchmark methods (standard kernel methods, deep hedging and RFNN architectures) and the use of signature-based features.

\section{Deep kernel learning}\label{sec:recall_deep_kernel}
In this section, we introduce the kernel learning framework used in the paper. Let $\mathcal{X} \subseteq \R^d$ be a nonempty set. A function
$K:\mathcal{X}\times\mathcal{X}\to\mathbb{R}$ is called a
\emph{kernel} on $\mathcal{X}$ if there exists a Hilbert space $\mathcal{H}_K$ and a map $\Phi : \mathcal{X} \to \mathcal{H}_K$ such that for all $x,x' \in \mathcal{X}$, we have $K(x,x') = \langle \Phi(x), \Phi(x') \rangle_{\mathcal{H}_K}$. Equivalently, a function $K$ is a kernel if and only if it is symmetric and  positive-semidefinite, that is, for every $m\in\mathbb{N}$, $x_1,\ldots,x_m\in\mathcal{X}$, 
$c_1,\ldots,c_m\in\mathbb{R}$,
\[
    \sum_{i=1}^{m}\sum_{j=1}^{m}c_i c_j K(x_i,x_j)\geq 0 \, .
\]

\begin{defn}[Reproducing Kernel Hilbert Space]
A Hilbert space $\mathcal{H}_K$ of real-valued functions on $\mathcal{X}$ equipped with an inner product $\langle \cdot, \cdot \rangle_{\mathcal{H}_K}$ is a \emph{reproducing kernel Hilbert space} (RKHS) associated with the kernel $K$ if for every $x \in \mathcal{X}$ and $f \in \mathcal{H}_K$, the point evaluation can be represented as
    \[
    f(x) = \langle f, K(\cdot, x) \rangle_{\mathcal{H}_K}.
    \]
\end{defn}
This property ensures that evaluating a function at a point is a continuous linear functional on $\mathcal{H}_K$. The Moore--Aronszajn theorem then states that every 
kernel \(K\) on \(\mathcal X\) determines a unique RKHS having \(K\) as
its reproducing kernel, and this space satisfies
    \(
    \mathcal{H}_K = \overline{\mathrm{span} \{ K(\cdot, x) : x \in \mathcal{X} \}}.
    \)
The reproducing property provides in the next sections the key link between the (potentially) infinite-dimensional space \(\mathcal H_K\) and the
finite-dimensional representations arising in kernel-based learning
problems.

\subsection{Empirical kernel regression}
We first recall the classical empirical kernel regression framework. Let $X\in\mathcal X$ denote a vector of covariates and $Y\in\mathbb R$ a target variable that we aim to approximate by a function $f:\mathcal X\to\mathbb R$. 
Kernel methods restrict $f$ to the RKHS $\mathcal{H}_K$ associated with a kernel  $K$ on $\mathcal X$. In this configuration, given a sample $\{(X_i,Y_i)\}_{i=1}^N$, a continuous convex bounded from below loss function $\mathcal{L} : \R \to \mathbb{R}$ and a regularization parameter $\lambda>0$, the goal of  empirical kernel regression is to solve
\begin{equation}\label{eq: kernel_ridge_regression}
\inf_{f \in \mathcal{H}_K} \Bigg\{ \frac{1}{N}
\sum_{i=1}^N \mathcal{L}\left( Y_i- f(X_i )\right) + \lambda \|f\|_{\mathcal{H}_K}^2 \Bigg\} \,,
\end{equation}
where $\|.\|_{\mathcal{H}_K}$ denotes the norm induced by the inner product of the Hilbert space $\mathcal{H}_K$, i.e. $\|.\|_{\mathcal{H}_K}=\sqrt{\langle.,. \rangle_{\mathcal{H}_K}}$. The term $\lambda\|f\|_{\mathcal{H}_K}^{2}$ controls the complexity of the estimator and makes the objective coercive and strictly convex. In particular, for every $\lambda>0$, problem \eqref{eq: kernel_ridge_regression} has a unique solution in $\mathcal{H}_K$, see also Lemma \ref{lem: existence and uniqueness}.
By the representer theorem (\cite{scholkopf2001generalized} and Theorem \ref{thm:representer_thm_inner_problem} below), the solution has the form
\begin{equation} \label{eq:represent}
f^\star_N(\cdot) = \sum_{i=1}^N \alpha^\star_i K(\cdot, {{X}}_i ) , 
\end{equation}
for some $\boldsymbol{\alpha}^\star=(\alpha^\star_1,\ldots,\alpha^\star_N)^\top\in\mathbb{R}^N$. Define the response vector $\mathbf{Y}\in\mathbb{R}^N$ and Gram matrix
$\mathbf{K}\in\mathbb{R}^{N\times N}$ by
\[
    \mathbf{Y}:=(Y_1,\ldots,Y_N)^\top,
    \qquad
    \mathbf{K}_{ij}:=K(X_i,X_j).
\]
The empirical kernel regression can then be reformulated as 
\begin{equation}\label{eq: KRR_formulation_min_alpha}
\min_{\boldsymbol{\alpha} \in \mathbb{R}^N} \Bigg\{ \frac{1}{N} \sum_{i=1}^N  \mathcal{L}\Big(Y_i- \sum_{j=1}^N \alpha_j K({{X}}_{j},X_i) \Big) +\lambda\boldsymbol{\alpha}^{\top}\mathbf{K}\boldsymbol{\alpha} \Bigg\} .
\end{equation}
For a general loss function, the optimal coefficient vector $\boldsymbol{\alpha}^\star$ solving \eqref{eq: KRR_formulation_min_alpha} does not generally admit a closed-form expression. However, for the squared loss $\mathcal{L}(u)=u^2$, the problem reduces to the quadratic optimization 
\begin{equation*}
\underset{\boldsymbol{\alpha}\in\mathbb{R}^{N}}{\operatorname{min}}
\, \Big\{ \frac{1}{N}
\|\mathbf{Y}-\mathbf{K}\boldsymbol{\alpha}\|_2^2
+\lambda\boldsymbol{\alpha}^{\top}\mathbf{K}\boldsymbol{\alpha} \Big\}.
\end{equation*}
Since $\mathbf{K}$ is positive semidefinite and $\lambda>0$, the matrix
$\mathbf{K}+\lambda N I_N$ is positive definite. A convenient coefficient
vector representing the unique minimizer in $\mathcal{H}_K$ for the square loss is then
\[
\boldsymbol{\alpha}^\star
    =(\mathbf{K}+\lambda N I_N)^{-1}\mathbf{Y},
\quad \text{or equivalently} \quad
    (\mathbf{K}+\lambda N I_N)\boldsymbol{\alpha}^\star
    =\mathbf{Y}.
\]

\subsection{Deep kernel regression}
While classical kernel ridge regression relies on a fixed kernel $K$, its performance is strongly influenced by the choice of kernel and its hyperparameters. In some cases, a fixed kernel may fail to capture the relevant structure of the data. To address this limitation, we consider a deep kernel learning approach, initially introduced by \cite{wilson2016deep}. The main idea is to combine kernel methods with deep learning by introducing a parametric feature map 
\[
\psi_\theta : \mathcal{X} \to \mathbb{R}^p, 
\]
where $\psi_\theta$ is typically a neural network with parameters $\theta\in \Theta$. Instead of applying the kernel directly to input features, we define a parametrized function of the form
\begin{equation}\label{eq:def_deep_kernel}
    K_\theta(x, x') := K(\psi_\theta(x), \psi_\theta(x')),\quad  x, x' \in \mathcal{X},
\end{equation}
where $K:\R^p \times \R^p \to \R$ is a standard kernel with feature map $\Phi$. The function $K_\theta$ is thus itself a symmetric positive-semidefinite kernel with feature map $x \mapsto \Phi(\psi_\theta(x))$. Therefore, the neural network $\psi_\theta$ learns a representation of the input space, while the kernel measures similarity in the resulting latent space. In particular, when $\psi_\theta$ reduces to (or learns) the identity function, we retrieve the classical kernel framework.  Thus, given the dataset $\{(X_i,Y_i)\}_{i=1}^N$, deep kernel empirical regression jointly selects the representation parameter \(\theta\) and a regression function in \(\mathcal H_{K_\theta}\) solving
\begin{equation} \label{eq:first_RFF}
\inf_{\theta\in\Theta}\;
\inf_{f\in\mathcal H_{K_\theta}}
\left\{ \frac{1}{N}
\sum_{i=1}^N
\mathcal{L}\bigl(Y_i-f(X_i)\bigr)
+
\lambda\lVert f\rVert_{\mathcal H_{K_\theta}}^2
\right\},
\qquad \lambda>0.
\end{equation}
For every fixed \(\theta\), the representer theorem and equation \eqref{eq:represent} imply that  the unique optimal regression function can be written as
\[
f^\star_{N,\theta}(\,\cdot\,)
=
\sum_{i=1}^N
\alpha^\star_{\theta,i} \, K_\theta(\,\cdot\,,X_i).
\]
for some coefficients $\boldsymbol{\alpha}^\star_\theta = (\alpha^\star_{\theta,1}, \ldots, \alpha^\star_{\theta,N})^\top$.\ The infinite-dimensional optimization problem hence reduces to
\begin{equation}\label{eq: deep_kernel_ridge_problem}
\inf_{\theta\in\Theta}\;
\min_{\boldsymbol{\alpha}\in\mathbb{R}^N}
\left\{ \frac{1}{N}
\sum_{i=1}^N 
\mathcal{L}\Big(
Y_i
-
\sum_{j=1}^N
\alpha_{j} \, K_\theta(X_i,X_{j})
\Big)
+
\lambda \boldsymbol{\alpha} ^\top\mathbf K_\theta \boldsymbol{\alpha}
\right\},
\end{equation}
where
\(
(\mathbf K_\theta)_{ij}
=
K_\theta(X_i,X_{j}) \). For fixed $\theta\in \Theta$ and squared loss $\mathcal{L}(u) = u^2$, the optimal coefficients $\boldsymbol{\alpha}^\star_\theta $ are again given by
\begin{equation} \label{eq: opt_alpha}
\boldsymbol{\alpha}^\star_\theta
=
(\mathbf K_\theta+\lambda N I_N)^{-1}\mathbf Y \, .
\end{equation}
The inverse is well defined because \(\mathbf K_\theta\) is positive semidefinite and \(\lambda>0\). We now extend this general deep kernel learning framework to the specific construction of hedging strategies.

\section{Deep kernel hedging}\label{sec:deep_kernel_hedging}

In this section, we formulate the deep kernel hedging problem under the empirical measure. Let \(S=(S_t)_{t\in[0,T]}\) denote the price process
of a traded risky asset, let \(B=(B_t)_{t\in[0,T]}\) be a strictly positive
cash account with \(B_0=1\), and let \(Y\) be a terminal liability payable
at time \(T\). Trading takes place at the discrete dates
\[
0=t_0<t_1<\cdots<t_n=T.
\]

For \(k=0,\ldots,n-1\), let \(\xi_{t_k}\) and \(\eta_{t_k}\) denote,
respectively, the numbers of units held in the risky asset and in the cash
account over the interval \([t_k,t_{k+1})\).  These holdings are assumed to
be chosen using only the information available at time \(t_k\). In
particular, if \((\mathcal F_t)_{t\in[0,T]}\) denotes the market
filtration, then \(\xi_{t_k}\) and \(\eta_{t_k}\) are
\(\mathcal F_{t_k}\)-measurable. Immediately after rebalancing at time \(t_k\), the portfolio value is
\begin{equation*}
    V_{t_k} = \xi_{t_k} S_{t_k} + \eta_{t_k} B_{t_k}, 
    \qquad k = 0,\dots,n-1,
\end{equation*}
with $V_0 = v_0 \in \mathbb{R}$ denoting the initial value of the hedging portfolio.

\begin{defn}[Self-financing property]\label{def: self-financing}
    The hedging strategy $(\xi_{t_k},\eta_{t_k})_{k=0}^{n-1}$ is self-financing if 
    \begin{equation}\label{eq: self-financing}
        V_{t_{k+1}}-V_{t_k}=\xi_{t_k} \left(S_{t_{k+1}}-S_{t_k}\right) + \eta_{t_k} \left(B_{t_{k+1}}-B_{t_k}\right), \qquad k=0,...,n-1.
    \end{equation}
    This condition means that all changes in the portfolio value are solely due to the gains and losses of the assets already held in the portfolio. In particular, no external capital is injected into or withdrawn from the strategy.
\end{defn}

Throughout this paper, we consider hedging strategies satisfying the
self-financing condition of Definition~\ref{def: self-financing}. In this case, the position in the cash account can be expressed in terms of the position in the risky asset as
\[
  \eta_{t_k}
    =
    \frac{V_{t_k}-\xi_{t_k}S_{t_k}}{B_{t_k}},
    \qquad k=0,\ldots,n-1.
\]
Hence, the self-financing condition \eqref{eq: self-financing} implies
\begin{equation*}
    V_{t_k}
    =
    B_{t_k}
    \left[
        v_0
        +
        \sum_{j=0}^{k-1}
        \xi_{t_j}
        \left(
            \frac{S_{t_{j+1}}}{B_{t_{j+1}}}
            -
            \frac{S_{t_j}}{B_{t_j}}
        \right)
    \right],
    \qquad k=1,\ldots,n.
\end{equation*}
In particular, the terminal value of the hedging portfolio is
\begin{equation} \label{eq:value_self_financing_port_at_k2}
    V_T
    =
    B_T
    \left[
        v_0
        +
        \sum_{k=0}^{n-1}
        \xi_{t_k}
        \left(
            \frac{S_{t_{k+1}}}{B_{t_{k+1}}}
            -
            \frac{S_{t_k}}{B_{t_k}}
        \right)
    \right].
\end{equation}

Let \(X_{t_k}\in \mathcal{X}\) be a vector of features observable at time
\(t_k\). We consider feedback strategies of the form
\begin{equation*}
\xi_{t_k}
=
\phi(X_{t_k}),
\qquad k=0,\ldots,n-1,
\end{equation*}
where \(\phi: \mathcal{X}\to\mathbb R\) is measurable. The input feature vector $X_{t_k}$ can include time-augmented market variables such as the underlying price $S_{t_k}$, the cash account $B_{t_k}$, asset volatility, etc. Moreover, in more general non-Markovian settings where the hedging strategy may depend on the path of the market features, the learning inputs may include truncated time-augmented signatures of the market features. Path-signatures, originally introduced by \cite{chen1957integration} as sequences of iterated integrals, provide a compact and tractable way to encode path-dependent information. See \cite{abi2025hedging, cirone2025rough} for applications of signatures to the learning of optimal hedging strategies via machine learning approaches. For further details on path-signatures, we refer to Appendix \ref{appendix:sig} and the references therein.\\ 

Our objective is now to learn the optimal hedging functional $\phi$ from sample data using a deep kernel-based approach. We assume access to a dataset of $N$ sample paths observed at the trading dates, defined as
\[
\mathcal D_N
=
\left\{
\left(
\left(S_{t_k}^i,B_{t_k}^i,X_{t_k}^i\right)_{k=0}^{n}, Y^i
\right)
\right\}_{i=1}^N .
\]
Each sample path $i$ therefore contains the underlying trajectory, the cash account and input features. For each path $i$, we define the gain associated with one
unit of the risky asset held over \([t_k,t_{k+1})\) by
\begin{equation*}
G_{t_{k+1}}^i
:=
B_T^i
\left(
\frac{S_{t_{k+1}}^i}{B_{t_{k+1}}^i}
-
\frac{S_{t_k}^i}{B_{t_k}^i}
\right).
\end{equation*}
It follows from \eqref{eq:value_self_financing_port_at_k2} that the terminal hedging error (or negative P\&L) $\mathcal{E}_i$ 
generated by \(\phi\) along path \(i\) is the random variable
\begin{equation} \label{eq:hedg_error}
\mathcal{E}_i(\phi; v_0) := Y^i-V_T^i
=
Y^i-B_T^iv_0
-
\sum_{k=0}^{n-1}
\phi(X_{t_k}^i)G_{t_{k+1}}^i \, , \end{equation} which we simply write $\mathcal{E}_i(\phi)$ when the initial capital $v_0$ is clear from the context. Let \(\mathcal L:\mathbb R\to \mathbb{R}\) be a continuous convex bounded from below loss function, and let \(\mathcal M\) be a prescribed class of admissible measurable hedging functions. The empirical hedging problem can then be written
\begin{equation*}
\inf_{\phi\in\mathcal M}
\frac{1}{N}\sum_{i=1}^{N}
\mathcal{L}\left( \mathcal{E}_i(\phi) \right) = \inf_{\phi\in\mathcal M}
\frac{1}{N}\sum_{i=1}^{N}
\mathcal{L}\left(
Y^i
-
B_T^iv_0
-
\sum_{k=0}^{n-1}
\phi({X}^i_{t_k})
\,G^i_{t_{k+1}}
\right).
\end{equation*}
The choice
\(
\mathcal L(u)=u^2
\)
gives the classical quadratic hedging problem, see Section \ref{subsec:QH}. Other convex losses can be
used to emphasize different properties of the terminal hedging error $\mathcal{E}_i(\phi)$, such as the CVaR loss \eqref{eq:deep_kernel_cvar_hedging} in Section \ref{subsec:CVaR hedging}.
\\

We now restrict the space of measurable functionals $\mathcal{M}$ to RKHS. In this sense, for $\theta\in \Theta$, let $\mathcal{H}_{K_{\theta}}$ be the RKHS associated with the kernel $K_\theta$ on $\mathcal{X}$. The regularized deep kernel hedging problem is
\begin{equation}\label{eq: deep_kernel_hedging}
\inf_{\theta\in \Theta, \, \phi\in \mathcal{H}_{K_\theta}} 
\frac{1}{N} \sum_{i=1}^N
\mathcal{L}\left(
Y^i - B_T^iv_0 - \sum_{k=0}^{n-1} \phi({X}^i_{t_k}) \, G^i_{t_{k+1}}
\right)
+
\lambda \|\phi\|_{\mathcal{H}_{K_\theta}}^2 ,
\end{equation}
with $\lambda>0$. For fixed $\theta\in\Theta$, define the functional 

\begin{equation*} 
    \mathcal{J}_{\theta}(\phi) :=\frac{1}{N} \sum_{i=1}^N
\mathcal{L}\left(
Y^i - B_T^iv_0 - \sum_{k=0}^{n-1} \phi({X}^i_{t_k}) \, G^i_{t_{k+1}}
\right)
+
\lambda \|\phi\|_{\mathcal{H}_{K_\theta}}^2 , \quad \phi \in \mathcal{H}_{K_\theta}, 
\end{equation*}

such that the hedging problem can be rewritten as 
\begin{equation} \label{eq: inner_outer}
    \inf_{\theta\in \Theta}\bigg{\{}~\inf_{\phi\in \mathcal{H}_{K_\theta}} \mathcal{J}_{\theta}(\phi) \bigg{\}}. 
\end{equation}

We first address in Section \ref{sec:hedging_representer} the well-posedness of the inner problem in \eqref{eq: inner_outer} for fixed $\theta \in \Theta$ and show that, using a generalized version of the representer theorem, this infinite-dimensional optimization problem admits an equivalent finite-dimensional reformulation, leading to an explicit characterization of the unique optimizer $\phi^\star$.  Then, in Section \ref{sec:existence_optimal_deep_kernel}, we reformulate the deep kernel hedging problem and establish the existence of a minimizer $\theta^\star \in \Theta$. Finally, in Section \ref{sec:rff_deep_kernel_hedging}, we develop a Random Fourier Feature approximation for deep kernel hedging. The following Table~\ref{tab:hedging_notation} summarizes the notations used in following sections.
\begin{table}[H]
\centering
\small
\renewcommand{\arraystretch}{1.2}
\caption{Notation for hedging functions. Here, $N$ is the number
of training paths and $D$ the number of RFF.}
\label{tab:hedging_notation}
\begin{tabular}{@{}lp{0.74\linewidth}@{}}
\toprule
Notation & Description \\
\midrule

$\phi$
& Generic hedging function, defining the risky-asset position through
$\xi_{t_k}=\phi(X_{t_k})$. \\

$\phi_{N,\theta}^\star$
& Optimal exact RKHS hedging function for a fixed kernel
parameter $\theta$. \\

$\phi_N^\star$ 
& A globally optimal exact deep kernel hedge $\phi_N^\star=\phi_{N,\theta^\star}^\star$, where $\theta^\star$
is also optimized. \\
$\phi^*_{D,\theta}$ & Optimal RFF hedging function for a fixed kernel
parameter $\theta$. \\
$\phi_D^\star$
& A globally optimal hedge for the RFF approximation, given by
$\phi_D^\star(x) = (\boldsymbol{\beta}_D^\star)^\top y_{\theta_D^\star}(x)$,
where both $\theta^\star_D$ and $\boldsymbol{\beta}^\star_D$ are optimized. \\

$\widehat{\phi}_D$
& RFF hedging function returned by numerical training (without optimality guarantee). \\ 
\bottomrule
\end{tabular}
\end{table}

\subsection{Representer theorem and finite-dimensional formulation}
\label{sec:hedging_representer}
We first focus on the inner problem in \eqref{eq: inner_outer} for fixed $\theta \in \Theta$.
\begin{lem}[Existence and uniqueness for the inner problem] \label{lem: existence and uniqueness}
    Suppose that \(\mathcal L:\mathbb R\to \mathbb{R}\) is continuous,
convex and bounded from below, and let \(\lambda>0\). Then, for every fixed
\(\theta\in\Theta\), the problem
\begin{equation*}
\inf_{\phi\in\mathcal H_{K_\theta}}
\mathcal J_\theta(\phi)
\end{equation*}
admits a unique minimizer
\(\phi_{N,\theta}^\star\in\mathcal H_{K_\theta}\).
\end{lem}

\begin{proof}
The proof relies on standard arguments for minimization problems in Hilbert spaces that can be found in \cite{bauschke2017correction}. Since the zero function belongs to $\mathcal H_{K_\theta}$, we have
\[
0\leq  \inf_{\phi\in\mathcal H_{K_\theta}}
\mathcal J_\theta(\phi)\leq\mathcal J_\theta(0)<\infty.
\]
Choose a minimizing sequence
$(\phi_m)_{m\geq1}\subset\mathcal H_{K_\theta}$ such that
\[
\mathcal J_\theta(\phi_m)
\leq
 \inf_{\phi\in\mathcal H_{K_\theta}}
\mathcal J_\theta(\phi) +\frac1m.
\]
Since \(\mathcal L\) is bounded from below, there exists a constant
\(C_{\mathcal L}\in\mathbb R\) such that
\(\mathcal L\geq C_{\mathcal L}\). Hence

$$
C_{\mathcal L}
+\lambda\|\phi_m\|_{\mathcal H_{K_\theta}}^2
\leq
\mathcal J_\theta(\phi_m)
\leq
\inf_{\phi\in\mathcal H_{K_\theta}}
\mathcal J_\theta(\phi)+\frac1m
\leq
\mathcal J_\theta(0)+1.
$$
Therefore,
$$
\lambda\|\phi_m\|_{\mathcal H_{K_\theta}}^2
\leq
\mathcal J_\theta(0)+1-C_{\mathcal L},
$$
and consequently
$$
\|\phi_m\|_{\mathcal H_{K_\theta}}
\leq
\sqrt{\frac{\mathcal J_\theta(0)+1-C_{\mathcal L}}{\lambda}}.
$$
Thus, \((\phi_m)_{m\geq1}\) is bounded in
\(\mathcal H_{K_\theta}\). As $\mathcal{H}_{K_\theta}$ is a Hilbert space, there exists a subsequence, still denoted by \((\phi_m)\), and some
\(\overline\phi\in\mathcal H_{K_\theta}\) such that
\[
\phi_m\rightharpoonup\overline\phi
\qquad\text{weakly in }\mathcal H_{K_\theta}.
\]

For each \(i=1,\ldots,N\), define
\[
L_i(\phi)
:=
\sum_{k=0}^{n-1}
G_{t_{k+1}}^i\,\phi(X_{t_k}^i).
\]
Because point evaluations are continuous linear functionals on an RKHS,
\(L_i\) is a continuous linear functional on
\(\mathcal H_{K_\theta}\). Hence,
\[
L_i(\phi_m)\rightarrow L_i(\overline\phi),
\qquad i=1,\ldots,N.
\]
The continuity of \(\mathcal L\) then gives
\[
\frac1N\sum_{i=1}^N
\mathcal L\left(
Y^i-B_T^iv_0-L_i(\phi_m)
\right)
\rightarrow
\frac1N\sum_{i=1}^N
\mathcal L\left(
Y^i-B_T^iv_0-L_i(\overline\phi)
\right).
\]
Moreover, the squared Hilbert-space norm is weakly lower
semicontinuous. Consequently,
\[
\mathcal J_\theta(\overline\phi)
\leq
\liminf_{m\to\infty}\mathcal J_\theta(\phi_m),
\]
so \(\overline\phi\) is a minimizer. Finally, the empirical loss is convex in \(\phi\), while
\(\lambda\|\phi\|_{\mathcal H_{K_\theta}}^2\) is strictly convex because
\(\lambda>0\). Thus, \(\mathcal J_\theta\) is strictly convex and its
minimizer is unique.
\end{proof}
We next derive a finite-dimensional representation of this minimizer \(\phi_{N,\theta}^\star\).

\begin{thm}[Representer theorem for the inner problem]\label{thm:representer_thm_inner_problem}
Fix \(\theta\in\Theta\). For each \(i=1,\ldots,N\), define
\begin{equation*}
g_{\theta,i}(\cdot)
:=
\sum_{k=0}^{n-1}
G_{t_{k+1}}^i
K_\theta\bigl(X_{t_k}^i,\cdot\bigr)
\in\mathcal H_{K_\theta},
\end{equation*}
and let the associated hedging Gram matrix
\(\mathbf Q_\theta\in\mathbb R^{N\times N}\) be given by
\begin{equation}
\label{eq:def_Q_theta}
(\mathbf Q_\theta)_{ij}
:=
\langle g_{\theta,i},g_{\theta,j}\rangle_{\mathcal H_{K_\theta}}
=
\sum_{k,\ell=0}^{n-1}
G_{t_{k+1}}^i G_{t_{\ell+1}}^j
K_\theta\bigl(X_{t_k}^i,X_{t_\ell}^j\bigr).
\end{equation} 
Moreover, define
\begin{equation}
\label{eq:def_H_v0}
\mathbf Y^{v_0}
:=
\left(
Y^1-B_T^1v_0,\ldots,Y^N-B_T^Nv_0
\right)^\top \in \R^N.
\end{equation}
Then,
\begin{align}
\min_{\phi\in\mathcal H_{K_\theta}}
\mathcal J_\theta(\phi) 
=
\min_{\bm\alpha\in\mathbb R^N}
\left\{
\frac1N
\sum_{i=1}^N
\mathcal L\left( ( \mathbf Y^{v_0} - \mathbf Q_\theta\bm\alpha)_i \right)
+
\lambda
\bm\alpha^\top
\mathbf Q_\theta
\bm\alpha
\right\}.
\label{eq:matrix_finite_dim_reformulation}
\end{align}
Moreover, the unique optimal hedging function is of the form
\begin{equation}
\label{eq:phi_alpha}
\phi_{N,\theta}^\star(\cdot)
= \sum_{i=1}^N
\alpha_{\theta,i}^\star \, 
g_{\theta,i} = 
\sum_{i=1}^N
\alpha_{\theta,i}^\star
\sum_{k=0}^{n-1}
G_{t_{k+1}}^i
K_\theta(X_{t_k}^i,\cdot),
\end{equation}
where \(\bm\alpha^\star_\theta\) is any minimizer of
\eqref{eq:matrix_finite_dim_reformulation}.
\end{thm}

\begin{proof}
For \(i=1,\ldots,N\), define
$$L_i(\phi)
:=
\sum_{k=0}^{n-1}
G_{t_{k+1}}^i\,\phi(X_{t_k}^i).$$
Since $\mathcal H_{K_\theta}$ is a RKHS, the reproducing property implies that
\begin{equation*}
L_i(\phi)
=
\langle\phi,g_{\theta,i}\rangle_{\mathcal H_{K_\theta}}.
\end{equation*} 
Let
\[
\mathcal H_{\theta,0}
:=
\operatorname{span}
\left\{
g_{\theta,1},\ldots,g_{\theta,N}
\right\}.
\]
Since \(\mathcal H_{\theta,0}\) is finite-dimensional, it is closed.
Every \(\phi\in\mathcal H_{K_\theta}\) therefore admits the orthogonal
decomposition
\[
\phi
=
\phi_\parallel+\phi_\perp,
\qquad
\phi_\parallel\in\mathcal H_{\theta,0},
\quad
\phi_\perp\in\mathcal H_{\theta,0}^\perp.
\]
Because \(g_{\theta,i}\in\mathcal H_{\theta,0}\),
\[
L_i(\phi)
=
\langle\phi,g_{\theta,i}\rangle_{\mathcal H_{K_\theta}}
=
\langle
\phi_\parallel,g_{\theta,i}
\rangle_{\mathcal H_{K_\theta}}.
\]
Furthermore,
\[
\|\phi\|_{\mathcal H_{K_\theta}}^2
=
\|\phi_\parallel\|_{\mathcal H_{K_\theta}}^2
+
\|\phi_\perp\|_{\mathcal H_{K_\theta}}^2.
\]
It follows that
\[
\mathcal J_\theta(\phi)
=
\mathcal J_\theta(\phi_\parallel)
+
\lambda
\|\phi_\perp\|_{\mathcal H_{K_\theta}}^2.
\]
Since \(\lambda>0\), every minimizer must satisfy
\(\phi_\perp=0\). Thus,
\[
\phi_{N,\theta}^\star
\in
\mathcal H_{\theta,0}.
\]

Consequently, there exists \(\bm\alpha^\star_\theta\in\mathbb R^N\) such that
\[
\phi^\star_{N,\theta}(\cdot)
=
\sum_{j=1}^N\alpha^\star_{\theta, j} g_{\theta,j}(\cdot).
\]
For such a function,
\[
L_i(\phi)
=
\sum_{j=1}^N
\alpha^\star_{\theta,j}
\langle
g_{\theta,j},g_{\theta,i}
\rangle_{\mathcal H_{K_\theta}}
=
(\mathbf Q_\theta\bm\alpha^\star_\theta )_i,
\]
and
\[
\|\phi\|_{\mathcal H_{K_\theta}}^2
=
\sum_{i,j=1}^N
\alpha^\star_{\theta, i} \alpha^\star_{\theta,j}
\langle
g_{\theta,i},g_{\theta,j}
\rangle_{\mathcal H_{K_\theta}}
=
(\bm{\alpha}^\star_{\theta})^\top\mathbf Q_\theta\bm\alpha_\theta^\star.
\]
Substitution into \(\mathcal J_\theta\) yields
\eqref{eq:matrix_finite_dim_reformulation} and
\eqref{eq:phi_alpha}.
\end{proof}
\begin{rem}
For the square loss $\mathcal{L}(u) = u^2$, one choice of optimal coefficients $\bm{\alpha}^\star_\theta$ is
\begin{equation} \label{eq:opt_alpha2}
    \bm{\alpha}^\star_\theta =(\mathbf{Q}_\theta + N\lambda I_{N})^{-1}\mathbf{Y}^{v_0}.
\end{equation}
\end{rem}

\subsection{Existence of an optimal deep kernel hedge}
We now establish existence of a solution to the joint optimization \eqref{eq: inner_outer} over
the kernel parameter $\theta \in \Theta$ and the hedging function $\phi \in \mathcal{H}_{K_\theta}$.
\label{sec:existence_optimal_deep_kernel}
\begin{assumption}\label{ass:kernel}
The following conditions hold:
\begin{align}
    &\Theta \subset \mathbb{R}^m \text{ is compact}, \tag{H1}\label{ass:kernel:H1}\\
    &\theta \mapsto K_\theta(x,x') \text{ is continuous on } \Theta,
    \qquad \forall x,x'\in\mathcal X. \tag{H2}\label{ass:kernel:H2}
\end{align}
\end{assumption}

\begin{lem}\label{lem:existence_deep_kernel_hedging_rewritten}
    Under Assumption~\ref{ass:kernel}, there exist $(\theta^\star,\boldsymbol{\alpha}_{\theta^\star}^\star)\in \Theta\times \mathbb{R}^N$ solution to
    \begin{equation}\label{eq: deep_kernel_hedging_rewritten}
    \inf_{\theta\in\Theta,\;\bm{\alpha}\in\mathbb{R}^N}
    \left\{
    \frac{1}{N}
    \sum_{i=1}^N
    \mathcal{L}\left(
    (\mathbf{Y}^{v_0}
    -
    \mathbf{Q}_\theta\bm{\alpha})_i
    \right)
    +
    \lambda
    \bm{\alpha}^{\top}\mathbf{Q}_\theta\bm{\alpha}
    \right\}. 
\end{equation}
\end{lem}
\begin{proof}
    The proof is deferred to Appendix \ref{proof1}. 
\end{proof}

\begin{thm}[Existence of an optimal deep kernel hedge]
\label{thm:existence_thm_hedging}
Under Assumption~\ref{ass:kernel}, the deep kernel hedging problem
\eqref{eq: deep_kernel_hedging} admits a global minimizer.
 More precisely, there exist
\[
\theta^\star\in\Theta
\qquad\text{and}\qquad
\phi^\star\in\mathcal{H}_{K_{\theta^\star}}
\]
such that
\begin{equation} \label{eq: minimizer_full}
\mathcal{J}_{\theta^\star}(\phi^\star)
=
\inf_{\theta\in\Theta}
\inf_{\phi\in\mathcal{H}_{K_\theta}}
\mathcal{J}_\theta(\phi). 
\end{equation}
In particular,  $\phi^\star_{N}(\cdot)$ defined by
\begin{equation} \label{eq:represent_opt}
\phi^\star_{N}(\cdot) := \phi^\star_{N,\theta^\star}(\cdot)
=
\sum_{i=1}^N
\alpha_{\theta^\star,i}^\star
\sum_{k=0}^{n-1}
G_{t_{k+1}}^i
K_{\theta^\star}
\bigl(X_{t_k}^i,\cdot\bigr),
\end{equation}
with $(\theta^\star,\boldsymbol{\alpha}_{\theta^\star}^\star)$ solution of \eqref{eq: deep_kernel_hedging_rewritten}, is an optimal deep kernel hedging strategy. 
\end{thm}

\begin{proof}
    The proof is a direct consequence of Theorem \ref{thm:representer_thm_inner_problem} and Lemma \ref{lem:existence_deep_kernel_hedging_rewritten}. 
\end{proof}

\begin{sqremark}
The optimal kernel parameter \(\theta^\star\) need not be unique.
Moreover, if \(\mathbf{Q}_{\theta^\star}\) is singular, the
coefficient vector \(\bm{\alpha}^\star\) need not be unique either.
Nevertheless, Lemma~\ref{lem: existence and uniqueness} ensures that,
for each fixed minimizer \(\theta^\star\), the corresponding optimal RKHS hedging
function \(\phi_{N, \theta^\star}^\star\) is unique.
\end{sqremark}

\vspace{2mm}
The finite-dimensional formulation \eqref{eq: deep_kernel_hedging_rewritten} provides a natural starting point for computing the deep kernel hedge characterized by Theorem~\ref{thm:existence_thm_hedging}. For the squared loss, one may substitute for each $\theta$ the closed-form solution \eqref{eq:opt_alpha2}  for the optimal coefficients $\bm{\alpha}^\star$ into the objective and optimize the resulting reduced objective over $\theta$ using gradient-based methods. However, a direct implementation is computationally demanding: constructing the hedging Gram matrix $\mathbf Q_\theta\in\mathbb R^{N\times N}$ requires $\mathcal O(N^2n^2)$ kernel evaluations, storing it requires $\mathcal O(N^2)$ memory, and solving the associated linear system using a standard  solver requires $\mathcal O(N^3)$ operations. These computations must then be repeated whenever $\theta$ changes.\\[-0.2cm]

Alternatively, for general loss functions or to avoid repeatedly solving the inner linear system, one may  instead optimize $\theta$ and $\bm{\alpha}$ jointly using stochastic gradient methods applied to \eqref{eq: deep_kernel_hedging_rewritten}. However,  this approach again requires storing and manipulating the Gram matrix $\mathbf Q_\theta$, resulting in a memory complexity of order $\mathcal{O}(N^2)$ and a computational cost of order $\mathcal{O}(N^2 n^2)$ at every training iteration, or $\mathcal{O}(N B n^2)$ when using mini-batching with batch size $B$. This cost can be prohibitive in the deep kernel setting because each update of the network parameters $\theta$ changes the kernel $K_\theta$ and therefore requires the corresponding Gram matrix entries to be recomputed. Moreover, although the number $N$ of paths is typically moderate in financial applications, the total number $Nn$  of training observations can be large, as illustrated in Section \ref{sec:num}, particularly for payoffs with relatively long maturities. The next section introduces a random Fourier feature approximation that avoids explicitly constructing this Gram matrix and leads to a substantially more scalable formulation

\subsection{Random Fourier features for deep kernel hedging}
\label{sec:rff_deep_kernel_hedging}
To obtain a more scalable formulation of our deep kernel hedging problem, we use a random Fourier feature approach. Random Fourier features (RFF), introduced by
\cite{rahimi2007random} and further analyzed by \cite{rahimi2008weighted, sutherland2015error, avron2017random}, provide an explicit finite-dimensional
approximation of continuous shift-invariant kernels. Importantly, the
deep kernel \(K_\theta\) needs not itself be shift-invariant as a function of the
original inputs. We only require the base kernel $K$ in \eqref{eq:def_deep_kernel} acting on the
latent space $\mathcal{Z}_
\theta= \psi_\theta(\mathcal{X}) \subseteq \R^p$ to be shift-invariant. More precisely, a base kernel $K$ is called shift-invariant, if it depends only on the difference between its arguments, i.e.,
\[
K(z, z') = \kappa(z - z'), \quad \forall z, z' \in \mathbb{R}^{p},
\]
for some continuous function $\kappa:\R^p \to \R$. Examples of shift-invariant kernels include the radial basis function (RBF) kernel. The starting point for the approximation is Bochner's theorem, which characterizes continuous, positive-definite, shift-invariant kernels. 
 \begin{thm}[Bochner]
   A continuous function \(\kappa:\mathbb R^p\to \R\) is positive
definite if and only if there exists a unique finite non-negative Borel
measure \(p_K\) on \(\mathbb R^p\) such that
\[
\kappa(z)
=
\int_{\mathbb R^p}
e^{\mathrm i w^\top z}\,p_K(dw),
\qquad z\in\mathbb R^p.
\]
\end{thm}

We further assume that the base kernel is normalized, so that \(\kappa(0)= 1\), and that $\kappa \in L^1(\R^p)$. Its spectral measure \(p_K\) is then an absolutely continuous probability
measure. Bochner's representation then becomes
\begin{equation}\label{eq: kernel_rep_positive_measure}
K(z,z')
=
\int_{\mathbb R^p}
e^{\mathrm i w^\top(z-z')}\,p_K(w) dw.
\end{equation}
This representation \eqref{eq: kernel_rep_positive_measure} thus shows that positive definite shift-invariant normalized kernels can be expressed as expectations over complex exponentials.
Following \cite{rahimi2007random} and since \(K\) is real-valued, we can rewrite
\begin{equation}\label{eq: kernel_rep_positive_measure_uniform}
K(z,z') = 2\, \E_{(w,b) \sim p_K \otimes \mathcal{U}(0,2\pi)}\left[ \cos\left( w^\top z+b\right) \cos\left( w^\top z'+b\right)\right] \, .
\end{equation}
Applying this representation to 
 \(z=\psi_\theta(x)\) and
\(z'=\psi_\theta(x')\) gives for the deep kernel \eqref{eq:def_deep_kernel} the formulation

\begin{equation*}
    K_\theta(x,x')= K(\psi_\theta(x),\psi_\theta(x')) = 2\, \E_{(w,b)\sim p_K \otimes \mathcal{U}(0,2\pi)}\left[ \cos\left( w^\top \psi_\theta(x)+b\right) \cos\left( w^\top \psi_\theta(x')+b\right)\right]. 
\end{equation*}

The key idea behind random Fourier features is to approximate this expectation by Monte Carlo sampling. For a given number of random features $D\ll N$, let
\[
(w_j,b_j)
\overset{\mathrm{i.i.d.}}{\sim}
p_K\otimes\mathcal U(0,2\pi),
\qquad j=1,\ldots,D,
\]
and define
\begin{equation} \label{eq: latent_repr}
y_\theta(x)
:=
\sqrt{\frac{2}{D}}
\begin{bmatrix}
\cos\bigl(w_1^\top\psi_\theta(x)+b_1\bigr)\\
\vdots\\
\cos\bigl(w_D^\top\psi_\theta(x)+b_D\bigr)
\end{bmatrix}
\in\mathbb R^D.
\end{equation}
The deep kernel $K_\theta$ in \eqref{eq:def_deep_kernel} then admits the finite approximation
\begin{equation}\label{eq:RFF_kernel_hedging}
\widetilde{K}_{\theta,D}(x,x'):=y_\theta(x)^\top  y_\theta(x'), \ \quad x, x' \in \mathcal{X}.
\end{equation}
The random frequencies and phases $\{(w_j,b_j)\}_{j=1}^D$ defining \(y_\theta\) are sampled once and
kept fixed throughout training. Replacing $K_\theta$ by $\widetilde K_{\theta,D}$ in \eqref{eq: deep_kernel_hedging} yields the RFF hedging problem
 \begin{equation}\label{eq: inner_outer2}
\inf_{\theta\in\Theta}\;
\inf_{\phi\in\mathcal H_{\widetilde K_{\theta,D}}} \left\{
\frac{1}{N}\sum_{i=1}^N
\mathcal L\left(
(\mathbf Y^{v_0})_i
-
\sum_{k=0}^{n-1}
\phi(X_{t_k}^i)G_{t_{k+1}}^i
\right)
+
\lambda
\|\phi\|_{\mathcal H_{\widetilde K_{\theta,D}}}^2 \right\} \, .
\end{equation}

Moreover, for each sample path \(i\), we define the aggregated feature vector
\begin{equation}
\label{eq:def_psi_theta}
Z_{\theta}^i
:=
\sum_{k=0}^{n-1}
y_\theta\bigl(X_{t_k}^i\bigr)
G_{t_{k+1}}^i
\in\mathbb{R}^D \, ,
\end{equation}
as well as the matrix of aggregated features
\begin{equation*}
\mathbf{Z}_\theta
:=
\begin{bmatrix}
(Z_\theta^1)^\top\\
\vdots\\
(Z_\theta^{N})^\top
\end{bmatrix}
\in\mathbb{R}^{N\times D}.
\end{equation*}
Substituting \eqref{eq:RFF_kernel_hedging} into the definition 
\eqref{eq:def_Q_theta} of $\mathbf{Q}_\theta$ yields the approximation $\widetilde{\mathbf{Q}}_{\theta,D}$  given by
\begin{equation*}
 \widetilde{\mathbf{Q}}_{\theta,D}
:=
\mathbf{Z}_\theta \mathbf{Z}_\theta^\top. 
\end{equation*}
Applying Theorem~\ref{thm:representer_thm_inner_problem}
to $\widetilde K_{\theta,D}$ hence yields the following
equivalent finite-dimensional formulation of
\eqref{eq: inner_outer2}\footnote{Under Assumption \ref{ass:kernel}, the existence of a solution to \eqref{eq: inner_outer2} and \eqref{eq: RFF_hedging_problem} follows by the same arguments as those used in the proof of Lemma \ref{lem:existence_deep_kernel_hedging_rewritten} and Theorem \ref{thm:existence_thm_hedging} with $K_\theta$
replaced by $\widetilde K_{\theta,D}$. In particular, the required continuity holds if
$\theta\mapsto y_\theta(x)$ is continuous for every
$x\in\mathcal X$.}
\begin{equation}\label{eq: RFF_hedging_problem}
    \min_{\theta\in\Theta,\;\bm{\alpha}\in\mathbb{R}^N}
    \left\{
    \frac{1}{N}
    \sum_{i=1}^N
    \mathcal{L}\left(
    (\mathbf{Y}^{v_0}
    -
    \widetilde{\mathbf{Q}}_{\theta,D}\bm{\alpha})_i
    \right)
    +
    \lambda
    \bm{\alpha}^{\top}\widetilde{\mathbf{Q}}_{\theta,D}\bm{\alpha}
    \right\}. 
\end{equation}
The following lemma reformulates the optimization \eqref{eq: RFF_hedging_problem} over
$\bm\alpha\in\mathbb R^N$, for fixed $\theta$, as an equivalent
problem over $\bm\beta\in\mathbb R^D$. This formulation avoids
constructing the Gram matrix and forms the basis of the
efficient optimization algorithm described below.

\begin{lem}[Primal formulation] \label{lem:RFF_hedging}Let $\lambda>0$ and fix $\theta\in\Theta$. Then,
\begin{equation}\label{eq: reform_RFF_deep_kernel_hedging_fixed_theta}
\min_{\bm{\alpha}\in\mathbb{R}^N}
\left\{
\frac{1}{N}
\sum_{i=1}^N
\mathcal{L}\left( (
\mathbf{Y}^{v_0}
-
\mathbf{Z}_\theta\mathbf{Z}^\top_\theta\bm{\alpha})_i
\right)
+
\lambda
\| \mathbf{Z}^\top_\theta \bm{\alpha} \|^2_2
\right\} = 
  \min_{\bm{\beta}\in \mathbb{R}^D}  \left\{  \frac{1}{N}\sum_{i=1}^N \mathcal{L}\left( (\mathbf{Y}^{v_0} -\mathbf{Z}_\theta \bm{\beta})_i\right) +\lambda \|\boldsymbol\beta\|^2_2 \right\}  .
    \end{equation}
Moreover,  the unique minimizer $\phi^\star_{D,\theta}$ of the inner RFF problem in \eqref{eq: inner_outer2}  satisfies
\begin{equation}\label{eq: min_inner_prob_RFF}
\phi^\star_{D,\theta}(x) = (\boldsymbol{\beta}^\star_\theta)^\top y_\theta(x), \quad x \in \mathcal{X}\,
\end{equation}
where  $\boldsymbol{\beta}^\star_\theta \in \R^D$ is the unique minimizer of the right-hand side in \eqref{eq: reform_RFF_deep_kernel_hedging_fixed_theta}.
\end{lem}

\begin{proof}
    The proof is deferred in Appendix \ref{proof_Prop_2.1}.
\end{proof}

By Lemma~\ref{lem:RFF_hedging}, the RFF approximation
of the deep kernel hedging \eqref{eq: RFF_hedging_problem} is equivalently given by
\begin{equation}\label{eq: reform_RFF_deep_kernel_hedging}
\min_{\theta\in\Theta, \, \boldsymbol\beta\in\mathbb R^D}
\left\{ \frac{1}{N}
\sum_{i=1}^N
\mathcal{L}\left( (\mathbf{Y}^{v_0} -\mathbf{Z}_\theta \bm{\beta})_i\right)
+
\lambda\|\boldsymbol\beta\|_2^2
\right\}.
\end{equation}

Let $(\theta^\star_D,\boldsymbol\beta^\star_{D}:= \boldsymbol{\beta}^\star_{\theta^\star_D})$ solve \eqref{eq: reform_RFF_deep_kernel_hedging}. The corresponding optimal  RFF hedging strategy minimizing \eqref{eq: inner_outer2} is then given by \[
\phi^\star_{D}(\cdot):= (\boldsymbol{\beta}^{\star}_{D})^\top \, y_{\theta^\star_D}(\cdot). 
\]

We now show with Theorem~\ref{thm:RFF_hedging_convergence} that both the optimal value \eqref{eq: reform_RFF_deep_kernel_hedging} of the RFF problem and the minimizer $\phi^*_D$ converge to those of the exact deep kernel hedging problem \eqref{eq: deep_kernel_hedging} when the number $D$ of Fourier features increases. First, we need the following extension of \cite[Claim 1]{rahimi2007random}.

\begin{prop}[Uniform approximation]
\label{prop:uniform_RFF_theta}
Let $\mathcal X\subset\mathbb R^d$ and
$\Theta\subset\mathbb R^m$ be nonempty compact sets. Suppose that the map
\(
(\theta,x)\mapsto\psi_\theta(x)
\)
is jointly continuous on $\Theta\times\mathcal X$ and that the spectral
distribution $p_K$ of the base kernel $K$ satisfies
\[
\mathbb E_{w\sim p_K}\bigl[\|w\|_2^2\bigr]<\infty.
\]
For each $D$, construct $\widetilde{K}_{\theta,D}$ using the same draw of  random Fourier features $\{(w_j,b_j)\}_{j=1}^D \overset{\mathrm{i.i.d.}}{\sim}
p_K\otimes\mathcal U(0,2\pi)$ for every
$\theta\in\Theta$. Then, there exist constants $C,c>0$, independent of $D$ and
$\varepsilon$, such that, for every $D\geq1$ and
$0<\varepsilon\leq1$,
\begin{equation}\label{eq: conv_rate_approx_kernel_D}
\mathbb{P} \left( \sup_{\theta\in\Theta}
\sup_{x,x'\in\mathcal X}
\left|
\widetilde K_{\theta,D}(x,x')
-
K_\theta(x,x')
\right| > \varepsilon \right)
\leq C\varepsilon^{-2} \exp(-cD \varepsilon^2)
\end{equation}
In particular,
\begin{equation}
\label{eq:uniform_RFF_theta}
\sup_{\theta\in\Theta}
\sup_{x,x'\in\mathcal X}
\left|
\widetilde K_{\theta,D}(x,x')
-
K_\theta(x,x')
\right|
\rightarrow0
\quad\text{ almost surely as }D\to\infty.
\end{equation}
 The corresponding hedging Gram matrices satisfy
\begin{equation}
\label{eq:uniform_Q_RFF_theta}
\sup_{\theta\in\Theta}
\left\|
\widetilde{\mathbf Q}_{\theta,D}
-
\mathbf Q_\theta
\right\|_{\mathrm{op}}
\rightarrow0
\qquad\text{ almost surely}\text{ as }D\to\infty .
\end{equation}
\end{prop}

\begin{proof}
The proof is deferred in Appendix \ref{Proof:Proposition3.1}. 
\end{proof}
We now establish the main result of this section on the convergence of the RFF approximation to the original deep kernel hedging problem.
\begin{thm}[Convergence of the RFF hedging problem]
\label{thm:RFF_hedging_convergence}
Suppose that Assumption \ref{ass:kernel} and the assumptions of
Proposition~\ref{prop:uniform_RFF_theta} hold. Let
\(
\mathcal V
:= \mathcal{J}_{\theta^\star}(\phi^\star)
\)
denote the optimal value \eqref{eq: minimizer_full} of the exact deep-kernel hedging problem, and let
\(
\mathcal V_D \)
denote the optimal value of its RFF approximation, i.e.,
\begin{equation} \nonumber
\mathcal V_D:=    \min_{\theta \in \Theta, \, \bm{\beta}\in \mathbb{R}^D}  \left\{  \frac{1}{N}\sum_{i=1}^N \mathcal{L}\left((\mathbf{Y}^{v_0}-\mathbf{Z}_\theta \bm{\beta})_i\right) +\lambda \|\boldsymbol\beta\|^2 \right\} \, .
\end{equation}
Then
\[
\mathcal V_D\rightarrow\mathcal V
\qquad \text{ almost surely as }D\to\infty.
\]

For each $D$, let
\begin{equation}
(\theta^\star_D, \boldsymbol\beta^\star_D)
\in
\operatorname*{argmin}_{\theta\in\Theta, \,
\boldsymbol\beta\in\mathbb R^D}
\left\{
\frac1N\sum_{i=1}^N
\mathcal L\left((\mathbf{Y}^{v_0}-\mathbf{Z}_\theta \bm{\beta})_i\right) 
+
\lambda\|\boldsymbol\beta\|_2^2
\right\},
\label{eq:findechezfin}
\end{equation}
and define the associated RFF hedging strategy by
\begin{equation}\label{eq: RFF_hedging_strategy}
\phi^\star_{D}(x)
:=
\boldsymbol\beta_D^{\star,\top}
y_{\theta^\star_D}(x),
\qquad x\in\mathcal X.
\end{equation}

Then every accumulation point of the sequence
$\{(\theta^\star_D,\phi^\star_D)\}_{D>1}$ in
$\Theta\times C(\mathcal X)$, where $C(\mathcal X)$ is equipped with the
uniform norm, is a global minimizer of the exact problem. In particular, if the exact problem \eqref{eq: minimizer_full} has a unique minimizing pair
$(\theta^\star,\phi^\star)$, then
\[
\theta^*_D\rightarrow\theta^\star \qquad \text{and} \qquad
\sup_{x \in \mathcal{X}}|\phi^*_D(x)-\phi^\star(x)|
\rightarrow0
\qquad \text{ almost surely as }D\to\infty.
\]
\end{thm}
\begin{proof}
    The proof is deferred in Appendix \ref{Proof:theorem3.3}.
\end{proof}

Based on the formulation of the hedging problem with the RFF approximation given by \eqref{eq: reform_RFF_deep_kernel_hedging}, we can implement a deep approach to jointly optimize $(\theta,\boldsymbol\beta)$ using a stochastic gradient-based method. The RFF approximation replaces the $N\times N$ Gram matrix $\mathbf{Q}_\theta$ by an explicit $N\times D$ Fourier feature matrix $\mathbf{Z}_{\theta}$, together with a coefficient vector $\boldsymbol\beta\in\mathbb{R}^D$. Thus, the kernel representation scales linearly in the number of random features $D$, rather than quadratically in the number of training samples. At every training iteration (with $\theta$ fixed), the RFF-based approximate kernel approach has a memory complexity of order $\mathcal{O}(ND)$ and a computational cost of order $\mathcal{O}(NnD)$, or $\mathcal{O}(BnD)$ when using mini-batches of size $B$. Whenever $D\ll nN$, this is significantly lower than the $\mathcal{O}(N^2n^2)$ computational cost of the corresponding full kernel-based hedging method (or $\mathcal{O}(BNn^2)$ when using mini-batches). Hence, the RFF formulation avoids storing or manipulating the full Gram matrix, making the proposed RFF-based kernel hedging approach substantially more scalable than its exact deep-kernel counterpart, especially when combined with stochastic gradient-based optimization as in the following algorithms.

\subsection{The deep kernel hedging algorithm}\label{sec:deep-kernel-algo}
To illustrate the flexibility of our deep kernel hedging framework, we consider two popular and complementary loss functions and design their corresponding algorithm.
\subsubsection{Quadratic hedging}\label{subsec:QH}
Given a dataset of $N$ sample paths with the same given initial conditions, it is common in quadratic hedging problems to also learn the initial price $v_0\in \mathbb{R}$, such that the problem is given by 
\begin{equation*}
\inf_{v_0\in \mathbb{R},~\theta\in \Theta,~\phi\in \mathcal{H}_{K_\theta}} 
\frac{1}{N} \sum_{i=1}^N
\left(
Y^i - B_T^iv_0 - \sum_{k=0}^{n-1} \phi({X}^i_{t_k}) \, G^i_{t_{k+1}}
\right)^2
+
\lambda \|\phi\|_{\mathcal{H}_{K_\theta}}^2. 
\end{equation*}
This hedging problem fits into our convex hedging setting and is therefore well-posed. \vspace{1mm}
\begin{sqremark}
The inclusion of \(v_0\) does not affect existence for any continuous,
convex and coercive loss function \(\mathcal L\).\ Indeed, since
\(B_T^i>0\) for every \(i\), the loss controls unbounded sequences in
\(v_0\), while the regularization term controls unbounded sequences in
\(\phi\). The same compactness and lower-semicontinuity argument as in Lemma \ref{lem: existence and uniqueness} and Lemma \ref{lem:existence_deep_kernel_hedging_rewritten} therefore yields a global minimizer \((v_0^\star,\theta^\star,\phi^\star)\). In particular, this applies to the squared loss above.
\end{sqremark}

To avoid the computational burden of storing the full $\R^{N\times N}$ hedging Gram matrix $\mathbf{Q}_\theta$, we consider the RFF approximation from the previous section  and solve the problem
\begin{equation} \label{eq:RFF_quadratic}
    \min_{v_0\in \mathbb{R}, \, \theta\in \Theta,~\bm{\beta}\in \mathbb{R}^D} \frac{1}{N}\|\mathbf{Y}^{v_0}-\mathbf{Z}_\theta \bm{\beta}\|^2 +\lambda \|\boldsymbol\beta\|^2, 
\end{equation}
where $\mathbf{Z}_\theta $ is defined by \eqref{eq:def_psi_theta} and the RFF  hedging strategy given by \eqref{eq: RFF_hedging_strategy}. This formulation leads to Algorithm \ref{alg:quadratic_deep_RFF_hedging_SGD}, which implements a fully stochastic gradient descent procedure to this finite-dimensional problem and returns an approximate numerical candidate solution $(\widehat{v}_0, \widehat{\theta}_D,\widehat{\boldsymbol{\beta}}_D)$.

\begin{algorithm}[!h]
\caption{ Quadratic deep kernel hedging with RFF approximation}
\label{alg:quadratic_deep_RFF_hedging_SGD}
\textbf{Input:} Sample paths $\{S^i_{t_k},B^i_{t_k},X^i_{t_k}\}_{i=1,...,N;~k=0,...,n}$, payoffs $\{Y^i\}_{i=1}^N$, shift-invariant base kernel $K$ with spectral measure $p_K$, regularization parameter $\lambda$, learning rate $\eta_e$, number of epochs $E$, batch size $N_{\mathrm{b}}$, and  number of RFF features $D$.  \vspace{-1mm}

\begin{enumerate}

\item For every $i=1,\ldots,N$ and $k=0,\ldots,n-1$, set\vspace{-1mm}
\[
x_{i,k} = {{X}}^i_{t_k}, 
\qquad 
G^i_{t_{k+1}} = B_T^i\left(\frac{S^i_{t_{k+1}}}{B^i_{t_{k+1}}} - \frac{S^i_{t_k}}{B^i_{t_k}}\right). \vspace{-1mm}
\]
\item Initialize:\vspace{-1mm}
\begin{itemize}
\item Price parameter $v_0 \in \R$ .
\item Neural network $\psi_\theta$ with parameters $\theta \in \Theta$.
\item Coefficients $\bm{\beta} \in \R^D$.
\item Random features  $(w_j,b_j)
\overset{\mathrm{i.i.d.}}{\sim}
p_K\otimes\mathcal U(0,2\pi)$ for $j=1,\ldots, D,$ and define
\[
y_\theta(x) = \sqrt{\frac{2}{D}}
\begin{bmatrix}
\cos(w_1^T \psi_\theta(x) + b_1) \\
\vdots \\
\cos(w_D^T \psi_\theta(x) + b_D)
\end{bmatrix} .
\]
\end{itemize}

\item \textbf{For} epoch $e = 1, \dots, E$ \textbf{do}
\begin{enumerate}

\item Sample a mini-batch $\mathcal{B} \subset \{1,\dots,N\}$ of size $N_{\mathrm{b}}$.

\item Compute the empirical loss on the batch:
\[
\mathcal{L}(v_0,\theta,\boldsymbol\beta) = \frac{1}{N_{\mathrm{b}}} \sum_{i \in \mathcal{B}} 
\left(
Y^i - B_T^i v_0 
-  \bm{\beta}^\top \sum_{k=0}^{n-1}  y_\theta(x_{i,k}) \, G^i_{t_{k+1}}
\right)^2
+ \lambda \, \bm{\beta}^\top \bm{\beta}. 
\]

\item Update parameters using a stochastic gradient descent (SGD) method:
\[
v_0\leftarrow \texttt{SGD}(v_0, \nabla_{v_0} \mathcal{L}, \eta_e),
\qquad
\theta \leftarrow \texttt{SGD}(\theta, \nabla_\theta \mathcal{L}, \eta_e),
\qquad
\bm{\beta} \leftarrow \texttt{SGD}(\bm{\beta}, \nabla_\beta \mathcal{L}, \eta_e).
\]

\end{enumerate}
\item[] \textbf{End for}

\item \textbf{Output:} Learned initial price $v_0$ and approximated hedging strategy \vspace{-1mm}
\[
\widehat{\phi}_{D}(\cdot) = \sum_{j=1}^D \beta_{j} \, y_\theta^j(\cdot). 
\]
\end{enumerate}
\end{algorithm}

\subsubsection{CVaR hedging}\label{subsec:CVaR hedging}

The quadratic hedging objective introduced in Section~\ref{subsec:QH} penalizes positive and negative terminal hedging errors $\{\mathcal{E}_i(\phi)\}_{i=1}^N$ symmetrically and controls their overall dispersion. It does not, however, specifically target adverse tail outcomes. To emphasize large hedging losses, we can replace the quadratic objective by a conditional value-at-risk (CVaR) objective. Let $\tilde\alpha\in(0,1)$ denote the lower-tail probability. The CVaR hedging problem consists in  minimizing the upper-tail $\operatorname{CVaR}_{1-\tilde\alpha}^{\text{upper}}$ of the  terminal hedging shortfall $\mathcal{E}_i(\phi)$ at confidence level $1-\tilde\alpha$. Equivalently, this amounts to maximizing the
lower-tail  $\operatorname{CVaR}_{\tilde\alpha}^{\text{lower}}$ of the terminal P\&L $-\mathcal{E}_i(\phi)$, that is, its average value over
the worst $\tilde\alpha$ fraction of outcomes. Using the Rockafellar--Uryasev representation \citep{buehler2019deep}, the corresponding regularized empirical problem can be written
\begin{equation} \begin{aligned}\label{eq:deep_kernel_cvar_hedging}
\inf_{\theta \in \Theta, \phi \in \mathcal{H}_{K_\theta}}& \left\{ \operatorname{CVaR}_{1-\tilde\alpha}^{\text{upper}}
\Big(\big\{\mathcal E_i(\phi;v_0)\big\}_{i=1}^N\big)  + \lambda \|\phi\|_{\mathcal{H}_{K_\theta}}^2
\right\}
\\
&= \inf_{\substack{
\,\theta\in\Theta,\,
\phi\in\mathcal{H}_{K_\theta},\,\zeta \in\mathbb{R}
}}
\left\{
\zeta
+
\frac{1}{\tilde\alpha N}
\sum_{i=1}^N
\left(
Y^i
-
B_T^i v_0
-
\sum_{k=0}^{n-1}
\phi(X^i_{t_k})
G^i_{t_{k+1}} -\zeta
\right)_+
+
\lambda
\|\phi\|_{\mathcal{H}_{K_\theta}}^2
\right\},
\end{aligned}
\end{equation}
where $\zeta \in\mathbb{R}$ is an auxiliary optimization variable. 
Note that at any optimum $(\theta^*, \phi^*, \zeta^*)$, $\zeta^\star$ is an empirical $(1-\tilde\alpha)$-quantile of the terminal hedging shortfall,
\(
\zeta^\star
\in
q^{\text{upper}}_{1-\tilde\alpha}
\Big(
\left\{
\mathcal{E}_i(\phi^*_{N,\theta^*};v_0)
\right\}_{i=1}^N
\Big) 
\). \vspace{1mm}

\begin{sqremark}
The Rockafellar--Uryasev integrand
\[
\mathcal L_{\tilde\alpha}^{\mathrm{CVaR}}(u,\zeta)
:=
\zeta+\frac{1}{\tilde\alpha}(u-\zeta)_+
\]
is continuous and jointly convex, but it is neither bounded from below
nor coercive jointly in $(u,\zeta)$.  The initial portfolio value $v_0$ is therefore treated as a fixed initial constraint. In
particular, when $B_T^i=B_T>0$, the cash-invariance of CVaR gives
\(
\operatorname{CVaR}^{\text{upper}}_{1-\tilde\alpha}
\bigl(\{\mathcal E_i(\phi;v_0)\}_{i=1}^N\bigr)
=
\operatorname{CVaR}^{\text{upper}}_{1-\tilde\alpha}
\bigl(\{\mathcal E_i(\phi;0)\}_{i=1}^N\bigr)-B_Tv_0
\).
Consequently, unconstrained minimization over $v_0$ would make the
objective unbounded below. A hedging price may instead be defined
through an indifference-pricing criterion, as in
\cite{buehler2019deep}.
\end{sqremark}

Although the CVaR integrand is not bounded below jointly in
$(u,\zeta)$, the regularized empirical problem
\eqref{eq:deep_kernel_cvar_hedging} remains well posed.

\begin{prop}[Existence for CVaR hedging]
\label{prop:existence_CVaR}
Let $\tilde\alpha\in(0,1)$ and $\lambda>0$, and suppose that $v_0\in \mathbb{R}$ is
fixed and Assumption~\ref{ass:kernel} holds. Then,
\eqref{eq:deep_kernel_cvar_hedging} admits a global minimizer.
\end{prop}
\begin{proof}
The proof is deferred in Appendix \ref{proof:existence_CVar}. 
\end{proof}
Using the definitions \eqref{eq:def_H_v0}--\eqref{eq:def_psi_theta} of $\mathbf{Y}^{v_0}$ and $\mathbf{Z}_\theta$, we once again consider the RFF approximate formulation introduced in Section \ref{sec:rff_deep_kernel_hedging} and solve 
\begin{equation*}
    \min_{\theta\in \Theta,~\bm{\beta}\in \mathbb{R}^D,  \, \zeta \in \R} \left\{\zeta
+
\frac{1}{\tilde\alpha N}
\sum_{i=1}^N
\left(
\bigl(\mathbf{Y}^{v_0}-\mathbf{Z}_\theta \bm{\beta}\bigr)_i
-\zeta
\right)_+ +\lambda \|\bm{\beta}\|^2 \right\} .
\end{equation*}
This formulation leads to Algorithm \ref{alg:CVaR_deep_RFF_hedging_SGD}, which applies stochastic
gradient-based optimization to this finite-dimensional problem and
returns a numerical candidate solution
$(\widehat\theta_D,\widehat{\boldsymbol\beta}_D,\widehat\zeta_D)$.

\begin{algorithm}[!h]
\caption{ CVaR deep kernel hedging with RFF approximation}
\label{alg:CVaR_deep_RFF_hedging_SGD}
\textbf{Input:} Sample paths $\{S^i_{t_k},B^i_{t_k},X^i_{t_k}\}_{i=1,...,N;~k=0,...,n}$, payoffs $\{Y^i\}_{i=1}^N$, fixed initial capital $v_0$, shift-invariant base kernel $K$ with spectral measure $p_K$, regularization parameter $\lambda$, learning rate $\eta_e$, number of epochs $E$, batch size $N_{\mathrm{b}}$, and  number of RFF features $D$.

\begin{enumerate}

\item For every $i=1,\ldots,N$ and $k=0,\ldots,n-1$, set\vspace{-1mm}
\[
x_{i,k} = {{X}}^i_{t_k}, 
\qquad 
G^i_{t_{k+1}} = B_T^i\left(\frac{S^i_{t_{k+1}}}{B^i_{t_{k+1}}} - \frac{S^i_{t_k}}{B^i_{t_k}}\right). \vspace{-1mm}
\]
\item Initialize:
\begin{itemize}
\item Neural network $\psi_\theta$ with parameters $\theta \in \Theta$.
\item Coefficients $\bm{\beta} \in \R^D$ and $\zeta \in \R$.
\item Random features  $(w_j,b_j)
\overset{\mathrm{i.i.d.}}{\sim}
p_K\otimes\mathcal U(0,2\pi)$ for $j=1,\ldots, D,$ and define
\[
y_\theta(x) = \sqrt{\frac{2}{D}}
\begin{bmatrix}
\cos(w_1^T \psi_\theta(x) + b_1) \\
\vdots \\
\cos(w_D^T \psi_\theta(x) + b_D)
\end{bmatrix} .
\]
\end{itemize}

\item \textbf{For} epoch $e = 1, \dots, E$ \textbf{do}
\begin{enumerate}

\item Sample a mini-batch $\mathcal{B} \subset \{1,\dots,N\}$ of size $N_{\mathrm{b}}$.

\item Compute the empirical loss on the batch:
\[
\mathcal{L}(\theta,\boldsymbol\beta,\zeta) = \zeta + \frac{1}{\tilde \alpha N_{\mathrm{b}}} \sum_{i \in \mathcal{B}} 
\left(
Y^i - B_T^i v_0 
- \bm{\beta}^\top \sum_{k=0}^{n-1} y_\theta(x_{i,k}) \, G^i_{t_{k+1}}-\zeta \right)_+
+ \lambda \, \bm{\beta}^\top \bm{\beta}. 
\]

\item Update parameters using a stochastic gradient descent (SGD) method:
\[
\theta \leftarrow \texttt{SGD}(\theta, \nabla_\theta \mathcal{L}, \eta_e),
\qquad
\bm{\beta} \leftarrow \texttt{SGD}(\bm{\beta}, \nabla_\beta \mathcal{L}, \eta_e),
\qquad
\zeta \leftarrow \texttt{SGD}(\zeta, \nabla_\zeta \mathcal{L}, \eta_e).
\]

\end{enumerate}
\item[] \textbf{End for}

\item \textbf{Output:} Learned hedging strategy
\[
\widehat{\phi}_{D}(\cdot) = \sum_{j=1}^D \beta_{j} \, y_\theta^j(\cdot). 
\]
\end{enumerate}
\end{algorithm}

\section{Numerical experiments}\label{sec:num}
This section presents the numerical experiments used to evaluate the proposed deep kernel hedging approach. To this end, we consider both synthetic data generated in a Heston stochastic volatility model and real data from S\&P500 time series. We compare our deep kernel approach with standard benchmarks, including deep hedging and classical kernel methods. We first detail the architectures of the different learning approaches, before discussing the numerical results.

\paragraph{Deep kernel architecture.}
We consider a deep kernel construction where the input features are first mapped through a shallow neural network
$\psi_\omega : \mathbb{R}^{d} \to \mathbb{R}^p$, followed by a standard Gaussian (RBF) kernel applied in the learned representation space. The choice of the RBF kernel is motivated by its strong universality and smoothness properties\footnote{In particular, if $\mathcal{X} \subset \R^d$ is compact and $\psi_\theta : \mathcal{X} \to \R^p$ is continuous, then $\psi_\theta(\mathcal{X})$ is compact, and the Gaussian kernel is universal on the  latent set $\psi_\theta(\mathcal{X})$ \citep{micchelli2006universal}. Consequently, the deep kernel $K_\theta$ is universal on $\mathcal{X}$ if and only if $\psi_\theta$ is injective on $\mathcal{X}$; that is, its associated RKHS $\mathcal{H}_{K_\theta}$ is uniformly dense in $C(\mathcal{X})$.}, which make it particularly well-suited for capturing nonlinear relationships in a wide range of function classes. More specifically, the deep kernel is defined as
\[
K_\theta(x,y) := \exp\left(-\frac{\|\tilde{\psi}_\omega(x) - \tilde{\psi}_\omega(y)\|^2}{2\gamma^2}\right),
\]
where $\theta := (\omega,\gamma)$, with $\gamma>0$ controlling the scale of the representation in the RBF kernel, and
\[
\tilde{\psi}_\omega(\cdot)
:=\frac{\psi_\omega(\cdot)}{\max\{\|\psi_\omega(\cdot)\|, \epsilon_{\text{norm}}\}} 
\in\mathbb{S}^{p-1} ,
\]
denoting the $\ell_2$-normalized\footnote{Note that $\ell_2$-normalization improves numerical stability in our experiments, even though it may destroy injectivity of $\psi_\theta$. Similarly, we do not constrain in practice $\theta$ to lie in a compact set $\Theta$ as doing so limits the expressiveness of the neural network.} output of the neural network with $\epsilon_{\text{norm}} = 10^{-6}$. By defining
\[
\psi_\theta(\cdot)
:=\frac{1}{\gamma}\tilde{\psi}_\omega(\cdot),
\]
$\gamma$ can be absorbed into the learned representation, yielding the equivalent formulation
\begin{equation}\label{eq: reparameterization_deep_RBF_kernel}
K_\theta(x,y) = K_{\text{RBF}}({\psi}_\theta(x), {\psi}_\theta(y)),
\end{equation}
where
\[
K_{\text{RBF}}(x,y):=\exp\left(-\frac{\|x- y\|^2}{2}\right).
\]

The $\ell_2$ normalization removes the scale degree of freedom from the neural representation, allowing the neural network $\psi_\omega$ to focus on learning the geometry of the representation, while $\gamma$ controls the scale of the RBF kernel. Moreover, the representation \eqref{eq: reparameterization_deep_RBF_kernel} shows that the RBF scaling parameter can be learned jointly with the neural network parameters rather than tuned separately. This is particularly advantageous in our setting, since selecting an appropriate scale parameter for the RBF kernel can be challenging in practice; see for instance \cite{wainer2021tune}. Thus, by incorporating $\gamma$ into the learned representation, our approach avoids the need for a separate fine-tuning step for the scale parameter (or any other hyperparameter) of the base kernel. \\

As detailed in Sections \ref{sec:rff_deep_kernel_hedging}--\ref{sec:deep-kernel-algo}, our deep kernel-based hedging algorithm relies on a RFF approximation. Since we consider the standard RBF kernel as the base kernel, its spectral density is given by
\begin{equation*}
    p_{K_\text{RBF}}(w) = (2\pi)^{-p/2}
    \exp\left(-\frac{\|w\|^2}{2}\right),
    \qquad w\in\mathbb{R}^p,
\end{equation*}
and hence the random frequencies are sampled as
\begin{equation*}
    w_i \sim \mathcal{N}(0,I_p),
    \qquad i=1,\ldots,D.
\end{equation*}
Thanks to the reparameterization in \eqref{eq: reparameterization_deep_RBF_kernel}, the random frequencies $w_i$ are sampled once at the beginning of the algorithm and kept fixed throughout training. Without such a reparameterization, the random Fourier features would have to be sampled from a distribution that explicitly depends on the kernel hyperparameter $\gamma$. As a result, any update of $\gamma$ during training would require resampling the random features, leading to a less stable optimization procedure due to the additional source of randomness at each iteration.\\

\begin{figure}[!ht]
   \centering
    \resizebox{\linewidth}{!}{%
    \begin{tikzpicture}[
    x=1cm,
    y=1cm,
    >=Latex,
    neuron/.style={
        circle,
        draw=black,
        thick,
        fill=white,
        minimum size=1.15cm,
        inner sep=0pt,
        font=\large
    },
    output/.style={
        neuron,
        minimum size=1.45cm
    },
    connection/.style={
        ->,
        thin,
        draw=black
    },
    layer title/.style={
        font=\large
    },
    annotation/.style={
        align=center,
        font=\normalsize
    }
]

% ------------------------------------------------------------------
% Node coordinates
% ------------------------------------------------------------------

% Input layer
\node[neuron] (x1) at (0,  3.0) {$x_1$};
\node[neuron] (x2) at (0,  1.0) {$x_2$};
\node[neuron] (x3) at (0, -1.0) {$x_3$};
\node at (0,-2.15) {$\vdots$};
\node[neuron] (xd) at (0, -3.4) {$x_d$};

% Feature network
\node[neuron] (h1) at (4,  3.0) {};
\node[neuron] (h2) at (4,  1.0) {};
\node[neuron] (h3) at (4, -1.0) {};
\node at (4,-2.15) {$\vdots$};
\node[neuron] (h4) at (4, -3.4) {};

% Latent features
\node[neuron] (z1) at (8,  2.7) {$z_1$};
\node[neuron] (z2) at (8,  0.7) {$z_2$};
\node at (8,-0.65) {$\vdots$};
\node[neuron] (zp) at (8, -2.5) {$z_p$};

% Random Fourier feature layer
\node[neuron] (y1) at (12,  2.7) {$y_1$};
\node[neuron] (y2) at (12,  0.7) {$y_2$};
\node at (12,-0.65) {$\vdots$};
\node[neuron] (yD) at (12, -2.5) {$y_D$};

% Output layer
\node[output] (phi) at (15.4,0.1) {$\widehat{\phi}_D$};

% ------------------------------------------------------------------
% Connections
% Drawn in the background so that lines do not appear inside nodes
% ------------------------------------------------------------------

\begin{scope}[on background layer]

% Input -> feature network
\foreach \x in {x1,x2,x3,xd}{
    \foreach \h in {h1,h2,h3,h4}{
        \draw[connection] (\x) -- (\h);
    }
}

% Feature network -> latent features
\foreach \h in {h1,h2,h3,h4}{
    \foreach \z in {z1,z2,zp}{
        \draw[connection] (\h) -- (\z);
    }
}

% Latent features -> RFF layer
\foreach \z in {z1,z2,zp}{
    \foreach \y in {y1,y2,yD}{
        \draw[connection] (\z) -- (\y);
    }
}

% RFF layer -> output
\foreach \y in {y1,y2,yD}{
    \draw[connection] (\y) -- (phi);
}

\end{scope}

% ------------------------------------------------------------------
% Layer titles
% ------------------------------------------------------------------

\node[layer title] at (0,    4.65) {Input layer};
\node[layer title] at (4,    4.65) {Feature network};
\node[layer title] at (8,    4.65) {Latent features};
\node[layer title] at (12,   4.65) {RFF layer};
\node[layer title] at (15.4, 4.65) {Output layer};

% ------------------------------------------------------------------
% Mathematical descriptions
% ------------------------------------------------------------------

\node[annotation] at (0,-5.5)
    {input $x\in\mathcal{X}\subseteq \R^d$};

\node[annotation] at (4,-5.5)
    {Linear $\rightarrow$ tanh $\rightarrow$ Linear\\[2mm]
     $z=\psi_{\theta}(x)= \tilde\psi_{\omega}(x)/\gamma$};

\node[annotation] at (7.8,-5.5)
    {$z\in\mathbb{R}^{p}$};

\node[annotation] at (11.4,-5.35)
    {$\displaystyle {\scriptsize
      y=\sqrt{\frac{2}{D}}
      \cos\!\left(w^{\top}z+b\right) 
      \in\mathbb{R}^{D}}$;\\ \\ \vspace*{-1.5mm}
      $\displaystyle {\scriptsize 
      (w,b)\sim p_{K_{\rm{RBF}}}\hspace{-0.6mm}\otimes\mathcal{U}(0,2\pi)}$ fixed};

\node[annotation] at (15.4,-5.5)
    {$\displaystyle
      \widehat{\phi}_D=\boldsymbol{\beta}^{\top}y\in\mathbb{R}$};

\end{tikzpicture}}
    \caption{Deep Kernel Hedging architecture}
    \label{fig: DKH_architecture}
    \end{figure}
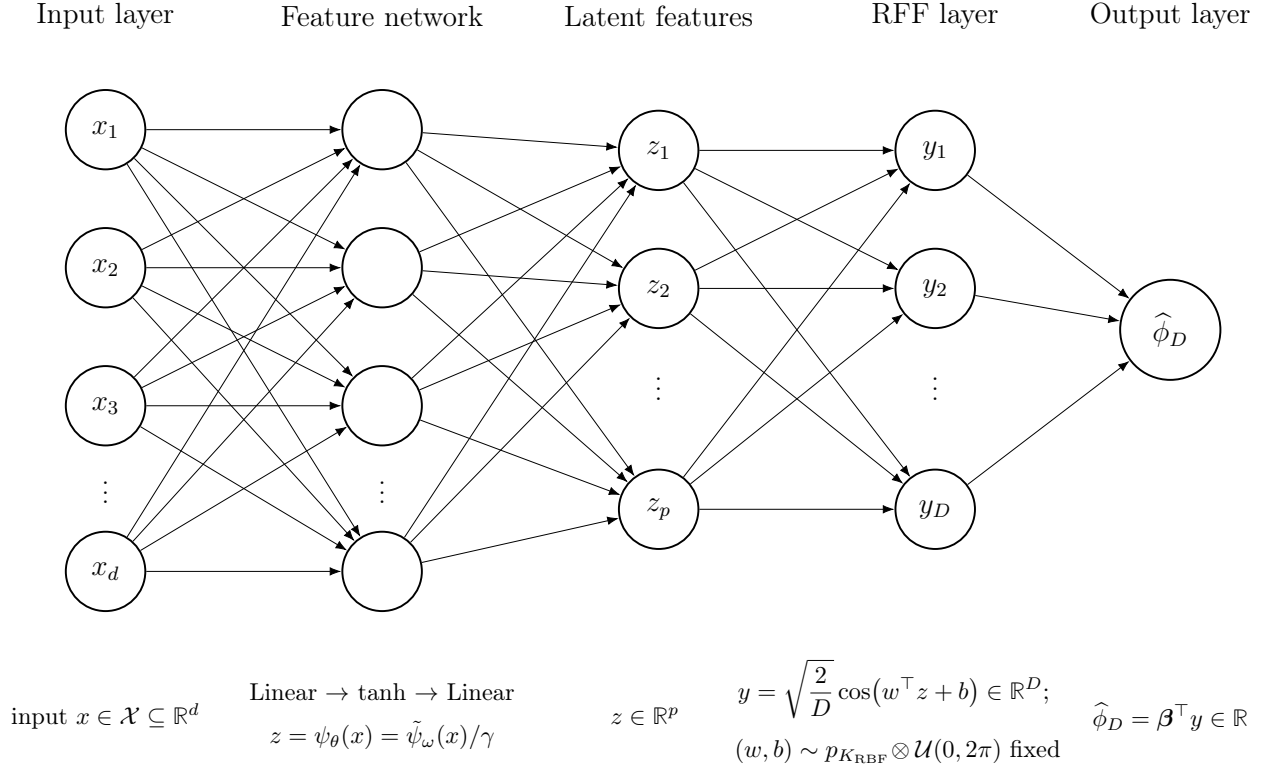

The neural network $\psi_\omega$ is composed of one hidden layer of 16 neurons with hyperbolic tangent (\texttt{Tanh}) activation function and one output linear layer of dimension 16. For the random Fourier feature approximation, we fix the number of features to $D=100$. This provides a good trade-off between approximation accuracy and computational cost in our low-data regime, as confirmed by Figure \ref{fig:RFF_conv} in Appendix \ref{sec:appendix_synthetic_data}. Note that we also standardize the input variable before it is fed into the neural network. Figure \ref{fig: DKH_architecture} highlights the chosen learning architecture for the deep kernel hedging approach with RFF approximation. The training is done in batches of size $N_{\mathrm{b}}=32$ for $E=100$ epochs, with a $\texttt{AdamW}$ optimizer, see \cite{loshchilov2017decoupled}, with weight decay $5\mathrm{e}{-3}$, a learning rate $\eta_e$ ranging from $1\mathrm{e}{-3}$ to $1\mathrm{e}{-4}$ with a \texttt{cosine} scheduler over the training steps, and a ridge regularization parameter $\lambda = 1\mathrm{e}{-6}$.

\paragraph{Benchmark methods.}
We compare the proposed deep kernel hedging approach against several benchmark methods. To this end, we consider both a classical kernel-based hedging strategy and a standard deep hedging approach with feedforward neural network (FNN). All the following benchmarks are trained using the $\texttt{AdamW}$ optimizer with the same hyperparameter configuration as the deep kernel hedging method. This ensures that differences in performance are due to  modeling choices rather than optimization settings. \\

For the classical kernel method, we adopt the same overall setting as in the deep kernel approach, but remove the learned representation. In particular, we fix the feature map to the identity, $\psi_\theta(x) = x$, and use an RBF kernel with bandwidth parameter $\gamma$. As in the deep kernel model, we rely on a RFF approximation of the kernel, with $D = 100$ features. In this case, the  kernel bandwidth $\gamma$ is the only learned kernel parameter. This design isolates the effect of learning a nonlinear representation, while keeping the kernel structure and RFF approximation scheme identical to the deep kernel framework.\\

For the classical deep hedging  \citep{buehler2019deep}, we implement a fully connected FNN as a benchmark model, where the network directly outputs the hedging strategy $\phi$. In general non-Markovian settings, one may naturally consider more sophisticated architectures such as recurrent neural networks (RNNs) or convolutional neural networks (CNNs) in order to capture temporal dependencies. However, recent works \citep{abi2025hedging} have shown that FNNs remain highly competitive in such settings when combined with appropriate feature engineering, in particular through time-augmented signature representations of the input features. This further supports the use of a FNN as a strong and versatile baseline across different modeling regimes. The chosen benchmark FNN consists of two hidden layers with $16$ neurons each, using \texttt{Tanh} activation functions, so as to match the deep kernel hedging architecture and to have a fair numerical comparison\footnote{Under the considered architectures, the numbers of trainable parameters are $16d+305$ for deep hedging and $16d+389$ for deep kernel hedging, where $d$ is the input dimension.}. \\

\begin{sqremark}
To assess whether the observed performance gains are specific to the use of \texttt{cosine} random Fourier features (specific to kernel approximation) or instead result from the learned latent representation $\psi_\theta$, we also investigate additional RFNN benchmarks with \texttt{ReLU} and \texttt{Tanh} activations \citep{huang2006extreme}, although these are not included in the main text. The corresponding results are reported by Table \ref{tab:other_metrics_qh_rfnn} in Appendix~\ref{sec:appendix_synthetic_data}. These experiments show that the choice of the random-feature activation has only a minor impact on performance, with \texttt{cosine}, \texttt{ReLU}, and \texttt{Tanh} features yielding very similar results when combined with the learned representation. In contrast, removing the learned representation  and using the random features directly leads to substantially weaker performance. This suggests that the main benefit of our approach comes from learning a nonlinear latent representation $\psi_\theta$ prior to the random-feature layer, rather than from the specific choice of the random-feature family. This further supports our use of \texttt{cosine} RFFs as the canonical choice throughout the paper.
\end{sqremark}

\subsection{Results on synthetic data: Heston model}
First, we consider numerical experiments on synthetic data. To this end, as in \cite{buehler2019deep}, we consider the \cite{heston1993closed} model. Let $\left(\Omega, \mathcal{F}, \mathcal{F}_t, \mathbb{P}\right)$ be a filtered probability space, where $\mathbb{P}$ stands for the real measure. The \cite{heston1993closed} model is specified
by the following stochastic differential equations:
\begin{equation*}
    \begin{aligned}
        dS_t &= \mu S_t dt + \sqrt{\nu_t} S_t dW^S_t, \quad S_0>0, \\
        d\nu_t&= \kappa\left(\bar\nu-\nu_t\right)dt + \sigma_\nu \sqrt{\nu_t}dW^\nu_t, \quad \nu_0>0,
    \end{aligned}
\end{equation*}
 where $(W^S,W^\nu)$ are correlated Brownian motions such that $d\langle W^S,W^\nu\rangle_t = \rho dt$, $\mu, \kappa, \bar\nu, \sigma_\nu>0$ and $\rho\in(-1,1)$. In addition, the cash account process $(B_t)_{t\ge0}$ satisfies $B_t = e^{rt}$, where $r\in \mathbb{R}$ denotes the constant risk-free interest rate. For our numerical results, we have chosen the set of parameters from \cite[Table 3]{mikkila2023empirical} given by
 \begin{equation*}
     S_0=1,r=0.02, \mu = 0.1, \nu_0 = 0.03, \bar\nu= 0.04, \kappa=3.41, \sigma_\nu=0.62, \rho = -0.79. 
 \end{equation*}
 
In this setting, we consider the quadratic and CVaR hedging of at-the-money (ATM) European and Asian call options. In the main text, we focus on the ATM European call with payoff
\begin{equation*}
Y_T = (S_T-S_0)_+,
\end{equation*}
while the corresponding results and analysis for the Asian call are provided in Appendix \ref{sec:appendix_synthetic_data}. To this end, we simulate the process $(S,\nu)$ using a discrete Euler scheme with $\delta_t=1/360$ that generates our training synthetic dataset consisting of $N$   sample paths observed at different trading dates $0=t_0<t_1<\cdots<t_n=T$,
 \[
\mathcal D_{N}
=
\Big\{
\left(\left(S_{t_k}^i,\nu^i_{t_k}\right)_{k=0}^n, Y^i_T\right)_{i=1}^{N}
\Big\}.
\]
In the Markovian setting considered here, we follow a similar approach to \cite{buehler2019deep,lutkebohmert2022robust,abi2025hedging} and use only the time-augmented risk factors as input features, namely\footnote{We also experimented with path-signature features, but found no additional benefit from their use in the Markovian setting considered here.}
 \begin{equation*}
     X^i_{t_k} = (t_k, S^i_{t_k},\nu^i_{t_k}), \qquad k=0,\cdots, n. 
 \end{equation*}
We remark that in more general non-Markovian settings, path-signature features could be used to enrich the state representation and capture path-dependent effects (\cite{abi2025hedging, cirone2025rough} and Section \ref{subsec:num_real} below). \\

We now compare the different learning approaches for different maturities
\[
T \in \{1/12,1/4,1/2\},
\]
and different training sample sizes
\[
N \in \{250,500,1000,2500\}.
\]
We restrict ourselves to relatively small training datasets in order to remain in a realistic financial setting. As a benchmark, we also include the classical \cite{black1973pricing} practitioner delta hedging strategy, where the hedge ratio at time $t$ is given by the Black--Scholes delta evaluated using the current stochastic volatility $\sqrt{\nu_t}$, i.e.,
\[
\Delta^{\mathrm{BS}}(S_t,K,T-t,r,\sqrt{\nu_t}).
\]
For quadratic hedging, the initial price $v_0$ is also learned, while for CVaR hedging, we fix the initial price to the corresponding Black--Scholes price with $\sigma=\sqrt{\nu_0}$, see Section \ref{sec:deep-kernel-algo}. We evaluate all methods out-of-sample on an independent test set $\mathcal{D}_{N^\text{OOS}}$ consisting of $N^\text{OOS}=25{,}000$ simulated paths. To assess hedging performance, we consider several OOS metrics based on the terminal hedging error \eqref{eq:hedg_error}, including the mean squared error (MSE),  mean absolute error (MAE), as well as tail-risk measures with the $95\%$ quantile $q^{\text{upper}}_{0.95}$ and upper-tail $\operatorname{CVaR}^{\text{upper}}_{0.95}$ of the terminal hedging error with $\tilde\alpha = 5\%$,
\begin{equation*}
    \begin{aligned}
    &\text{MSE} = \frac{1}{N^\text{OOS}}\sum_{i=1}^{N^\text{OOS}} \mathcal{E}^2_i(\widehat{\phi}_D), \ \ \, \qquad \text{MAE} = \frac{1}{N^\text{OOS}}\sum_{i=1}^{N^\text{OOS}} |\mathcal{E}_i(\widehat{\phi}_D)|, \\
    &q^{\text{upper}}_{0.95} = q^{\text{upper}}_{0.95}\big(\{\mathcal{E}_i(\widehat{\phi}_D) \}_{i=1}^{N^\text{OOS}}\big), \quad  \operatorname{CVaR}^{\text{upper}}_{0.95} =\operatorname{CVaR}^{\text{upper}}_{0.95}\big(\{ \mathcal{E}_i(\widehat{\phi}_D) \}_{i=1}^{N^\text{OOS}}\big) .
    \end{aligned}
\end{equation*} Moreover, since the learning procedures involve random initialization of the neural network parameters and since the kernel-based methods additionally rely on randomly sampled random Fourier features, we repeat the experiments using different random seeds. We report the mean of each performance metric across the different seeds, together with its corresponding standard deviation. \\

\begin{table}[!ht]
\small
\centering
\caption{Out-of-sample MSE for the ATM European call for quadratic hedging. Reported values are averages over ten seeds, with standard deviations in parentheses.}
\label{tab:loss_set_1}
\begin{minipage}{1\textwidth}
\centering
\caption*{$T=1/12$}
\begin{tabular}{lcccc}
\toprule
$N$& $250$ & $500$ &$1000$ &$2500$\\
\midrule
BS & $2.56 \mathrm{e}{-5}$ &$2.56 \mathrm{e}{-5}$ & $2.56 \mathrm{e}{-5}$& $2.56 \mathrm{e}{-5}$\\
Kernel hedging & $4.64\mathrm{e}{-5}$ ($2.2\mathrm{e}{-6}$) & $4.21\mathrm{e}{-5}$ ($1.1\mathrm{e}{-6}$) & $3.94\mathrm{e}{-5}$ ($7.1\mathrm{e}{-7}$)& $3.82 \mathrm{e}{-5}$ ($3.7\mathrm{e}{-7}$)\\
Deep hedging & $3.08\mathrm{e}{-5}$ ($2.4\mathrm{e}{-6}$) & $2.56\mathrm{e}{-5}$  ($7.9\mathrm{e}{-7}$)& $2.33\mathrm{e}{-5}$ ($3.4\mathrm{e}{-7}$)& $2.21\mathrm{e}{-5}$ ($2.4 \mathrm{e}{-7}$)\\
Deep kernel hedging& $\mathbf{2.44{e}{-5}~(1.2{e}{-6})}$ & $\mathbf{2.22 {e}{-5}~(4.8{e}{-7})}$ & $\mathbf{2.14{e}{-5}~(1.5{e}{-7})}$ & $\mathbf{2.09{e}{-5}~(5.9{e}{-8})}$\\
\bottomrule 
\end{tabular}
\end{minipage}
\\ \vspace{0.25cm}
\begin{minipage}{1\textwidth}
\centering
\caption*{$T=1/4$}
\begin{tabular}{lcccc}
\toprule
$N$& $250$ & $500$ &$1000$ &$2500$\\
\midrule
BS & $1.07\mathrm{e}{-4}$ &$1.07\mathrm{e}{-4}$ & $1.07\mathrm{e}{-4}$& $1.07\mathrm{e}{-4}$\\
Kernel hedging & $1.56\mathrm{e}{-4}$ ($8.3\mathrm{e}{-6}$) & $1.38\mathrm{e}{-4}$ ($3.0\mathrm{e}{-6}$) & $1.33\mathrm{e}{-4}$ ($1.6\mathrm{e}{-6}$)& $1.28 \mathrm{e}{-4}$ ($7.5\mathrm{e}{-7}$)\\
Deep hedging & $1.06\mathrm{e}{-4}$ ($5.7\mathrm{e}{-6}$) & $8.88\mathrm{e}{-5}$ ($3.8\mathrm{e}{-6}$) & $7.71\mathrm{e}{-5}$ ($2.4\mathrm{e}{-6}$) & $7.24\mathrm{e}{-5}$ ($1.4\mathrm{e}{-6}$)\\
Deep kernel hedging&$\mathbf{7.91{e}{-5}~(2.2{e}{-6})}$ & $\mathbf{7.11{e}{-5}~(1.2{e}{-6})}$& $\mathbf{6.88{e}{-5}~(8.0{e}{-7})}$& $\mathbf{6.75{e}{-5}~(4.9{e}{-7})}$\\
\bottomrule 
\end{tabular}
\end{minipage}
\\ \vspace{0.25cm}
\begin{minipage}{1\textwidth}
\centering
\caption*{$T=1/2$}
\begin{tabular}{lcccc}
\toprule
$N$& $250$ & $500$ &$1000$ &$2500$\\
\midrule
BS & $2.40\mathrm{e}{-4}$ &$2.40\mathrm{e}{-4}$ & $2.40\mathrm{e}{-4}$& $2.40\mathrm{e}{-4}$\\
Kernel hedging & $3.21\mathrm{e}{-4}$ ($1.3\mathrm{e}{-5}$) &  $2.97\mathrm{e}{-4}$ ($5.9\mathrm{e}{-6}$) & $2.87\mathrm{e}{-4}$ ($3.7\mathrm{e}{-6}$)& $2.65 \mathrm{e}{-4}$ ($1.7\mathrm{e}{-6}$)\\
Deep hedging & $2.04\mathrm{e}{-4}$ ($9.9\mathrm{e}{-6}$) & $1.79\mathrm{e}{-4}$ ($9.1\mathrm{e}{-6}$) & $1.59\mathrm{e}{-4}$ ($6.8\mathrm{e}{-6}$) & $1.43\mathrm{e}{-4}$ ($3.9\mathrm{e}{-6}$)\\
Deep kernel hedging& $\mathbf{1.62{e}{-4}~ (3.8{e}{-6})}$ & $\mathbf{1.44{e}{-4}~(2.0{e}{-6})}$& $\mathbf{1.37{e}{-4}~(2.3{e}{-6})}$ & $\mathbf{1.31{e}{-4}~(5.8{e}{-7}})$\\
\bottomrule 
\end{tabular}
\end{minipage}

\end{table}

\begin{table}[h!]
\small
\centering
\caption{Out-of-sample metrics for the ATM European call for quadratic hedging with $N=1000$. Reported values are averages over ten seeds, with standard deviations in parentheses.}
\label{tab:other_metrics_set_1}
\begin{minipage}{1\textwidth}
\centering
\caption*{$T=1/12$}
\begin{tabular}{lcccc}
\toprule
 & MSE & MAE &$q^{\text{upper}}_{0.95}$ & CVaR$^{\text{upper}}_{0.95}$\\
\midrule
BS  & $2.56 \mathrm{e}{-5}$& ${4.01\mathrm{e}{-3}}$ & $8.96\mathrm{e}{-3}$ & $1.21\mathrm{e}{-2}$ \\
Kernel hedging & $3.94\mathrm{e}{-5}~(7.2 \mathrm{e}{-7})$& $4.79 \mathrm{e}{-3}$ ($4.5\mathrm{e}{-5}$) & $1.12 \mathrm{e}{-2}$ ($1.2 \mathrm{e}{-4}$) & $1.53\mathrm{e}{-2}$ ($1.6 \mathrm{e}{-4}$) \\
Deep hedging& $2.33\mathrm{e}{-5}~(3.4 \mathrm{e}{-7})$ & $3.66 \mathrm{e}{-3}$ ($3.6 \mathrm{e}{-5})$& $8.28 \mathrm{e}{-3}$ ($7.0\mathrm{e}{-5}$)& $1.16 \mathrm{e}{-2}$ ($9.2\mathrm{e}{-5})$\\
Deep kernel hedging& $\mathbf{2.14{e}{-5}~(1.5 {e}{-7})}$ & $\mathbf{3.51 {e}{-3}~(7.3{e}{-6})}$ & $\mathbf{8.05{e}{-3}~(9.0{e}{-5})}$ & $\mathbf{1.13{e}{-2}~(1.1 {e}{-4})}$ \\
\bottomrule 
\end{tabular}
\end{minipage}
\\ \vspace{0.25cm}
\begin{minipage}{1\textwidth}
\centering
\caption*{$T=1/4$}
\begin{tabular}{lcccc}
\toprule
& MSE & MAE &$q^{\text{upper}}_{0.95}$ & CVaR$^{\text{upper}}_{0.95}$\\
\midrule
BS  & $1.07\mathrm{e}{-4}$& $8.28\mathrm{e}{-3}$ & $1.95 \mathrm{e}{-2}$ & $2.49\mathrm{e}{-2}$ \\
Kernel hedging & $1.33\mathrm{e}{-4}~(1.6 \mathrm{e}{-6})$& $8.58 \mathrm{e}{-3}$ ($7.7\mathrm{e}{-5}$) & $1.99 \mathrm{e}{-2}$ ($1.3 \mathrm{e}{-4}$) & $2.78\mathrm{e}{-2}$ ($1.7 \mathrm{e}{-4}$) \\
Deep hedging& $7.71\mathrm{e}{-5}~(2.4 \mathrm{e}{-6})$& $6.56 \mathrm{e}{-3}$ ($9.5\mathrm{e}{-5})$& $1.54\mathrm{e}{-2}$ ($3.5\mathrm{e}{-4}$)& $2.16 {e}{-2}~(4.0{e}{-4})$\\
Deep kernel hedging& $\mathbf{6.87{e}{-5}~(8.0 {e}{-7})}$& $\mathbf{6.29 {e}{-3}~(2.4{e}{-5})}$ & $\mathbf{1.46{e}{-2}~(1.3{e}{-4})}$ & $\mathbf{2.06{e}{-2} (2.5 {e}{-4})}$ \\
\bottomrule 
\end{tabular}
\end{minipage}
\\ \vspace{0.25cm}
\begin{minipage}{1\textwidth}
\centering
\caption*{$T=1/2$}
\begin{tabular}{lcccc}
\toprule
 & MSE &MAE &$q^{\text{upper}}_{0.95}$ & CVaR$^{\text{upper}}_{0.95}$\\
\midrule
BS  & $2.40\mathrm{e}{-4}$& $1.22\mathrm{e}{-2}$ & $3.08\mathrm{e}{-2}$ & $3.98\mathrm{e}{-2}$ \\
Kernel hedging& $2.87\mathrm{e}{-4}~(3.7 \mathrm{e}{-6})$ & $1.24 \mathrm{e}{-2}$ ($1.3\mathrm{e}{-4}$) & $2.92 \mathrm{e}{-2}$ ($2.3 \mathrm{e}{-4}$) & $4.31\mathrm{e}{-2}$ ($2.9\mathrm{e}{-4}$) \\
Deep hedging& $1.59\mathrm{e}{-4}~(6.8 \mathrm{e}{-6})$&$9.34 \mathrm{e}{-3}$ ($1.6 \mathrm{e}{-4})$& $2.24 \mathrm{e}{-2}$ ($4.9\mathrm{e}{-4}$)& $3.15 \mathrm{e}{-2}$ ($7.6\mathrm{e}{-4})$\\
Deep kernel hedging& $\mathbf{1.37{e}{-4}~(2.3 {e}{-6})}$& $\mathbf{8.68 {e}{-3}~(8.1{e}{-5})}$ & $\mathbf{2.11{e}{-2}~(2.3{e}{-4)}}$ & $\mathbf{3.02{e}{-2}~(4.6 {e}{-4})}$\\
\bottomrule 
\end{tabular}
\end{minipage}
\end{table}

We first focus on quadratic hedging. Table~\ref{tab:loss_set_1} reports the out-of-sample MSE across maturities and training sample sizes. Table~\ref{tab:other_metrics_set_1} reports additional out-of-sample metrics across maturities, with a fixed training sample size of $N=1000$. Several patterns can be observed from the numerical experiments. First, increasing the number of training paths consistently improves performance across all learning-based methods, leading to lower out-of-sample hedging MSE. The improvement is particularly significant in the low-data regime, where additional training paths lead to substantial gains in hedging performance. Overall, deep kernel hedging achieves the lowest MSE across all considered maturities and training sample sizes, and this improvement can also be seen in the other out-of-sample metrics (MAE, $q^{\text{upper}}_{0.95}$, $\text{CVaR}^{\text{upper}}_{0.95}$) reported in Table~\ref{tab:other_metrics_set_1}. These results already highlight the importance of the learning method in the low-data regime, where the choice of method can have a substantial impact on out-of-sample performance. \\

Second, comparing deep kernel hedging with standard deep hedging shows that the proposed approach systematically outperforms the classical deep hedging method across all considered maturities and training sample sizes. The improvement is particularly pronounced for smaller training sets, while the performance gap progressively narrows as more training paths become available. Nevertheless, deep kernel hedging remains superior in all considered settings. This suggests that the kernel component provides a useful inductive bias that is particularly beneficial when training data are limited, while the neural network part learns a flexible representation of the market state.  \\

Third, the comparison between classical kernel hedging and deep kernel hedging highlights the importance of learning the representation of the state variables. The classical kernel approach performs poorly and even underperforms the Black--Scholes benchmark in several settings, indicating that applying the RBF kernel directly to the original input space does not provide a suitable geometry for the hedging problem. By contrast, deep kernel hedging jointly learns a nonlinear representation of the input space and the kernel-based hedging strategy, resulting in a substantial and consistent improvement in performance. This comparison suggests that the main benefit of the proposed approach does not simply come from the use of a kernel, but rather from combining the kernel structure with a learned latent representation, as confirmed in Appendix \ref{sec:appendix_synthetic_data}. \\

Fourth, we examine the variability of the different methods across random seeds. Deep kernel hedging tends to have a lower standard deviation than the other learning-based approaches. This is particularly significant since, in addition to the random initialization of the neural network parameters, the deep kernel method involves an additional source of randomness through the sampling of the random Fourier features. Different realizations can lead to different feature representations and consequently to greater variability in the resulting hedging strategies. In contrast, the learned representation is adapted during training to the fixed random-feature map, which may mitigate part of the variability induced by different RFF realizations. The lower dispersion observed across seeds therefore suggests that the learned representation not only improves hedging performance, but also contributes to making the method more robust across different training runs.\\

\begin{table}[!ht]
\small
\centering
\caption{Out-of-sample metrics for the ATM European call for CVaR hedging with $N=1000$ and $v_0=\text{BS}(S_0,K,T,r,\sigma=\sqrt{\nu_0})$. Reported values are averages over ten seeds, with standard deviations in parentheses.}
\label{tab:other_metrics_set_1_cvar}
\begin{minipage}{1\textwidth}
\centering
\caption*{$T=1/12$}
\begin{tabular}{lcccc} \toprule & MSE & MAE & $q^{\text{upper}}_{0.95}$ & CVaR$^{\text{upper}}_{0.95}$\\ 
\midrule
BS & $\mathbf{2.56{e}{-5}}$ & $\mathbf{4.01{e}{-3}}$ & $8.96\mathrm{e}{-3}$ & $1.21\mathrm{e}{-2}$  \\ 
Kernel hedging & $5.02\mathrm{e}{-5}$ (1.2$\mathrm{e}{-6}$) & $5.56\mathrm{e}{-3}$ (8.1$\mathrm{e}{-5}$) & $1.19\mathrm{e}{-2}$ (1.1$\mathrm{e}{-4}$) & $1.56\mathrm{e}{-2}$ (1.1$\mathrm{e}{-4}$) \\
Deep hedging & $3.12\mathrm{e}{-5}$ (1.2$\mathrm{e}{-6}$) & $4.47\mathrm{e}{-3}$ (1.1$\mathrm{e}{-4}$) & $8.81\mathrm{e}{-3}$ (1.8$\mathrm{e}{-4}$) & $1.16\mathrm{e}{-2}$ (1.6$\mathrm{e}{-4}$) \\ 
Deep kernel hedging & ${2.93\mathrm{e}{-5}}$ {(1.2$\mathrm{e}{-6}$)} & {$4.28\mathrm{e}{-3}$ (1.1$\mathrm{e}{-4}$)} & $\mathbf{8.47{e}{-3}~(7.2{e}{-5})}$ & $\mathbf{1.13{e}{-2} ~(1.2{e}{-4})}$ \\ 
\bottomrule \end{tabular}

\end{minipage}
\\ \vspace{0.25cm}
\begin{minipage}{1\textwidth}
\centering
\caption*{$T=1/4$}
\begin{tabular}{lcccc} 
\toprule
& MSE & MAE & $q^{\text{upper}}_{0.95}$ & CVaR$^{\text{upper}}_{0.95}$ \\ 
\midrule 
BS & $1.07\mathrm{e}{-4}$ & $8.28\mathrm{e}{-3}$ & $1.95\mathrm{e}{-2}$ & $2.49\mathrm{e}{-2}$ \\ 
Kernel hedging & $1.63\mathrm{e}{-4}$ (6.9$\mathrm{e}{-6}$) & $9.93\mathrm{e}{-3}$ (3.3$\mathrm{e}{-4}$) & $2.28\mathrm{e}{-2}$ (1.6$\mathrm{e}{-4}$) & $3.06\mathrm{e}{-2}$ (1.7$\mathrm{e}{-4}$)\\ 
Deep hedging & $1.06\mathrm{e}{-4}$ (2.6$\mathrm{e}{-6}$) & $8.25\mathrm{e}{-3}$ (1.0$\mathrm{e}{-4}$) & $1.73\mathrm{e}{-2}$ (3.6$\mathrm{e}{-4}$) & $2.30\mathrm{e}{-2}$ (4.0$\mathrm{e}{-4}$) \\ 
Deep kernel hedging & $\mathbf{9.92{e}{-5}~(4.0{e}{-6})}$ & $\mathbf{8.02{e}{-3}~(2.2{e}{-4})}$ & $\mathbf{1.68{e}{-2} ~(1.3{e}{-4})}$ & $\mathbf{2.21{e}{-2}~(2.5{e}{-4}})$ \\ \bottomrule \end{tabular}
\end{minipage}
\\ \vspace{0.25cm}
\begin{minipage}{1\textwidth}
\centering
\caption*{$T=1/2$}
\begin{tabular}{lcccc} 
\toprule 
& MSE & MAE & $q^{\text{upper}}_{0.95}$ & CVaR$^{\text{upper}}_{0.95}$ \\ 
\midrule 
BS & $2.40\mathrm{e}{-4}$ & $1.22\mathrm{e}{-2}$ & $3.08\mathrm{e}{-2}$ & $3.98\mathrm{e}{-2}$ \\ 
Kernel hedging & $3.62\mathrm{e}{-4}$ (1.7$\mathrm{e}{-5}$) & $1.49\mathrm{e}{-2}$ (5.5$\mathrm{e}{-4}$) & $3.56\mathrm{e}{-2}$ (2.5$\mathrm{e}{-4}$) & $4.65\mathrm{e}{-2}$ (2.4$\mathrm{e}{-4}$) \\ 
Deep hedging & $2.41\mathrm{e}{-4}$ (7.0$\mathrm{e}{-6}$) & $1.26\mathrm{e}{-2}$ (2.4$\mathrm{e}{-4}$) & $2.80\mathrm{e}{-2}$ (4.2$\mathrm{e}{-4}$) & $3.60\mathrm{e}{-2}$ (6.2$\mathrm{e}{-4}$) \\ 
Deep kernel hedging & $\mathbf{2.10{e}{-4}~(7.4{e}{-6})}$ & $\mathbf{1.15{e}{-2}~(3.4{e}{-4})}$ & $\mathbf{2.67{e}{-2} ~(1.6{e}{-4})}$ & $\mathbf{3.47{e}{-2}~(3.4{e}{-4})}$ \\ 
\bottomrule 
\end{tabular}
\end{minipage}
\vspace{0.4mm}
\end{table}
All the above results are confirmed by Figure \ref{fig:quadratic_loss_training}, which compares the convergence of the three learned hedging methods in terms of the out-of-sample hedging MSE. The left plot reports the MSE as a function of the training epoch, whereas the right plot reports it as a function of cumulative training time.\ Each curve shows the mean over 10 random seeds, with shaded regions representing $\pm$ one\hspace{-0.2mm} standard\hspace{-0.2mm} deviation.\ The \hspace{-0.2mm}Black--Scholes\hspace{-0.2mm} MSE\hspace{-0.2mm} is\hspace{-0.2mm} shown\hspace{-0.2mm} as\hspace{-0.2mm} a\hspace{-0.2mm} horizontal benchmark\hspace{-0.2mm} as\hspace{-0.2mm} it\hspace{-0.2mm} does\hspace{-0.2mm} not\hspace{-0.2mm} require\hspace{-0.2mm} training. 
\begin{figure}[!ht]
    \centering

    \begin{subfigure}[t]{0.502\linewidth}
        \centering
        \includegraphics[width=\linewidth]
            {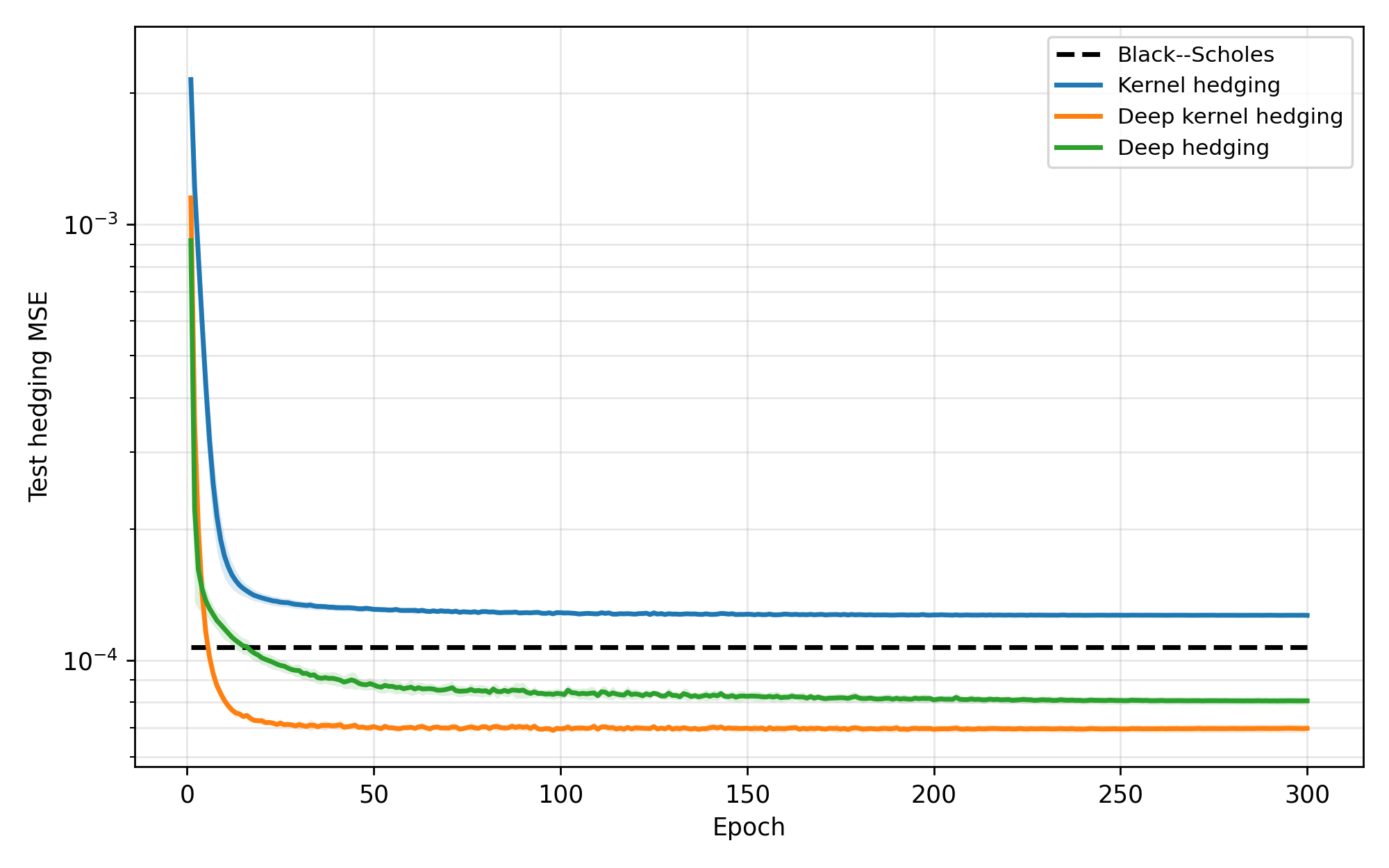}
        \caption{MSE on test set versus training epoch.}
        \label{fig:quadratic_loss_epoch}
    \end{subfigure} \hspace{-4mm}
    \hfill
    \begin{subfigure}[t]{0.502\linewidth}
        \centering
        \includegraphics[width=\linewidth]
            {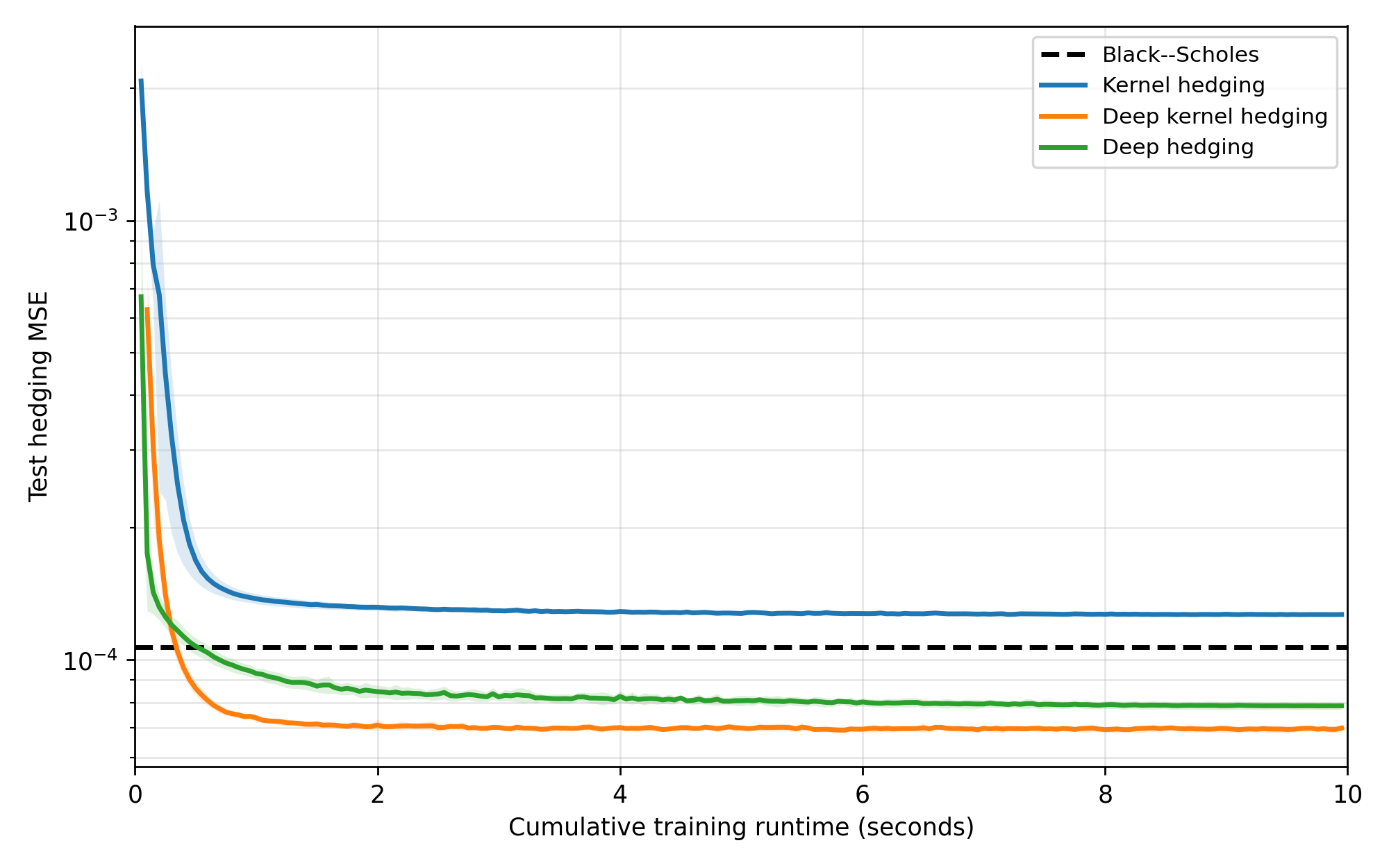}
        \caption{MSE on test set versus cumulative runtime.}
        \label{fig:quadratic_loss_runtime}
    \end{subfigure}

    \caption{Hedging MSE $\sum_{i=1}^{N^\text{OOS}} \mathcal{E}^2_i(\widehat{\phi}_D) /N^\text{OOS}$ during training on a test set of $N^\text{OOS}=25,000$ synthetic paths of an ATM European call with $T=1/4$ and $N=1000$. The curves and shaded regions show the mean and one standard deviation, respectively, over 10 seeds.}
    \label{fig:quadratic_loss_training}
\end{figure}
\\ 

Finally, to assess hedging performance under a tail-risk objective, we also consider the CVaR hedging problem \eqref{eq:deep_kernel_cvar_hedging} with $\tilde\alpha=5\%$ for ATM European call options. Table~\ref{tab:other_metrics_set_1_cvar} reports the corresponding out-of-sample results for $N=1000$. Overall, the conclusions are consistent with those obtained under quadratic hedging: deep kernel hedging systematically outperforms the other learning-based approaches in terms of CVaR, while also achieving the best (or near-best) performance according to the other reported metrics. This again demonstrates the benefit of the learned representation, also when the hedging strategy is directly optimized for tail risk. Moreover, the deep kernel approach generally exhibits a lower standard deviation across random seeds, indicating a stable learning procedure despite the additional randomness introduced by the RFF approximation. Note that a direct comparison of the absolute CVaR values between quadratic and CVaR hedging does not make sense. In the CVaR setting, the initial portfolio value $v_0$ is fixed to the Black--Scholes price, whereas it is learned under quadratic hedging. %In particular, the initial price learned under quadratic hedging may be higher than the corresponding Black--Scholes price, which shifts the resulting hedging P\&L upwards and can mechanically lead to a less negative CVaR. 
Therefore, the absolute CVaR values obtained under the two hedging objectives are not directly comparable as they correspond to different initial portfolio values. Overall, these results illustrate the flexibility of our framework: the loss function can be selected according to the performance criterion of interest, allowing the learned strategy to prioritize either average hedging accuracy or protection against extreme P\&L downside risk.\\

Overall, within the considered Heston framework and for the neural architectures selected in this study, the numerical results indicate that the deep kernel approach offers a competitive and robust alternative to standard deep hedging by effectively combining the advantages of both deep and kernel methodologies. Naturally, these conclusions remain conditional on the specific experimental setting under consideration, and different market dynamics or architectures may lead to different relative performances.

\subsection{Results on real data: S\&P500 time series} \label{subsec:num_real}Finally, we conclude this paper by focusing on real data of S\&P500 time series. We construct our dataset using S\&P500 historical data over the period from January 1st, 2010 to May 31st, 2026. The data from January 1st, 2010 to December 31st, 2022 is used for training, while the period from January 1st, 2023 to May 31st, 2026 is reserved for out-of-sample evaluation. For a fixed maturity $T_\text{days}$ (in days in this section), we extract overlapping trajectories using a sliding-window procedure (rolling windows with overlap), which yields a collection of market paths observed at discrete trading dates $0=t_0 < t_1 < \cdots < t_n = T_\text{days}$. In addition to the spot price process $S_t$, we compute a realized volatility proxy $\sqrt{\nu_t}$ using a one-month (20 trading days) rolling window, and fix the interest rate $r$ to zero. The resulting dataset is therefore given by
\[
\mathcal D_{N_T}
=
\Big\{
\left(\left(S_{t_k}^i/S_{t_0}^i, \nu_{t_k}^i\right)_{k=0}^n, \, Y_T^i\right)
\Big\}_{i=1}^{N_T},
\]
where $N_T$ denotes the number of overlapping trajectories associated with maturity $T_\text{days}$. For the training set, we have $N_T = 3262-T_\text{days}$ sample paths and for the out-of-sample evaluation $N^\text{OOS}_T=834-T_\text{days}$. 

\begin{sqremark}
An alternative approach to increase the amount of training data is to first fit a stochastic or generative model (e.g., Heston, GARCH, or GAN) to the historical data and then use the fitted model to generate additional sample paths. This would effectively bring the problem back to the synthetic-data setting considered above. However, both our experiments and the empirical deep-hedging results reported in \cite[Table 6]{mikkila2023empirical} suggest that this approach performs worse than training directly on the observed data.
\end{sqremark}

As model inputs, we consider two different feature specifications to assess the impact of path-dependent information on hedging performance. In the first, baseline specification, we use the classical time-augmented state variables
\begin{equation} \label{vanilla_feat}
X_{t_k} = (t_k, S_{t_k}/S_{t_0}, \nu_{t_k}),
\quad k = 0,\dots,n, \tag{Vanilla features}
\end{equation}
which provide a standard Markovian representation of the market state. However, we find that the resulting hedging performance is relatively poor across the considered methods. We therefore report these baseline results in Appendix~\ref{sec: appendix_SP500} and focus in the main text on a richer feature representation based on time-augmented path-signatures. %Path-signatures provide a systematic and tractable way of encoding path-dependent information without imposing a specific parametric structure on the underlying market dynamics. Originally introduced by \cite{chen1957integration}, they have recently been extensively used in mathematical finance, including in hedging applications; see, e.g., \cite{lyons2020non,abi2025signature,abi2025hedging,cirone2025rough,jaber2025signature}. For completeness, a short reminder of path-signatures is provided in Appendix~\ref{appendix:sig}. In the main experiments, we therefore use time-augmented signature features.
More precisely, we consider the time-augmented signature of order $M=3$ of the path $Z:=(S/S_{t_0},\nu)$ and define
\[
X_{t_k}=\hat{\mathbb{Z}}_{t_k}^{\leq 3},
\quad k = 0,\dots,n, \tag{Signature features}
\]
where $\hat{\mathbb{Z}}$ denotes the time-augmented signature of $Z$, see Definition~\ref{def:sig}. As discussed in Appendix \ref{appendix:sig}, this provides a richer representation of the observed market dynamics and allows the learning methods to exploit path-dependent information beyond the instantaneous state variables without imposing a specific parametric structure on the underlying market dynamics. Numerically, these signature features are computed using the Python library \texttt{iisignature}. Since path-signatures depend only on the increments of a path, they are invariant under translations and therefore do not capture its absolute level. This can be problematic in our empirical application, where the data streams have different initial conditions. To address this issue, we use the base-point augmentation introduced in \cite[Section 3.1.4]{chevyrev2025primer}, which consists in appending the origin $0$ to the beginning of each data stream and thereby removes translation invariance. Throughout the paper, signatures are hence understood to be computed after this base-point augmentation.\\

In this real (scarce) dataset setting, we consider the quadratic and CVaR hedging of ATM European and Asian call options for different maturities,
\begin{equation*}
T_\text{days}\in \{10 \text{ days},~ 20 \text{ days},~ 40 \text{ days}\}.
\end{equation*}
As in the experiments on synthetic data, we focus on ATM European calls in the main text, while the corresponding results for the Asian payoff are reported in Appendix \ref{sec: appendix_SP500}. \\

Moreover, following \cite{buehler2019deep}, we do not learn the initial price $v_0$, even for the quadratic hedging problem. Instead, for each trajectory, we fix
\[
v_0^i = \text{BS}\big(S_0=S_0^i,K=S_0^i,T,\sigma=\sqrt{\nu_0^i}\big),
\]
where $\nu_0^i$ denotes the initial realized variance of the $i$-th path and $\text{BS}(S,K,T,\sigma)$ refers to the Black-Scholes price of a European call option. This choice is motivated by the fact that real market trajectories do not start from a common initial market state. As a consequence, it is natural that the option price at time zero depends on the prevailing market conditions, and in particular on the initial volatility level. In contrast to the synthetic data setting where $v_0$ can be treated as a single parameter to learn, we therefore fix here $v_0^i$ pathwise and do not learn it from data. More generally, one could consider learning the pricing functional as a function of the initial market state (in particular $\nu_0^i$), but our focus in this paper is mainly on learning the hedging strategy rather than the optimal pricing rule. Finally, we benchmark our learned strategies against the classical Black-Scholes practitioner delta, where the hedge is computed using the Black-Scholes delta evaluated with the realized volatility.\\

We keep the same architectures as in the synthetic data experiments for all learning-based methods, but use a smaller  and constant learning rate of $1\mathrm{e}{-4}$, to mitigate overfitting in the noisier real-market setting. Unlike the synthetic experiments, where all simulated paths start from the same initial spot price $S_0$, the real-data paths are associated with different initial spot prices $S_0^i$. We therefore normalize the terminal hedging error by the corresponding initial spot price both during training and evaluation, i.e.\ we use $\mathcal{E}_i^{\text{norm}}(\phi) := \mathcal{E}_i(\phi;v_0^i)/ S_0^i $. Specifically, the hedging strategies are trained by minimizing losses defined on the normalized hedging error $\mathcal{E}_i^{\text{norm}}(\phi)$. Similarly, all out-of-sample terminal hedging errors are normalized by their respective initial spot prices before computing the reported performance metrics, making comparisons across periods with different market levels meaningful. \\

\begin{table}[!ht]
\small
\centering
\caption{Out-of-sample metrics for the quadratic hedging of an ATM European call on real dataset with time-augmented signature features. P\&Ls are scaled by the underlying price at the beginning of the periods. Reported values are averages over ten seeds, with standard deviations in parentheses.}
\label{tab:loss_SP500_sign}

\begin{minipage}{1\textwidth}
\centering
\caption*{$T_\text{days}=$ 10 days.}
\begin{tabular}{lcccc} 
\toprule & MSE & MAE & $q^{\text{upper}}_{0.95}$ & CVaR$^{\text{upper}}_{0.95}$ \\ 
\midrule 
BS & $3.12\mathrm{e}{-5}$ & $3.48\mathrm{e}{-3}$ & $9.13\mathrm{e}{-3}$ & $1.41\mathrm{e}{-2}$ \\ 
Kernel hedging & $2.64\mathrm{e}{-5}$ (1.5$\mathrm{e}{-6}$) & $3.83\mathrm{e}{-3}$ (1.3$\mathrm{e}{-4}$) & $8.61\mathrm{e}{-3}$ (4.2$\mathrm{e}{-4}$) & $1.19\mathrm{e}{-2}$ (3.8$\mathrm{e}{-4}$)\\ 
Deep hedging & $2.24\mathrm{e}{-5}$ (2.4$\mathrm{e}{-6}$) & $3.51\mathrm{e}{-3}$ (1.2$\mathrm{e}{-4}$) & $7.17\mathrm{e}{-3}$ (1.7$\mathrm{e}{-4}$) & $1.02\mathrm{e}{-2}$ (6.1$\mathrm{e}{-4}$) \\ 
Deep kernel hedging & $\mathbf{2.06{e}{-5} \, (1.2{e}{-6})}$ & $\mathbf{3.33{e}{-3} \, (4.9{e}{-5}})$ & $\mathbf{6.76{e}{-3} \, (2.6{e}{-4}})$ & $\mathbf{1.00{e}{-2} \, (4.5{e}{-4}})$ \\ 
\bottomrule \end{tabular}
\end{minipage}
\\ \vspace{0.25cm}

\begin{minipage}{1\textwidth}
    \centering 
    \caption*{$T_\text{days}=$ 20 days.}
\begin{tabular}{lcccc}
\toprule & MSE & MAE & $q^{\text{upper}}_{0.95}$ & CVaR$^{\text{upper}}_{0.95}$\\ 
\midrule 
BS & $5.64\mathrm{e}{-5}$ & $4.83\mathrm{e}{-3}$ & $1.17\mathrm{e}{-2}$ & $1.73\mathrm{e}{-2}$ \\ 
Kernel hedging & $4.50\mathrm{e}{-5}$ (4.4$\mathrm{e}{-6}$) & $4.81\mathrm{e}{-3}$ (2.2$\mathrm{e}{-4}$) & $1.12\mathrm{e}{-2}$ (4.8$\mathrm{e}{-4}$) & $1.71\mathrm{e}{-2}$ (9.7$\mathrm{e}{-4}$) \\ 
Deep hedging & $3.46\mathrm{e}{-5}$ (4.3$\mathrm{e}{-6}$) & $4.18\mathrm{e}{-3}$ (1.4$\mathrm{e}{-4}$) & $8.63\mathrm{e}{-3}$ (3.7$\mathrm{e}{-4}$) & $1.20\mathrm{e}{-2}$ (8.8$\mathrm{e}{-4}$) \\ 
Deep kernel hedging & $\mathbf{3.13{e}{-5} \, (2.8{e}{-6}})$ & $\mathbf{3.90{e}{-3} \, (1.2{e}{-4}})$ & $\mathbf{8.34{e}{-3} \, (3.7{e}{-4}})$ & $\mathbf{1.13{e}{-2} \, (6.8{e}{-4}})$ \\ 
\bottomrule 
\end{tabular}
\end{minipage}
\\ \vspace{0.25cm}

\begin{minipage}{1\textwidth}
\centering
\caption*{$T_\text{days}=$ 40 days.}
\begin{tabular}{lcccc} 
\toprule 
& MSE & MAE & $q^{\text{upper}}_{0.95}$ & CVaR$^{\text{upper}}_{0.95}$\\ 
\midrule 
BS & $9.75\mathrm{e}{-5}$ & $6.23\mathrm{e}{-3}$ & $1.36\mathrm{e}{-2}$ & $1.76\mathrm{e}{-2}$\\ 
Kernel hedging & $7.29\mathrm{e}{-5}$ (1.2$\mathrm{e}{-5}$) & $6.20\mathrm{e}{-3}$ (3.6$\mathrm{e}{-4}$) & $1.62\mathrm{e}{-2}$ (2.0$\mathrm{e}{-3}$) & $2.30\mathrm{e}{-2}$ (3.7$\mathrm{e}{-3}$) \\ 
Deep hedging & $4.35\mathrm{e}{-5}$ (5.3$\mathrm{e}{-6}$) & $4.75\mathrm{e}{-3}$ (1.9$\mathrm{e}{-4}$) & $1.08\mathrm{e}{-2}$ (9.3$\mathrm{e}{-4}$) & $1.64\mathrm{e}{-2}$ (2.3$\mathrm{e}{-3}$) \\ 
Deep kernel hedging & $\mathbf{3.55{e}{-5} \, (4.2{e}{-6}})$ & $\mathbf{4.08{e}{-3} \, (1.9{e}{-4}})$ & $\mathbf{8.74{e}{-3} \, (1.0{e}{-3}})$ & $\mathbf{1.36{e}{-2} \, (2.2{e}{-3}})$ \\ 
\bottomrule 
\end{tabular}
\end{minipage}
\end{table}

Table \ref{tab:loss_SP500_sign} reports the out-of-sample performance of the different learning-based methods for the quadratic hedging problem using the time-augmented signature features across different maturities. For the three maturities considered, deep kernel hedging achieves the best overall hedging performance, yielding the lowest MSE and MAE, but also tail risk measures $q^{\text{upper}}_{0.95}$ and $\text{CVaR}^{\text{upper}}_{0.95}$ (which are not directly minimized during training). The aforementioned complementary roles of the kernel and neural network components are therefore particularly beneficial in the real-data setting considered here, where the available training samples are limited and the rolling-window trajectories are highly correlated. Moreover, as in the synthetic-data experiments, deep kernel hedging exhibits greater stability across random seeds, as reflected by the lower standard deviations observed across most of the reported metrics.  \\

\begin{table}[!ht]
\small
\centering
\caption{Out-of-sample metrics for the CVaR hedging of an ATM European call on real dataset with time-augmented signature features. P\&Ls are scaled by the underlying price at the beginning of the periods. Reported values are averages over ten seeds, with standard deviations in parentheses.}
\label{tab:loss_SP500_sign_CVaR_hedging}

\begin{minipage}{1\textwidth}
\centering
\caption*{$T_\text{days}=$ 10 days.}
\begin{tabular}{lcccc} 
\toprule 
& MSE & MAE & $q^{\text{upper}}_{0.95}$ & CVaR$^{\text{upper}}_{0.95}$ \\ 
\midrule 
BS & $3.12\mathrm{e}{-5}$ & $\mathbf{3.48{e}{-3}}$ & $9.13\mathrm{e}{-3}$ & $1.42\mathrm{e}{-2}$ \\ 
Kernel hedging & $3.04\mathrm{e}{-5}$ (2.3$\mathrm{e}{-6}$) & $4.14\mathrm{e}{-3}$ (1.7$\mathrm{e}{-4}$) & $8.90\mathrm{e}{-3}$ (3.6$\mathrm{e}{-4}$) & $1.20\mathrm{e}{-2}$ (5.2$\mathrm{e}{-4}$) \\
Deep hedging & $2.84\mathrm{e}{-5}$ (2.4$\mathrm{e}{-6}$) & $4.02\mathrm{e}{-3}$ (1.0$\mathrm{e}{-4}$) & $7.84\mathrm{e}{-3}$ (2.9$\mathrm{e}{-4}$) & $1.00\mathrm{e}{-2}$ (7.9$\mathrm{e}{-4}$) \\ 
Deep kernel hedging & $\mathbf{2.81{e}{-5} \, (1.3{e}{-6}})$ & ${4.03\mathrm{e}{-3}}$ (6.8$\mathrm{e}{-5}$) & $\mathbf{7.01{e}{-3} \,  (2.4{e}{-4}})$ & $\mathbf{9.19{e}{-3} \, (4.8{e}{-4}})$ \\ 
\bottomrule \end{tabular}
\end{minipage}
\\ \vspace{0.25cm}

\begin{minipage}{1\textwidth}
    \centering 
    \caption*{$T_\text{days}=$ 20 days.}
\begin{tabular}{lcccc} 
\toprule & MSE & MAE & $q^{\text{upper}}_{0.95}$ & CVaR$^{\text{upper}}_{0.95}$ \\ 
\midrule
BS & $5.64\mathrm{e}{-5}$ & $4.83\mathrm{e}{-3}$ & $1.17\mathrm{e}{-2}$ & $1.73\mathrm{e}{-2}$ \\ 
Kernel hedging & $5.22\mathrm{e}{-5}$ (6.4$\mathrm{e}{-6}$) & $5.38\mathrm{e}{-3}$ (3.8$\mathrm{e}{-4}$) & $1.13\mathrm{e}{-2}$ (6.6$\mathrm{e}{-4}$) & $1.63\mathrm{e}{-2}$ (9.7$\mathrm{e}{-4}$) \\
Deep hedging & $4.98\mathrm{e}{-5}$ (6.0$\mathrm{e}{-6}$) & $4.91\mathrm{e}{-3}$ (1.4$\mathrm{e}{-4}$) & $\mathbf{8.30{e}{-3} \, (2.3{e}{-4}})$ & $\mathbf{1.04{e}{-2} \, (1.0{e}{-3})}$ \\ 
Deep kernel hedging & $\mathbf{4.09{e}{-5} \,  (4.2{e}{-6}})$ & $\mathbf{4.68{e}{-3} \, (2.1{e}{-4}})$ & ${8.42\mathrm{e}{-3}}$ (2.3$\mathrm{e}{-4}$) & ${1.08\mathrm{e}{-2}}$ (4.7$\mathrm{e}{-4}$) \\ 
\bottomrule \end{tabular}
\end{minipage}
\\ \vspace{0.25cm}

\begin{minipage}{1\textwidth}
\centering
\caption*{$T_\text{days}=$ 40 days.}
\begin{tabular}{lcccc} 
\toprule & MSE & MAE & $q^{\text{upper}}_{0.95}$ & CVaR$^{\text{upper}}_{0.95}$ \\ 
\midrule
BS & ${9.75\mathrm{e}{-5}}$ & ${6.23\mathrm{e}{-3}}$ & ${1.36\mathrm{e}{-2}}$ & ${1.76\mathrm{e}{-2}}$ \\ 
Kernel hedging & $1.08\mathrm{e}{-4}$ (1.7$\mathrm{e}{-5}$) & $8.04\mathrm{e}{-3}$ (7.0$\mathrm{e}{-4}$) & $1.74\mathrm{e}{-2}$ (1.2$\mathrm{e}{-3}$) & $2.42\mathrm{e}{-2}$ (2.8$\mathrm{e}{-3}$) \\ 
Deep hedging & $1.53\mathrm{e}{-4}$ (3.4$\mathrm{e}{-5}$) & $7.36\mathrm{e}{-3}$ (3.3$\mathrm{e}{-4}$) & $9.10\mathrm{e}{-3}$ (7.4$\mathrm{e}{-4}$) & $\mathbf{1.10{e}{-2} \, (1.1{e}{-3})}$\\ 
Deep kernel hedging & $\mathbf{7.16{e}{-5} \, (9.1{e}{-6})}$ & $\mathbf{5.92{e}{-3} \, (2.1{e}{-4})}$ & $\mathbf{8.65\mathrm{e}{-3} \, (1.4{e}{-3})}$ & ${1.15\mathrm{e}{-2}}$ (1.8$\mathrm{e}{-3}$) \\ 
\bottomrule 
\end{tabular}
\end{minipage}
\end{table}

Finally, we consider CVaR hedging, with the out-of-sample results reported in Table~\ref{tab:loss_SP500_sign_CVaR_hedging}. Compared with Table~\ref{tab:loss_SP500_sign}, the CVaR-trained strategies exhibit higher MSE and MAE values, as they are no longer trained to minimize the quadratic hedging error. However, the  upper-tail metrics $\operatorname{CVaR}^{\text{upper}}_{0.95}$ are substantially improved\footnote{Since the methods are trained with the same initial prices, the comparison of the metrics between the hedging strategies optimizing quadratic or CVaR criteria is fair here, unlike in the synthetic-data setting.}. This is consistent with the CVaR training objective, which explicitly penalizes the severity of adverse hedging outcomes in the worst $5\%$ of the P\&L distribution. We note, however, that deep kernel hedging does not systematically outperform deep hedging in terms of CVaR in this empirical setting. These results should however be interpreted with some caution since the out-of-sample sample size is only $N^\text{OOS}_T=834-T_{\mathrm{days}}$. Consequently, the empirical out-of-sample  CVaR is estimated from only approximately $40$ tail observations (since $\alpha=5\%$), making this tail-risk metric subject to substantial sampling variability.  \\

Overall, these empirical results on real data further support deep kernel hedging as a robust and competitive alternative to standard deep hedging. In particular, the proposed approach consistently achieves strong hedging performance across different maturities. Taken together, these results suggest that the combination of a learned representation with the structured kernel-based hedging framework provides an effective and flexible approach for learning hedging strategies from real data.

\section{Conclusion}

\begin{comment}
These conclusions are specific to the real dataset (S\&P 500 index), and learning procedures considered in this study. Different market regimes, alternative assets, or modifications of the feature construction and training methodology would be helpful to strenghten the claims of the paper.
\\
\\
Although we focus on Fourier features based on the \texttt{cosine} map, our framework could accommodate straightforardly more general random nonlinear features. For instance, Gaussian random weights combined with a ReLU activation induce an arc-\texttt{cosine} kernel, while sigmoid or \texttt{Tanh} activations generate other neural-network kernels. Replacing the RFF layer with such random features would connect our approach directly to the broader class of random-feature neural networks, while retaining the learned representation $\psi_\theta$.
\end{comment}

In this paper, we introduce a deep kernel hedging framework that combines the flexibility of deep learning with the structural inductive bias of kernel methods. The hedging functional is restricted in a reproducing kernel Hilbert space whose kernel is parameterized through a neural network embedding of the input features. By means of a generalized representer theorem, we derive a finite-dimensional formulation of the original infinite-dimensional hedging problem and propose a random Fourier feature approximation with convergence guarantees that overcomes the computational limitations associated with large kernel matrices. \\

The numerical results on both synthetic and real-world data support the benefits of this approach. In particular, the proposed deep kernel method consistently achieves competitive (and in most settings superior) hedging performance compared with both classical kernel methods and standard deep hedging architectures. The advantage over standard deep hedging is especially relevant in low-data regimes, where the inductive bias inherited from the kernel appears to be particularly beneficial. Moreover, the proposed approach exhibits greater numerical stability with a reduced sensitivity to initialization and feature randomization across different seeds. \\%controlling the initialization of the neural network and the random Fourier features, the variability of the hedging performance is generally smaller than for the corresponding deep hedging architectures. This reduced sensitivity to initialization and randomization is an additional practical advantage when the amount of available data is limited. \\

The comparison with classical kernel methods and RFNNs suggests that the performance gains are not specific to the Fourier representation. Our numerical experiments indeed show that replacing the random Fourier features used for approximating the underlying base kernel by alternative random feature constructions leads to similar hedging performance. This confirms that the main benefit of the deep kernel approach lies in learning a data-dependent neural representation of the input space that is adapted to the hedging objective, while the specific choice of base kernel and  its associated random feature map appears to play a minor role. \\ 

Future work could extend the framework to multiple tradable assets and examine transaction costs and market impact, including their implications for convexity and computational tractability. Finally, when the learning dataset contains paths with different initial conditions, one could extend the framework to learn not only the hedging functional but also the initial pricing functional as a function of the initial market state. This would allow the learned model to jointly capture the dependence of the optimal hedging strategy and the corresponding optimal initial price on the observed market conditions.

\section{Appendix}

\subsection{Proofs}

\subsubsection{Proof of Lemma \ref{lem:existence_deep_kernel_hedging_rewritten}}
\label{proof1}
\begin{proof}
 For each \(\theta\in\Theta\), let
\(
\mathbf{A}_\theta
:=
\mathbf{Q}_\theta^{1/2}
\)
denote the unique positive-semidefinite square root of
\(\mathbf{Q}_\theta\). For fixed \(\theta\), introduce the variable
\[
\mathbf c
=
\mathbf{A}_\theta\bm{\alpha}.
\]
Since \(\mathbf{Q}_\theta=\mathbf{A}_\theta^2\), we have
\[
\mathbf{Q}_\theta\bm{\alpha}
=
\mathbf{A}_\theta \mathbf c
\qquad 
\text{ and } \qquad
\bm{\alpha}^{\top}\mathbf{Q}_\theta\bm{\alpha}
=
\lVert\mathbf{A}_\theta\bm{\alpha}\rVert_2^2
=
\lVert\mathbf c\rVert_2^2.
\]
As \(\bm{\alpha}\) ranges over \(\mathbb{R}^N\), the vector
\(\mathbf c\) ranges over
\(\operatorname{range}(\mathbf{A}_\theta)\). Hence, for every fixed
\(\theta\), from Theorem \ref{thm:representer_thm_inner_problem}, the inner problem in \eqref{eq: deep_kernel_hedging_rewritten} is equivalent to
\begin{equation*}
\min_{\mathbf c\in\operatorname{range}(\mathbf{A}_\theta)}
\left\{
\frac{1}{N}
\sum_{i=1}^N
\mathcal{L}\left(
(\mathbf{Y}^{v_0}
-
\mathbf{A}_\theta\mathbf c)_i
\right)
+
\lambda\lVert\mathbf c\rVert_2^2
\right\}.
\end{equation*}
Since
\(\mathbf{A}_\theta\) is symmetric,
\(
\mathbb{R}^N
=
\operatorname{range}(\mathbf{A}_\theta)
\oplus
\ker(\mathbf{A}_\theta).
\)
Thus, every \(\mathbf c\in\mathbb{R}^N\) admits a unique decomposition
\[
\mathbf c
=
\mathbf c_{\parallel}
+
\mathbf c_{\perp},
\]
where
\(
\mathbf c_{\parallel}
\in\operatorname{range}(\mathbf{A}_\theta)\) and 
\(\mathbf c_{\perp}
\in\ker(\mathbf{A}_\theta).
\)
This decomposition satisfies
\[
\mathbf{A}_\theta\mathbf c
=
\mathbf{A}_\theta\mathbf c_{\parallel}
\qquad
\text{and}
\qquad
\lVert\mathbf c\rVert_2^2
=
\lVert\mathbf c_{\parallel}\rVert_2^2
+
\lVert\mathbf c_{\perp}\rVert_2^2.
\]
Therefore, replacing \(\mathbf c\) by
\(\mathbf c_{\parallel}\) leaves the empirical loss unchanged and
weakly decreases the regularization term, with a strict decrease
whenever \(\mathbf c_{\perp}\neq0\). Consequently, every minimizer
of the unconstrained problem belongs to
\(\operatorname{range}(\mathbf{A}_\theta)\). It follows that the joint problem \eqref{eq: deep_kernel_hedging_rewritten} is equivalent to
\begin{equation}
\label{eq:square_root_reformulation}
\inf_{\theta\in\Theta,\;\mathbf c\in\mathbb{R}^N}
\left\{
\frac{1}{N}
\sum_{i=1}^N
\mathcal{L}\left((
\mathbf{Y}^{v_0}
-
\mathbf{A}_\theta\mathbf c)_i
\right)
+
\lambda\lVert\mathbf c\rVert_2^2
\right\}.
\end{equation}

We now establish the existence of a minimizer of
\eqref{eq:square_root_reformulation}. By \eqref{ass:kernel:H2} in
Assumption \ref{ass:kernel}, and the definition of
\(\mathbf{Q}_\theta\), the mapping
\(
\theta\mapsto\mathbf{Q}_\theta
\)
is continuous. Since the positive-semidefinite matrix square-root map
is continuous, the mapping
\(
\theta\mapsto\mathbf{A}_\theta
\)
is also continuous. Together with the continuity of
\(\mathcal{L}\), this implies that the objective in
\eqref{eq:square_root_reformulation} is continuous in
\((\theta,\mathbf c)\). Moreover, since $\mathcal{L}$ is bounded from below, there exists $C_{\mathcal L}\in\mathbb R$ such that $\mathcal{L}\geq C_{\mathcal L}$. Hence,

$$
\frac{1}{N}
\sum_{i=1}^N
\mathcal{L}\left( (
\mathbf{Y}^{v_0}
-
\mathbf{A}_\theta\mathbf c)_i
\right)
+
\lambda\lVert\mathbf c\rVert_2^2
\geq
C_{\mathcal L}
+
\lambda\lVert\mathbf c\rVert_2^2.
$$
Define
\[
C_0
:=
\frac{1}{N}
\sum_{i=1}^N
\mathcal{L}\left(\mathbf{Y}^{v_0}_i\right)\geq C_\mathcal{L}.
\]
The value of the objective at \(\mathbf c=0\) is \(C_0 < \infty\),
independently of \(\theta\). Therefore, it is sufficient to minimize
over the sublevel set on which the objective is at most \(C_0\).
Every point in this sublevel set satisfies
\[
\lVert\mathbf c\rVert_2
\leq
\sqrt{\frac{C_0-C_\mathcal{L}}{\lambda}}.
\]
Thus, as \eqref{ass:kernel:H1} in
Assumption \ref{ass:kernel} holds, the optimization can be restricted to the compact set
\[
\Theta
\times
\overline{B}_{\mathbb{R}^N}
\left(
0,\sqrt{\frac{C_0-C_\mathcal{L}}{\lambda}}
\right).
\]
The Weierstrass theorem therefore yields a minimizer
\(
(\theta^\star,\mathbf c^\star_{\theta^\star})
\)
of \eqref{eq:square_root_reformulation}. By the preceding orthogonal-decomposition argument,
\(
\mathbf c_{\theta^\star}^\star
\in
\operatorname{range}(\mathbf{A}_{\theta^\star}),
\) and there exists \(\bm{\alpha}_{\theta^\star}^\star\in\mathbb{R}^N\) such that
\[
\mathbf c_{\theta^\star}^\star
=
\mathbf{A}_{\theta^\star}\bm{\alpha}_{\theta^\star} ^\star. 
\]
 Hence, \((\theta^\star,\boldsymbol{\alpha}_{\theta^\star}^\star)\) is a minimizer of \eqref{eq: deep_kernel_hedging_rewritten}.
\end{proof}

\subsubsection{Proof of Lemma \ref{lem:RFF_hedging}}\label{proof_Prop_2.1}
\begin{proof}
For fixed \(\theta\in\Theta\), define
\[
F_\theta(\boldsymbol\beta)
= \frac{1}{N}
\sum_{i=1}^N
\mathcal L\left(
Y_i-(\mathbf Z_\theta\boldsymbol\beta)_i
\right)
+
\lambda\|\boldsymbol\beta\|_2^2,
\qquad
\boldsymbol\beta\in\mathbb R^D.
\]
 Because \(\mathcal L\) is continuous, \(F_\theta\) is continuous. Moreover, since \(\mathcal L\) is bounded below by a constant $C_\mathcal{L}\in \mathbb{R}$, and \(\lambda>0\),
\[
F_\theta(\boldsymbol\beta)\geq C_\mathcal{L}+\lambda\|\boldsymbol\beta\|_2^2,
\]
so \(F_\theta\) is coercive and hence admits a minimizer. Furthermore,
the loss term is convex in \(\boldsymbol\beta\), while
\(\lambda\|\boldsymbol\beta\|_2^2\) is strictly convex. Thus, \(F_\theta\) is
strictly convex and its minimizer \(\boldsymbol\beta_\theta^\star \in \R^D\) is unique. Moreover, the
left-hand side of \eqref{eq: reform_RFF_deep_kernel_hedging_fixed_theta} is equal to
\[
\min_{\boldsymbol\beta\in\operatorname{range}(\mathbf Z_\theta^\top)}
F_\theta(\boldsymbol\beta).
\]

We now use the orthogonal decomposition
\(
\mathbb R^D
=
\operatorname{range}(\mathbf Z_\theta^\top)
\oplus
\ker(\mathbf Z_\theta).
\)
Thus, every \(\boldsymbol\beta\in\mathbb R^D\) can be written uniquely as
\[
\boldsymbol\beta=\boldsymbol\beta_\parallel+\boldsymbol\beta_\perp,
\qquad
\boldsymbol\beta_\parallel\in\operatorname{range}(\mathbf Z_\theta^\top),
\quad
\boldsymbol\beta_\perp\in\ker(\mathbf Z_\theta).
\]
Since
\(
\mathbf Z_\theta\beta
=
\mathbf Z_\theta\beta_\parallel
\)
and the decomposition is orthogonal, we obtain
\begin{align*}
F_\theta(\boldsymbol\beta)
&= \frac{1}{N}
\sum_{i=1}^N
\mathcal L\left(
Y_i-(\mathbf Z_\theta\boldsymbol\beta_\parallel)_i
\right)
+
\lambda\|\boldsymbol\beta_\parallel\|_2^2
+
\lambda\|\boldsymbol\beta_\perp\|_2^2 =
F_\theta(\boldsymbol\beta_\parallel)
+
\lambda\|\boldsymbol\beta_\perp\|_2^2, \qquad
\boldsymbol\beta\in\mathbb R^D.
\end{align*}
Consequently, \(F_\theta(\boldsymbol\beta)\geq F_\theta(\boldsymbol\beta_\parallel)\), with
strict inequality whenever \(\boldsymbol\beta_\perp\neq 0\). It follows that
\[
\boldsymbol\beta_\theta^\star
\in\operatorname{range}(\mathbf Z_\theta^\top).
\]
In particular, there exists
$\boldsymbol\alpha^\star\in\mathbb R^N$ such that
$\mathbf Z_\theta^\top\boldsymbol\alpha^\star
=\boldsymbol\beta_\theta^\star$, and
\[
\min_{\boldsymbol\alpha\in\mathbb R^N}
F_\theta(\mathbf Z_\theta^\top\boldsymbol\alpha)
=
\min_{\boldsymbol\beta\in
      \operatorname{range}(\mathbf Z_\theta^\top)}
F_\theta(\boldsymbol\beta)
=
\min_{\boldsymbol\beta\in\mathbb R^D}
F_\theta(\boldsymbol\beta).
\]
This proves
\eqref{eq: reform_RFF_deep_kernel_hedging_fixed_theta}. Moreover, applying Lemma \ref{lem: existence and uniqueness} and Theorem \ref{thm:representer_thm_inner_problem} to the kernel $\widetilde K_{\theta,D}$ show that $\phi_{D,\theta}^\star$ defined by \eqref{eq: min_inner_prob_RFF} is the unique minimizer of the inner RFF problem.
\begin{comment}
This proves
\eqref{eq: reform_RFF_deep_kernel_hedging_fixed_theta}
and shows that $\boldsymbol\alpha^\star$ is a minimizer
of the corresponding coefficient problem. Applying Theorem~\ref{thm:representer_thm_inner_problem}
to the kernel $\widetilde K_{\theta,D}$ shows that
$\boldsymbol\alpha^\star$ determines a minimizer of the
inner problem in \eqref{eq: inner_outer2}, given by
\begin{align*}
\phi_{D,\theta}^\star(x)
&=
\sum_{i=1}^N \alpha_i^\star
\sum_{k=0}^{n-1}
G_{t_{k+1}}^i
\widetilde K_{\theta,D}(X_{t_k}^i,x)
\\
&=
\sum_{i=1}^N \alpha_i^\star
(Z_\theta^i)^\top y_\theta(x)
\\
&=
(\mathbf Z_\theta^\top\boldsymbol\alpha^\star)^\top
y_\theta(x)
\\
&=
(\boldsymbol\beta_\theta^\star)^\top y_\theta(x),
\qquad x\in\mathcal X.
\end{align*}
Finally, $\mathcal J_{\theta,D}$ is strictly convex on
$\mathcal H_{\widetilde K_{\theta,D}}$, since its loss term
is convex and its squared RKHS norm is multiplied by
$\lambda>0$. Thus, $\phi_{D,\theta}^\star$ is the unique
minimizer of the inner RFF problem.
\end{comment}
\end{proof}

\subsubsection{Proof of Proposition \ref{prop:uniform_RFF_theta}}\label{Proof:Proposition3.1}
\begin{proof}
Since $\Theta\times\mathcal X$ is compact and
$(\theta,x)\mapsto\psi_\theta(x)$ is jointly continuous, the joint
latent set
\[
\mathcal Z
:=
\left\{
\psi_\theta(x):
\theta\in\Theta,\ x\in\mathcal X
\right\}
\subset\mathbb R^p
\]
is compact. Define the latent-space RFF approximation by
\[
\widetilde K_D(z,z')
:=
\frac{2}{D}\sum_{j=1}^D
\cos(w_j^\top z+b_j)
\cos(w_j^\top z'+b_j) , \quad z,z'\in \mathcal{Z}. 
\] 
By construction, \[
 \widetilde K_{\theta,D}(x,x')= \widetilde K_D\bigl(\psi_\theta(x),\psi_\theta(x')\bigr), \quad x,x'\in \mathcal{X}, 
\]
and hence
\[
\sup_{\theta\in\Theta}\sup_{x,x'\in\mathcal X}
\left|
\widetilde K_{\theta,D}(x,x')-K_\theta(x,x')
\right|
\leq
\sup_{z,z'\in\mathcal Z}
\left|
\widetilde K_D(z,z')-K(z,z')
\right|,
\]
Standard uniform RFF approximation results (Claim 1 of 
\cite{rahimi2007random} and Proposition 2 of \cite{sutherland2015error}) then  give
\[
\mathbb{P} \left( \sup_{z,z'\in\mathcal Z}
\left|
\widetilde K_D(z,z')-K(z,z')
\right| > \varepsilon \right) \leq C\varepsilon^{-2}
\exp(-cD\varepsilon^2), 
\]
which proves \eqref{eq: conv_rate_approx_kernel_D} for $\varepsilon \in (0,1]$. For every fixed $\varepsilon\in(0,1]$, applying the Borel--Cantelli lemma for
$\varepsilon=1/q$, $q\in\mathbb N$, directly proves
\eqref{eq:uniform_RFF_theta}.
By the definitions of $\mathbf Q_\theta$ and
$\widetilde{\mathbf Q}_{\theta,D}$, we then have 
\begin{align*}
\left|
(\widetilde{\mathbf Q}_{\theta,D}-\mathbf Q_\theta)_{ij}
\right|
&\leq
\sum_{k=0}^{n-1}|G_{t_{k+1}}^i| \,  \sum_{k=0}^{n-1}|G_{t_{k+1}}^j| \, 
\sup_{\theta \in \Theta} \sup_{x,x'\in\mathcal X}
\left|
\widetilde K_{\theta,D}(x,x')-K_\theta(x,x')
\right|
\\
&= a_i a_j \,  \sup_{\theta \in \Theta} \sup_{x,x'\in\mathcal X}
\left|
\widetilde K_{\theta,D}(x,x')-K_\theta(x,x')
\right| ,
\end{align*}
where $a_i := \sum_{k=0}^{n-1}|G_{t_{k+1}}^i|$. Consequently, for any unit vectors $u,v\in\mathbb R^N$,
\begin{align*}
\left|
u^\top
\bigl(
\widetilde{\mathbf Q}_{\theta,D}-\mathbf Q_\theta
\bigr)v
\right|
&\leq
\sup_{\theta \in \Theta} \sup_{x,x'\in\mathcal X}
\left|
\widetilde K_{\theta,D}(x,x')-K_\theta(x,x')
\right|
\left(\sum_{i=1}^N a_i|u_i|\right)
\left(\sum_{j=1}^N a_j|v_j|\right)\\
&\leq \left(\sum_{i=1}^N a_i^2 \right) \,  \sup_{\theta \in \Theta} \sup_{x,x'\in\mathcal X}
\left|
\widetilde K_{\theta,D}(x,x')-K_\theta(x,x')
\right|,
\end{align*}
where the last inequality follows from
Cauchy--Schwarz.\ Taking the\hspace{-0.2mm} supremum over $u$, $v$ and $\theta$
proves \eqref{eq:uniform_Q_RFF_theta}.
\end{proof}
\subsubsection{Proof of Theorem \ref{thm:RFF_hedging_convergence}}
\label{Proof:theorem3.3}
\begin{proof}
We work on the probability-one event on which the uniform RFF convergence in Proposition~\ref{prop:uniform_RFF_theta} holds.

\textbf{Step 1: convergence of $\mathcal{V}_D$.} First, let us show the convergence of $(\mathcal{V}_D)_{D>1}$.  Let
\[
\mathbf A_\theta:=\mathbf Q_\theta^{1/2},
\qquad
\widetilde{\mathbf A}_{\theta,D}
:=
\widetilde{\mathbf Q}_{\theta,D}^{1/2}.
\]
By Proposition~\ref{prop:uniform_RFF_theta} and continuity of the
positive-semidefinite square-root map,
\begin{equation}\label{eq: conv_uni_A}
\eta_D := \sup_{\theta\in\Theta}
\left\|
\widetilde{\mathbf A}_{\theta,D}
-
\mathbf A_\theta
\right\|_{\mathrm{op}}
\rightarrow0, \quad \text{as } D\to \infty.  
\end{equation} 
Then, using the square-root reformulation from the proof of
Lemma~\ref{lem:existence_deep_kernel_hedging_rewritten}, the exact and approximate optimal
values can be written as
\[
\mathcal V
=
\min_{\theta\in\Theta, \,  \mathbf c\in\mathbb R^N}
F(\theta,\mathbf c),
\qquad
\mathcal V_D
=
\min_{\substack{\theta\in\Theta, \, \mathbf c\in\mathbb R^N}}
F_D(\theta,\mathbf c),
\]
where
\[
F(\theta,\mathbf c)
=
\frac1N\sum_{i=1}^N
\mathcal L\left( (
\mathbf Y^{v_0}
-
\mathbf A_\theta\mathbf c)_i
\right)
+
\lambda\|\mathbf c\|_2^2
\]
and
\[
F_D(\theta,\mathbf c)
=
\frac1N\sum_{i=1}^N
\mathcal L\left( (
\mathbf Y^{v_0}
-
\widetilde{\mathbf A}_{\theta,D}\mathbf c)_i
\right)
+
\lambda\|\mathbf c\|_2^2.
\]

Set
\[
C_0:=\frac1N\sum_{i=1}^N\mathcal L(\mathbf Y_i^{v_0}),
\qquad
R:=\sqrt{\frac{C_0-{C}_\mathcal{L}}{\lambda}},
\]
where $C_\mathcal{L}\in \mathbb{R}$ is the lower bound of $\mathcal{L}$. We observe that every minimizer $\mathbf c^\star$ of either problem belongs to the ball $\{\mathbf c:\|\mathbf c\|_2\leq R\}$. We next show that $F_D\to F$ uniformly on
\[
\Theta\times\overline B_R,
\qquad
\overline B_R:=\{\mathbf c\in\mathbb R^N:\|\mathbf c\|_2\leq R\},
\]
as  $D \to \infty$. By compactness of $\Theta$ and continuity of
$\theta\mapsto\mathbf A_\theta$,
\[
\sup_{\theta\in\Theta}\|\mathbf A_\theta\|_{\mathrm{op}}<\infty.
\]
Together with \eqref{eq: conv_uni_A}, this implies that, for all
sufficiently large $D$,
\[
\sup_{\theta\in\Theta}
\|\widetilde{\mathbf A}_{\theta,D}\|_{\mathrm{op}}
\leq B
\]
for some $B<\infty$. Hence, for $\|\mathbf c\|_2\leq R$, all arguments
of $\mathcal L$ appearing in $F$ and $F_D$ belong to the compact interval $[-M,M]$, where
\[
M:=\max_{1\leq i\leq N}|\mathbf Y_i^{v_0}|+BR.
\]
The continuity of $\mathcal L$ therefore implies its uniform continuity
on $[-M,M]$. Let
\[
\omega_{\mathcal L,M}(\delta)
:=
\sup_{\substack{u,v\in[-M,M]\\|u-v|\leq\delta}}
|\mathcal L(u)-\mathcal L(v)|
\]
denote its modulus of continuity on this interval. For every
$\theta\in\Theta$ and $\|\mathbf c\|_2\leq R$,
\[
\left|
\bigl(
(\widetilde{\mathbf A}_{\theta,D}-\mathbf A_\theta)\mathbf c
\bigr)_i
\right|
\leq R\eta_D.
\]
It follows that
\[
\sup_{\theta\in\Theta, \, \|\mathbf c\|_2\leq R}
|F_D(\theta,\mathbf c)-F(\theta,\mathbf c)|
\leq
\omega_{\mathcal L,M}(R\eta_D)
\rightarrow0.
\]
almost surely as $D\to\infty$. It follows directly that
\begin{equation}\label{eq: conv_V_D_to_V}
|\mathcal V_D-\mathcal V|
\rightarrow0 \,,
\end{equation}
almost surely as $D\to\infty$.\\

\textbf{Step 2: convergence of $(\theta^\star_D, \phi_D^\star)_{D>1}$.} It remains to establish convergence of the corresponding hedging
strategies. The definition of the RKHS norm and the regularization bound give
\begin{equation} \label{eq:RFF_norm_bound}
\|\phi^\star_D\|_{\mathcal H_{
\widetilde K_{\theta^\star_D,D}}}
\leq 
\|\boldsymbol\beta^\star_D\|_2
\leq R.
\end{equation}
Consider an arbitrary subsequence. By compactness of $\Theta$, it has a
further subsequence, still indexed by $D$, such that
\[
\theta^\star_D\rightarrow\bar\theta \, .
\]
The uniform RFF approximation of Proposition~\ref{prop:uniform_RFF_theta} and the continuity of
$\theta\mapsto K_\theta$ in $C(\mathcal X\times\mathcal X)$ imply
\begin{equation} \begin{aligned}
\varepsilon_D
&:=
\sup_{x,x' \in \mathcal{X}}\left|
\widetilde K_{\theta_D^\star,D}(x,x')
-
K_{\bar\theta}(x,x')
\right|&
\\
&\hspace{9mm}\leq
\sup_{\theta\in\Theta} \sup_{x,x' \in \mathcal{X}}
\left|
\widetilde K_{\theta,D}(x,x')-K_\theta(x,x')
\right|
+
 \sup_{x,x' \in \mathcal{X}} \left|K_{\theta_D^\star}(x,x')-K_{\bar\theta}(x,x') \right|
\rightarrow0.
\label{eq:uni_conv_K_tilde_theta}
\end{aligned} \end{equation}

We next verify uniform boundedness. By the reproducing property and
\eqref{eq:RFF_norm_bound},
\[
|\phi_D^\star(x)|
\leq
R\sqrt{
\widetilde K_{\theta_D^\star,D}(x,x)}.
\]

\begin{comment}
Equation~\eqref{eq:uni_conv_K_tilde_theta} and boundedness of
$K_{\bar\theta}$ on the compact set $\mathcal X\times\mathcal X$
therefore show that
\[
\sup_D \sup_{x\in \mathcal{X}} |\phi_D^\star(x)|<\infty,
\]
after disregarding at most finitely many indices.
\end{comment}

By \eqref{eq:uni_conv_K_tilde_theta}, there exists $D_0$ such that
$\varepsilon_D\leq 1$ for all $D\geq D_0$. Moreover, since
$K_{\bar\theta}$ is continuous on the compact set
$\mathcal X\times\mathcal X$, we may define
\[
M:=\sup_{x\in\mathcal X} K_{\bar\theta}(x,x)<\infty.
\]
Hence, for all $D\geq D_0$,
\[
\sup_{x\in\mathcal X}
\widetilde K_{\theta_D^\star,D}(x,x)
\leq M+1,
\]
and therefore
\[
\sup_{x\in\mathcal X}|\phi_D^\star(x)|
\leq R\sqrt{M+1}.
\]
For the finitely many indices $D<D_0$, since each $\phi_D^\star$
is continuous on the compact set $\mathcal X$, we have
\[
C_0:=
\max_{1\leq D<D_0}
\sup_{x\in\mathcal X}|\phi_D^\star(x)|
<\infty.
\]
Thus, defining
\[
C:=\max\left\{C_0,R\sqrt{M+1}\right\}<\infty,
\]
we obtain the uniform bound
\[
\sup_D\sup_{x\in\mathcal X}|\phi_D^\star(x)|
\leq C<\infty.
\]

We now establish equicontinuity. To this end, let us define
\[
\rho_D(x,x')^2
:=
\widetilde K_{\theta_D^\star,D}(x,x)
+
\widetilde K_{\theta_D^\star,D}(x',x')
-
2\widetilde K_{\theta_D^\star,D}(x,x')
\]
and
\[
\rho_{\bar\theta}(x,x')^2
:=
K_{\bar\theta}(x,x)
+
K_{\bar\theta}(x',x')
-
2K_{\bar\theta}(x,x').
\]
The RKHS inequality gives
\begin{equation}
\label{eq:equicont_RKHS}
|\phi_D^\star(x)-\phi_D^\star(x')|
\leq R\rho_D(x,x').
\end{equation}
Furthermore, \eqref{eq:uni_conv_K_tilde_theta} implies
\[
\rho_D(x,x')^2
\leq
\rho_{\bar\theta}(x,x')^2+4\varepsilon_D.
\]
Since $K_{\bar\theta}$ is continuous and $\mathcal X$ is compact,
\[
\sup_{\|x-x'\|_2 \leq\delta}
\rho_{\bar\theta}(x,x')
\rightarrow0
\qquad\text{as }\delta\downarrow0.
\]
Combining this property with $\varepsilon_D\to0$ and
\eqref{eq:equicont_RKHS} shows that $(\phi_D^\star)_D$ is
equicontinuous. Since the family $(\phi_D^\star)_D$ is uniformly bounded and equicontinuous on the compact set $\mathcal X$, the Arzel\`a--Ascoli theorem therefore yields a further subsequence and a function $\bar\phi\in C(\mathcal X)$ such that
\begin{equation} \label{eq:unife}\sup_{x \in \mathcal{X}}|\phi^*_D(x)-\bar\phi(x)|
\rightarrow0,
\end{equation}
as $D \to \infty$. Moreover, the uniform convergence of the kernels
\eqref{eq:uni_conv_K_tilde_theta} and of the strategies
\eqref{eq:unife}, together with the bound
\eqref{eq:RFF_norm_bound}, allows us to pass to the limit in the
positive-semidefiniteness condition of \cite[Theorem 3.11]{paulsen2016introduction}, yielding
\begin{equation} \label{eq:liminf}
\bar\phi\in\mathcal H_{K_{\bar\theta}},
\qquad
\|\bar\phi\|_{\mathcal H_{K_{\bar\theta}}}
\leq
\liminf_{D\to\infty}
\|\phi^\star_D\|_{\mathcal H_{
\widetilde K_{\theta^\star_D,D}}}.
\end{equation}

We next pass to the limit in the empirical hedging loss.
For $f\in C(\mathcal X)$, define
\[
\mathcal R(f)
:=
\frac1N\sum_{i=1}^N
\mathcal L\left(
Y^i-B_T^iv_0
-\sum_{k=0}^{n-1}
G_{t_{k+1}}^i f(X_{t_k}^i)
\right).
\]
For each training path $i$,
\begin{align*}
\left|
\sum_{k=0}^{n-1}
G_{t_{k+1}}^i
\left(
\phi_D^\star(X_{t_k}^i)-\bar\phi(X_{t_k}^i)
\right)
\right| \leq
\left(\sum_{k=0}^{n-1}|G_{t_{k+1}}^i|\right) \, 
\sup_{x \in \mathcal{X}} |\phi_D^\star(x)-\bar\phi(x) | 
\rightarrow0.
\end{align*}
Since $\mathcal L$ is continuous and the dataset is
finite, it follows that
\[
\mathcal R(\phi_D^\star)\rightarrow\mathcal R(\bar\phi).
\]

Combining this convergence with \eqref{eq: conv_V_D_to_V}, \eqref{eq:RFF_norm_bound} and \eqref{eq:liminf}, we obtain
\begin{align*}
\mathcal J_{\bar\theta}(\bar\phi)
=
\mathcal R(\bar\phi)
+\lambda
\|\bar\phi\|_{\mathcal H_{K_{\bar\theta}}}^2 &\leq
\mathcal R(\bar\phi)
+\lambda\liminf_{D\to\infty}
\|\phi_D^\star\|_{\mathcal H_{
\widetilde K_{\theta^\star_D,D}}}^2\\
&\leq
\liminf_{D\to\infty}
\left\{
\mathcal R(\phi_D^\star)
+\lambda\|\boldsymbol\beta_D^\star\|_2^2
\right\}\\
&=
\liminf_{D\to\infty}\mathcal V_D
=
\mathcal V.
\end{align*}
On the other hand, $\bar\theta\in\Theta$ and
$\bar\phi\in\mathcal H_{K_{\bar\theta}}$, so the definition
of $\mathcal V$ gives
\[
\mathcal V\leq\mathcal J_{\bar\theta}(\bar\phi).
\]
Hence,
\[
\mathcal J_{\bar\theta}(\bar\phi)=\mathcal V,
\]
and $(\bar\theta,\bar\phi)$ is a global minimizer of the
exact deep kernel hedging problem. The argument applies to every subsequence. Consequently, every
accumulation point of
$\{(\theta_D^\star,\phi_D^\star)\}_{D\geq1}$ is a global minimizer.
If the exact minimizing pair $(\theta^\star,\phi^\star)$ is unique,
every subsequence has a further subsequence converging to that same
pair. Therefore,
\[
\theta_D^\star\rightarrow\theta^\star,
\qquad
\sup_{x \in \mathcal{X}}|\phi^*_D(x)-\phi^\star(x)|
\rightarrow0,
\]
almost surely as $D\to \infty$.
\end{proof}
\subsubsection{Proof of Proposition \ref{prop:existence_CVaR}}
\begin{proof} \label{proof:existence_CVar}
Since infima over product sets commute, we may exchange the order of minimization and write \eqref{eq:deep_kernel_cvar_hedging} as
\begin{equation}\label{eq: cvar_hedging_rewritten}
\inf_{\zeta\in\mathbb R} \
\inf_{\theta\in\Theta, \, 
                \phi\in\mathcal H_{K_\theta}}
\left\{
\zeta
+
\frac{1}{\tilde\alpha N}
\sum_{i=1}^N
\left(
Y^i
-
B_T^i v_0
-
\sum_{k=0}^{n-1}
\phi(X^i_{t_k})
G^i_{t_{k+1}} -\zeta
\right)_+
+
\lambda
\|\phi\|_{\mathcal{H}_{K_\theta}}^2
\right\},
\end{equation}
For every fixed $\zeta\in \mathbb{R}$, the loss
\[
u\mapsto
\mathcal L_{\tilde\alpha}^{\mathrm{CVaR}}(u,\zeta)
=
\zeta+\frac{1}{\tilde\alpha}(u-\zeta)_+
\]
is continuous, convex and bounded from below. Therefore, Theorem
\ref{thm:existence_thm_hedging} implies that the inner minimum is
attained by some $(\theta_\zeta^\star,\phi_\zeta^\star)$. It remains to minimize the resulting value function in $\zeta$.
Using the finite-dimensional representation from Theorem~\ref{thm:representer_thm_inner_problem} and the proof of Lemma \ref{lem:existence_deep_kernel_hedging_rewritten} with $\mathbf c
= \mathbf{Q}^{1/2}_\theta\bm{\alpha} \in \R^N$, we write
\[
F(\theta,\mathbf c,\zeta)
:=
\zeta+
\frac{1}{\tilde\alpha N}
\sum_{i=1}^N
\left(
Y^i-B_T^iv_0-(\mathbf{Q}_\theta^{1/2}\mathbf c)_i-\zeta
\right)_+
+
\lambda\|\mathbf c\|_2^2 , \ \  \text{ and } \ \ \mathcal{V}(\zeta)
:=
\min_{\theta\in\Theta,\,\mathbf c\in\mathbb R^N}
F(\theta,\mathbf c,\zeta),
\]
such that \eqref{eq: cvar_hedging_rewritten} is equivalent to 
\begin{equation*}
    \inf_{\zeta\in \mathbb{R}} \mathcal{V}(\zeta).
\end{equation*}
Moreover, we observe that 
\[
F(\theta,\mathbf c,\zeta)\geq\zeta,
\]
and completing the square in
$\mathbf c$, we also have that there exists a constant $C_0<\infty$, independent
of $(\theta,\mathbf c,\zeta)$, such that
\[
F(\theta,\mathbf c,\zeta)
\geq
-C_0-\frac{1-\tilde\alpha}{\tilde\alpha}\zeta.
\]
The  value function $\mathcal{V}(\zeta)$ therefore tends to $+\infty$ as
$\zeta\to+\infty$ and as $\zeta\to-\infty$, and is therefore coercive. Moreover, under
Assumption~\ref{ass:kernel}, the continuity of
$\theta\mapsto\mathbf Q_\theta^{1/2}$ and the uniform coercivity in
$\mathbf c$ imply that $\mathcal{V}$ is continuous. Hence, $\mathcal{V}$ attains its minimum at some $\zeta^\star\in\mathbb R$. Since the corresponding inner minimum is also attained, there exists a global minimizer $(\theta^\star,\phi^\star,\zeta^\star)$ of
\eqref{eq:deep_kernel_cvar_hedging}.
\end{proof}

\subsection{Path-signatures} \label{appendix:sig}

In this appendix, we provide a brief reminder on path-signatures. For $T > 0$, let $X:=(X_t)_{t \leq T}$ be a (continuous) path with finite variation in $\R^d$.  Let $\{ e_1, \dots, e_d \} \subset \R^d$ be the canonical basis of $\mathbb{R}^d$ and $\alphabet = \{ \word{1}, \word{2}, \dots, \word{d} \}$ be the corresponding alphabet. For $i \in \{ 1, \dots, d \}$, we write $e_{i}$ as the blue letter $\word{i}$ and for $n \geq 1, i_1, \dots, i_n \in \{ 1, \dots, d \}$, we write $e_{i_1} \otimes \cdots \otimes e_{i_n} $ as the concatenation of letters  $\word{i_1 \cdots i_n}$, that we call a word of length $n$. We note that $(e_{i_1} \otimes \cdots \otimes e_{i_n})_{(i_1, \dots, i_n) \in \{ 1, \dots, d \}^n}$ is a basis of $(\R^d) \conpow{n}$ that can be identified with the set of words of length $n$ defined by 
    
    \begin{equation*}
        V_n := \{ \word{i_1 \cdots i_n}: \word{i_k} \in \alphabet \text{ for } k = 1, 2, \dots, n \}. 
    \end{equation*} 
    
    Moreover, $\emptyword$ denotes the empty word and $V_0 = \{ \emptyword \}$ serves as a basis for $(\R^d) \conpow{0} = \R$. 

\begin{defn}\label{def:sig} The signature of $X$ is defined by $\mathbb{X}$ such that 
\begin{align*}
            \mathbb{X}_t := \left({\mathbb{X}}_t^\word{i_1 \cdots i_n}\right)_{(\word{i_1\cdots i_n})\in V_n,~n\in \mathbb{N}}, \quad t \leq T, 
\end{align*}
where $\mathbb{X}^\emptyword := 1$, and 
$$ \sigX^\word{i_1\dots i_n} := \int_{0 < u_1 < \cdots < u_n < t} d X_{u_1}^{\word{i_1}} \cdots   d X_{u_n}^\word{i_n}, \quad (\word{i_1\dots i_n})\in V_n,\quad n\in \mathbb{N}_0.$$

Similarly, the truncated signature of order $M \in \N$ is defined by
\begin{equation*}
\begin{aligned} 
            \sigX^{\leq M} :=\left({\mathbb{X}}_t^\word{i_1 \cdots i_n}\right)_{(\word{i_1\cdots i_n})\in V_n,~n\leq M}, \quad t\leq T. 
\end{aligned}
\end{equation*}
such that $\sigX^{\leq M}\in \mathbb{R}^{\frac{d^{M+1}-1}{d-1}}$. 
\end{defn} 
 In dimension $d=1$, the signature of $X$ reduces to the sequence of monomials $\left( \frac{1}{n!} (X_t-X_0)^n \right)_{n \in \mathbb{N}}$. In higher dimensions, however, the signature contains iterated integrals involving different components of the path. This makes path-signatures particularly useful for describing path-dependent quantities, especially when augmenting the path with time.
 \begin{comment}
Let $X:[0,T] \to \mathbb{R}$ and define its time-augmentation $\hat{X}:=(t,X_t)_{t\leq T}$. The first few signature levels of $\hat{X}$ are given by
    \begin{equation*}
    \begin{aligned}        
    &(\text{Order } 0):\hat{\mathbb{X}}_t^\emptyword = 1, \\
    &(\text{Order } 1): \hat{\mathbb{X}}_t^\word{1} = t, \quad \hat{\mathbb{X}}_t^\word{2} = X_t-X_0, \\
    &(\text{Order } 2):\hat{\mathbb{X}}_t^\word{11} = \frac{t^2}{2!}, \quad \hat{\mathbb{X}}_t^\word{12} = \int_0^t s d X_s, \quad \hat{\mathbb{X}}_t^\word{21} =  \int_0^t (X_s-X_0) ds, \quad \hat{\mathbb{X}}_t^\word{22} =  \frac{(X_t-X_0)^2}{2!}, \\
    &~~~~\dots
    \end{aligned}
    \end{equation*}

These terms already illustrate how the time-augmented signatures encode path-dependent features. While the first level only records the current time and the increment of the path, higher-order terms contain iterated integrals that depend on the evolution of the path over the whole time interval.
\end{comment}
The collection of all time-augmented signature terms therefore provides a rich summary of the path's history. As shown in \cite{cuchiero2025universal}, signatures provide a universal and algebraic representation of paths. In fact, under suitable conditions, the signatures uniquely characterize paths, up to the natural equivalence induced by tree-like pieces. This provides the theoretical foundation for using time-augmented signatures as input features in learning problems involving path-dependent quantities. Rather than using the entire past trajectory, one can represent the path through its time-augmented signature coordinates. In particular, when the signature is truncated at a finite order, this yields a finite-dimensional representation of a path, with higher truncation orders providing increasingly rich information about its history. \\

\begin{comment}
It is worth noting, however, that the universal approximation theorem stated in \cite{cuchiero2025universal} shows that, under suitable conditions, continuous functionals of paths can be approximated arbitrarily well by linear combinations of the coordinates of time-augmented signatures. In other words, for a suitable path functional $F$, we can approximate $F(X)$ by expressions of the form $F(X)\approx \langle \alpha, \hat{\mathbb{X}}^{\leq M}_T\rangle$, with $\alpha\in\mathbb{R}^{\frac{(d+1)^{M+1}-1}{d}}$ for a sufficiently large $M\in \mathbb{N}$. In practice, however, we typically work with low-order truncated signatures for computational reasons, and the resulting approximation error can be non-negligible, especially in highly nonlinear settings, see \cite[Section 5]{abi2025hedging} and \cite[Section 6.2]{jaber2025signaturehedging}. This motivates combining low-order signature features with nonlinear machine learning techniques, which can increase the expressive power of the representation without requiring prohibitively high signature orders. We adopt this approach in the present paper, combining low-order truncated signatures with nonlinear machine learning techniques. A similar approach was previously proposed in \cite{abi2025hedging} for learning optimal hedging strategies in non-Markovian settings. \\
\end{comment}

For further details on path-signatures and their applications in hedging context, we refer among others to \cite{lyons2020non, cuchiero2025universal, abi2025hedging,jaber2025signaturehedging, chevyrev2025primer}.

\subsection{Additional experiments on synthetic data}\label{sec:appendix_synthetic_data}

In this appendix, we provide additional results for the numerical experiments on synthetic data. 
\paragraph{Comparison to Random Feature Neural Networks (RFNNs).} To assess whether the benefits of our approach are specific to the random Fourier approximation  \eqref{eq: latent_repr} with \texttt{cosine} features and Gaussian kernel, or extend to more general random-feature families, we additionally consider RFNNs based on \texttt{ReLU} and \texttt{Tanh} activations \citep{huang2006extreme, rahimi2008weighted, gonon2023approximation}. We first note that our current implementation of (standard) kernel hedging with RFF is a particular RFNN with \texttt{cosine} activation. More precisely, for the Gaussian kernel
\(
K(x,x')
=
\exp (-\lVert x-x'\rVert_2^2/2\gamma^2 ),
\) the normalized RFF map is \begin{equation*}
y^{\rm{RFF}}_D(x)
=
\sqrt{\frac{2}{D}}
\begin{bmatrix}
\cos\bigl(w_1^\top x/\gamma +b_1\bigr)\\
\vdots\\
\cos\bigl(w_D^\top \, x/\gamma+b_D\bigr)
\end{bmatrix}
,
\end{equation*} where the Fourier frequencies and phases are sampled according to
\[
w_j \sim p_{K}
=
\mathcal{N}(0, I_d),
\qquad
b_j \sim \mathcal{U}(0,2\pi),
\]
and remain fixed throughout training. For the \texttt{ReLU} and \texttt{Tanh} alternatives, we similarly sample the hidden weights and biases once according to
\[
w_j \sim\mathcal{N}(0,I_d),
\qquad
b_j\sim\mathcal{N}(0,1),
\]
and keep them fixed during optimization. The corresponding normalized random-feature maps are given by
\begin{equation}
\label{eq:rfnn_maps}
y_D^{\mathrm{ReLU}}(x)
=
\frac{1}{\sqrt{D}}
\begin{bmatrix}
\operatorname{ReLU}\left(w_1^\top x+b_1\right)\\
\vdots\\
\operatorname{ReLU}\left(w_D^\top \, x+b_D\right)
\end{bmatrix},
\qquad
y_D^{\tanh}(x)
=
\frac{1}{\sqrt{D}}
\begin{bmatrix}
\tanh\left(w_1^\top x+b_1\right)\\
\vdots\\
\tanh\left(w_D^\top \, x+b_D\right)
\end{bmatrix}.
\end{equation}
For the classical RFNN hedging strategies, only the output coefficients $\boldsymbol{\beta}$ and the initial portfolio value $v_0$ are optimized. Their deep  counterparts instead apply the random-feature maps in \eqref{eq:rfnn_maps} to the learned representation $\psi_\theta(x)$, see \eqref{eq: latent_repr}. Consequently, the representation parameters \(\theta\), the output coefficients \(\boldsymbol{\beta}\), and the initial portfolio value \(v_0\) are optimized jointly, while the random weights and biases $(w,b)$ remain fixed. To isolate the effect of the activation function, the \texttt{ReLU} and \texttt{Tanh} models use identical realizations of \(w\) and \(b\) within each seed, and all methods share the same number of random features, training seeds, loss function, and optimization settings. These comparisons therefore allow us to disentangle the contribution of the learned representation from that of the underlying random-feature family. The main observation from Table \ref{tab:other_metrics_qh_rfnn} is that the learned representation $\psi_\theta$ is far more important than the random-feature activation. Without it, kernel hedging with \texttt{cosine} RFF and both shallow RFNNs perform even worse than Black-Scholes. Once the representation is learned, deep kernel hedging with \texttt{cosine} and deep RFNN hedging with \texttt{ReLU} or \texttt{Tanh} features yield very similar results.
\begin{table}[!ht]
\small
\centering
\caption{Out-of-sample metrics for the ATM European call under quadratic hedging with $T=1/4$, $N=1000$. Reported values are averages over ten seeds, with standard deviations in parentheses.}
\label{tab:other_metrics_qh_rfnn}
\begin{minipage}{\textwidth}
\centering
\resizebox{\textwidth}{!}{%
\begin{tabular}{lcccc}
\toprule
& MSE & MAE & $q^{\text{upper}}_{0.95}$ & $\operatorname{CVaR}^{\text{upper}}_{0.95}$ 
 \\
\midrule

BS
& $1.07\mathrm{e}{-4}$
& $8.28\mathrm{e}{-3}$
& $1.95\mathrm{e}{-2}$
& $2.49\mathrm{e}{-2}$
 \\

Kernel hedging (\texttt{cosine})
& $1.33\mathrm{e}{-4}$ ($1.6\mathrm{e}{-6}$)
& $8.58\mathrm{e}{-3}$ ($7.7\mathrm{e}{-5}$)
& $1.99\mathrm{e}{-2}$ ($1.3\mathrm{e}{-4}$)
& $2.78\mathrm{e}{-2}$ ($1.7\mathrm{e}{-4}$)
 \\

RFNN hedging (\texttt{ReLU})
& $1.38\mathrm{e}{-4}$ ($1.9\mathrm{e}{-6}$)
& $8.91\mathrm{e}{-3}$ ($7.1\mathrm{e}{-5}$)
& $2.06\mathrm{e}{-2}$ ($2.8\mathrm{e}{-4}$)
& $2.87\mathrm{e}{-2}$ ($4.2\mathrm{e}{-4}$)
 \\

RFNN hedging (\texttt{Tanh})
& $1.45\mathrm{e}{-4}$ ($2.9\mathrm{e}{-6}$)
& $8.96\mathrm{e}{-3}$ ($9.9\mathrm{e}{-5}$)
& $2.15\mathrm{e}{-2}$ ($3.8\mathrm{e}{-4}$)
& $2.99\mathrm{e}{-2}$ ($4.8\mathrm{e}{-4}$)
 \\

Deep kernel hedging (\texttt{cosine})
& $\mathbf{6.87{e}{-5}}$ $\mathbf{(8.0{e}{-7})}$
& $\mathbf{6.29{e}{-3}}$ $\mathbf{(2.4{e}{-5})}$
& $1.46\mathrm{e}{-2}$ ($1.3\mathrm{e}{-4}$)
& $2.06\mathrm{e}{-2}$ ($2.5\mathrm{e}{-4}$)
 \\

Deep RFNN hedging (\texttt{ReLU})
& $6.90\mathrm{e}{-5}$ ($4.4\mathrm{e}{-7}$)
& $6.32\mathrm{e}{-3}$ ($2.4\mathrm{e}{-5}$)
& $1.46\mathrm{e}{-2}$ ($7.3\mathrm{e}{-5}$)
& $2.06\mathrm{e}{-2}$ ($1.2\mathrm{e}{-4}$)
 \\

Deep RFNN hedging (\texttt{Tanh})
& $6.88\mathrm{e}{-5}$ ($5.3\mathrm{e}{-7}$)
& $6.33\mathrm{e}{-3}$ ($3.8\mathrm{e}{-5}$)
& $\mathbf{1.45{e}{-2}}$ $\mathbf{(1.0e{-4})}$
& $\mathbf{2.04{e}{-2}}$ $\mathbf{(1.4{e}{-4})}$
 \\

\bottomrule
\end{tabular}%
}
\end{minipage}
\end{table}
\\ \\
We do not include the heavier deep kernel architectures of \cite{xie2019deep} and \cite{mehrkanoon2019deep} in our numerical comparison, as they lead to  substantial overfitting in our limited-data financial setting.
\paragraph{Robustness of the RFF approximation.}
\begin{figure}
    \centering
    \includegraphics[width=0.75\linewidth]{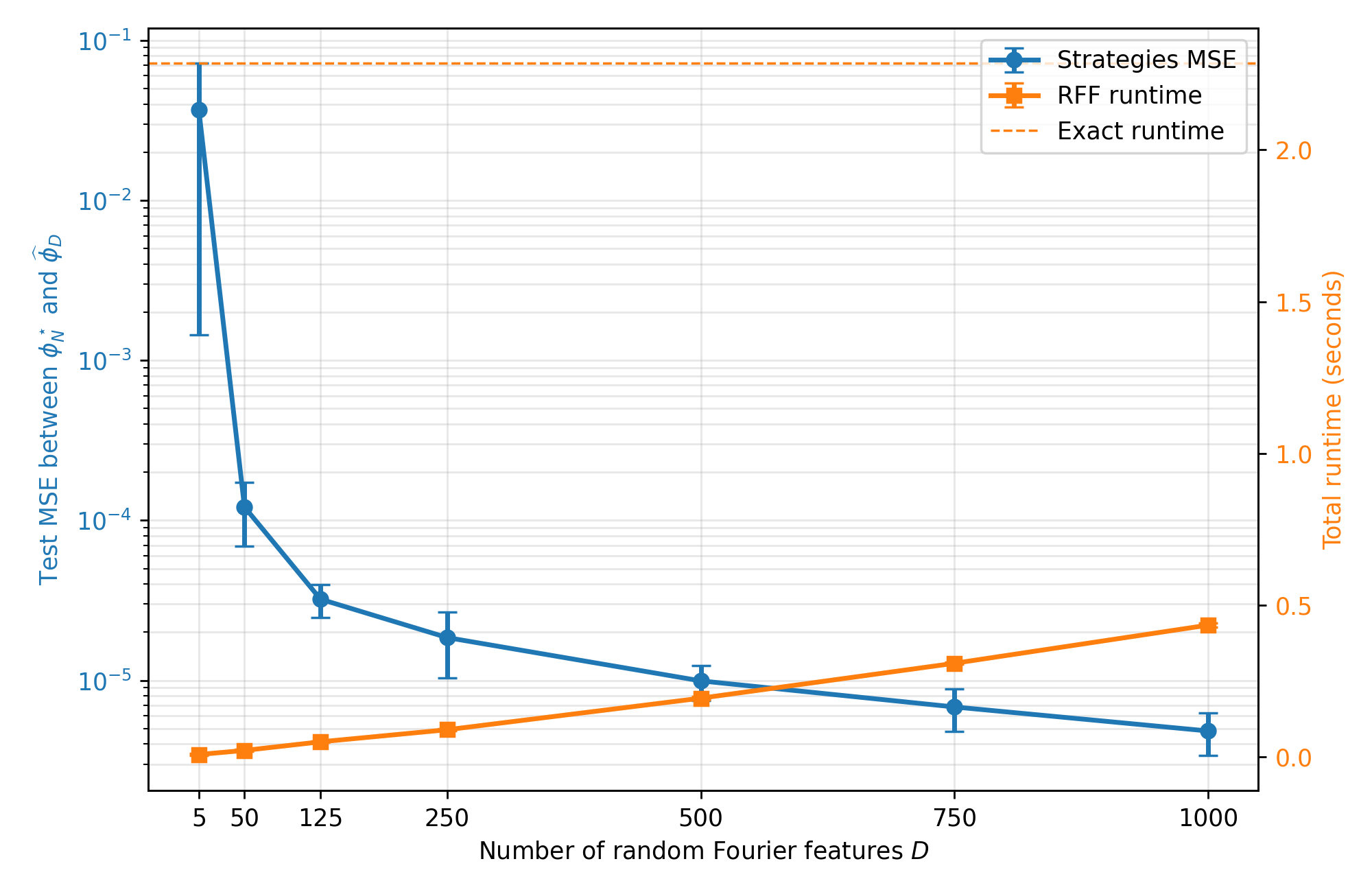}
    \caption{Comparison of the exact RKHS hedging strategy $\phi_{N,\hat\theta}^\star$ and its RFF approximation $\widehat{\phi}_{D,\hat\theta}$ for $N=1000$ training paths of an ATM European call under quadratic hedging. The blue curve shows the test-set MSE$_D$ \eqref{testMSEcomp}, while the orange curve reports the runtime required to construct the RFF features, solve the RFF hedging problem \eqref{eq:RFF_quadratic}, and evaluate $\widehat{\phi}_{D,\hat\theta}$ on the test set. The dashed orange line indicates the corresponding runtime for computing and evaluating the exact RKHS solution $\phi_{N,\hat\theta}^\star$. Markers and error bars represent the mean and one standard deviation over 10 RFF draws. The learned deep representation and kernel bandwidth $\hat{\theta} = (\hat\omega, \hat\gamma)$ are held fixed throughout the experiment.}
    \label{fig:RFF_conv}
\end{figure}
To assess the accuracy-efficiency trade-off of the RFF approximation, we first learn the deep representation and kernel bandwidth $\hat\theta = (\hat\omega,\hat\gamma)$ and then keep them fixed (such as to fix the learned kernel and thus only assess the RFF accuracy). Using \(N=1000\) training paths and conditionally on the learned kernel $K_{\hat\theta}$, we compute the exact RKHS optimizer \(\phi_{N,\hat\theta}^\star\) through the representer theorem \eqref{eq:represent_opt} and compare it with the RFF solution \(\widehat{\phi}_{D,\hat{\theta}}\) in \eqref{eq: RFF_hedging_strategy} for increasing numbers \(D\) of random features. Approximation accuracy is measured on an independent test set of $N^{\text{OOS}}=25,000$ simulated sample paths by the mean squared error between the two strategies 
\begin{equation}\text{MSE}_D = \frac{1}{N^{\text{OOS}} \, n} \sum_{i=1}^{N^{\text{OOS}}} \sum_{k=0}^{n-1} \left( \phi_{N,\hat{\theta}}^\star(X^i_{t_k}) - \widehat{\phi}_{D,\hat{\theta}}(X^i_{t_k}) \right)^2, \label{testMSEcomp}
\end{equation}while the reported runtime includes the construction of the RFF features, the solution of the resulting ridge-regression problem, and test-set evaluation, excluding the shared learning of the deep representation $\hat{\theta}$. Figure \ref{fig:RFF_conv} reports means and standard deviations over 10 RFF draws and shows that the strategy MSE decreases sharply as \(D\) increases, approaching \(10^{-5}\), whereas the computational cost rises roughly linearly. Most of the approximation gain is achieved for moderate values of \(D\) around 100, with diminishing improvements beyond approximately \(D=250\), illustrating the trade-off between fidelity to the exact RKHS solution and computational efficiency.

\paragraph{Asian option hedging.} We consider the quadratic hedging of ATM arithmetic Asian call options of the form
\begin{equation}\label{eq: arithm_asian_payoff}
    Y_T=\left(\frac{1}{T}\int_0^TS_t dt-S_0 \right)_+. 
\end{equation}
Since the Asian payoff depends on the average of the path, we augment the state space with the running average process $A=(A_t)_{0<t\leq T}$, defined by
\begin{equation*}
A_t := \frac{1}{t}\int_0^t S_s \, ds, \qquad 0<t\leq T,
\end{equation*}
with $A_0 := S_0$. As the underlying state $(S_t,\nu_t)$ is Markovian, the augmented state $(S_t,\nu_t,A_t)$ is also Markovian. Table \ref{tab:metrics_QH_Asian_call} reports the out-of-sample results based on synthetic data. Note that, unlike for the European call payoff, we do not include the Black-Scholes benchmark, as there is no explicit hedging strategy for an arithmetic Asian call of the form \eqref{eq: arithm_asian_payoff} in the Black-Scholes setting. Across all three maturities, the deep kernel hedging strategy substantially outperforms the benchmarks in terms of all the out-of-sample metrics, and its advantage becomes more pronounced as the maturity increases. 
\begin{table}[!ht]
\small
\centering
\caption{Out-of-sample metrics for the ATM Asian call for quadratic hedging with $N=1000$. Reported values are averages over ten seeds, with standard deviations in parentheses.}
\label{tab:metrics_QH_Asian_call}
\begin{minipage}{1\textwidth}
\centering
\caption*{$T=1/12$}
\begin{tabular}{lcccc}
\toprule 
& MSE & MAE & $q^{\text{upper}}_{0.95}$ & CVaR$^{\text{upper}}_{0.95}$\\ 
\midrule 
Kernel hedging & $1.42\mathrm{e}{-5}$ (2.9$\mathrm{e}{-7}$) & $2.97\mathrm{e}{-3}$ (3.5$\mathrm{e}{-5}$) & $5.64\mathrm{e}{-3}$ (1.3$\mathrm{e}{-4}$) & $7.86\mathrm{e}{-3}$ (1.1$\mathrm{e}{-4}$) \\ 
Deep hedging & $8.75\mathrm{e}{-6}$ (4.2$\mathrm{e}{-7}$) & $2.33\mathrm{e}{-3}$ (5.8$\mathrm{e}{-5}$) & $4.30\mathrm{e}{-3}$ (1.1$\mathrm{e}{-4}$) & $6.08\mathrm{e}{-3}$ (1.3$\mathrm{e}{-4}$) \\ 
Deep kernel hedging & $\mathbf{7.91{e}{-6} \, (1.8{e}{-7})}$ & $\mathbf{2.23{e}{-3} \, (3.1{e}{-5})}$ & $\mathbf{4.10{e}{-3} \, (1.2{e}{-4})}$ & $\mathbf{5.80{e}{-3} \, (1.5{e}{-4})}$ \\ \bottomrule \end{tabular}

\end{minipage}
\\ \vspace{0.25cm}
\begin{minipage}{1\textwidth}
\centering
\caption*{$T=1/4$}
\begin{tabular}{lcccc} 
\toprule
& MSE & MAE & $q^{\text{upper}}_{0.95}$ & CVaR$^{\text{upper}}_{0.95}$\\ 
\midrule 
Kernel hedging & $4.48\mathrm{e}{-5}$ (1.1$\mathrm{e}{-6}$) & $5.28\mathrm{e}{-3}$ (6.5$\mathrm{e}{-5}$) & $9.31\mathrm{e}{-3}$ (1.1$\mathrm{e}{-4}$) & $1.35\mathrm{e}{-2}$ (1.9$\mathrm{e}{-4}$) \\
Deep hedging & $2.52\mathrm{e}{-5}$ (1.2$\mathrm{e}{-6}$) & $3.97\mathrm{e}{-3}$ (1.0$\mathrm{e}{-4}$) & $6.42\mathrm{e}{-3}$ (1.3$\mathrm{e}{-4}$) & $9.48\mathrm{e}{-3}$ (2.2$\mathrm{e}{-4}$) \\ 
Deep kernel hedging & $\mathbf{2.12{e}{-5} \, (5.1{e}{-7})}$ & $\mathbf{3.67{e}{-3} \, (4.3{e}{-5})}$ & $\mathbf{5.93{e}{-3} \, (9.9{e}{-5})}$ & $\mathbf{8.77{e}{-3} \, (2.0{e}{-4})}$\\ 
\bottomrule
\end{tabular}
\end{minipage}
\\ \vspace{0.25cm}
\begin{minipage}{1\textwidth}
\centering
\caption*{$T=1/2$}
\begin{tabular}{lcccc} 
\toprule 
& MSE & MAE & $q^{\text{upper}}_{0.95}$ & CVaR$^{\text{upper}}_{0.95}$ \\ 
\midrule 
Kernel hedging & $1.04\mathrm{e}{-4}$ (2.4$\mathrm{e}{-6}$) & $8.09\mathrm{e}{-3}$ (9.4$\mathrm{e}{-5}$) & $1.41\mathrm{e}{-2}$ (2.2$\mathrm{e}{-4}$) & $2.04\mathrm{e}{-2}$ (3.0$\mathrm{e}{-4}$) \\ 
Deep hedging & $5.86\mathrm{e}{-5}$ (3.5$\mathrm{e}{-6}$) & $6.19\mathrm{e}{-3}$ (1.8$\mathrm{e}{-4}$) & $9.32\mathrm{e}{-3}$ (3.3$\mathrm{e}{-4}$) & $1.42\mathrm{e}{-2}$ (5.1$\mathrm{e}{-4}$) \\ 
Deep kernel hedging & $\mathbf{5.00{e}{-5} \, (1.2{e}{-6})}$ & $\mathbf{5.78{e}{-3} \, (7.8{e}{-5})}$ & $\mathbf{8.86{e}{-3} \, (2.3{e}{-4})}$ & $\mathbf{1.34{e}{-2} \, (3.6{e}{-4})}$ \\ 
\bottomrule 
\end{tabular}

\end{minipage}
\end{table}

\subsection{Additional experiments on S\&P500 data}\label{sec: appendix_SP500}
In this appendix, we provide additional results for the numerical experiments on real  data. 
\paragraph{Vanilla features.} As first baseline, we use vanilla features as inputs for learning the quadratic hedging strategy of an ATM European call based on real data. Among the learning methods, we observe in Table \ref{tab:loss_SP500_vanilla} that the deep kernel hedging remains the best-performing method under quadratic hedging when using the vanilla features as first baseline on S\&P500 data, achieving the lowest MSE and MAE, as well as the most favorable $q^{\text{upper}}_{0.95}$ and $\operatorname{CVaR}^{\text{upper}}_{0.95}$ in most settings. However, comparison with the results obtained using truncated time-augmented signatures in Table~\ref{tab:loss_SP500_sign} shows that the vanilla features do not attain the same level of hedging performance since they fail to include any path-dependent information. Moreover, the benefits of signature-based features become increasingly pronounced as the option maturity increases.
\begin{table}[!ht]
\small
\centering
\caption{Out-of-sample metrics for the quadratic hedging of an ATM European call on real dataset with vanilla features. P\&Ls are scaled by the underlying price at the beginning of the periods. Reported values are averages over ten seeds, with standard deviations in parentheses.}
\label{tab:loss_SP500_vanilla}

\begin{minipage}{1\textwidth}
\centering
\caption*{$T_\text{days}=$ 10 days.}
\begin{tabular}{lcccc} 
\toprule 
& MSE & MAE & $q^{\text{upper}}_{0.95}$ & CVaR$^{\text{upper}}_{0.95}$\\ 
\midrule BS & $3.12\mathrm{e}{-5}$ & $3.48\mathrm{e}{-3}$ & $9.13\mathrm{e}{-3}$ & $1.41\mathrm{e}{-2}$ \\ 
Kernel hedging & $2.62\mathrm{e}{-5}$ (3.9$\mathrm{e}{-7}$) & $3.72\mathrm{e}{-3}$ (2.6$\mathrm{e}{-5}$) & $7.99\mathrm{e}{-3}$ (1.9$\mathrm{e}{-4}$) & $1.16\mathrm{e}{-2}$ (1.7$\mathrm{e}{-4}$) \\ 
Deep hedging & $2.49\mathrm{e}{-5}$ (9.3$\mathrm{e}{-7}$) & $3.70\mathrm{e}{-3}$ (6.8$\mathrm{e}{-5}$) & $7.67\mathrm{e}{-3}$ (2.5$\mathrm{e}{-4}$) & $1.06\mathrm{e}{-2}$ (4.0$\mathrm{e}{-4}$) \\ 
Deep kernel hedging & $\mathbf{2.20{e}{-5} (5.6{e}{-7}})$ & $\mathbf{3.35{e}{-3} (3.7{e}{-5}})$ & $\mathbf{6.82{e}{-3} \, (1.2{e}{-4}})$ & $\mathbf{9.76{e}{-3} \, (3.0{e}{-4}})$ \\ 
\bottomrule 
\end{tabular}
\end{minipage}
\\ \vspace{0.25cm}

\begin{minipage}{1\textwidth}
\caption*{$T_\text{days}=$ 20 days.}

  \centering 
\begin{tabular}{lcccc} 
\toprule & MSE & MAE & $q^{\text{upper}}_{0.95}$ & CVaR$^{\text{upper}}_{0.95}$ \\ 
\midrule 
BS & $5.64\mathrm{e}{-5}$ & $4.83\mathrm{e}{-3}$ & $1.17\mathrm{e}{-2}$ & $1.73\mathrm{e}{-2}$  \\ 
Kernel hedging & $4.75\mathrm{e}{-5}$ (7.7$\mathrm{e}{-7}$) & $4.73\mathrm{e}{-3}$ (4.3$\mathrm{e}{-5}$) & $1.05\mathrm{e}{-2}$ (3.3$\mathrm{e}{-4}$) & $1.63\mathrm{e}{-2}$ (2.2$\mathrm{e}{-4}$)\\ 
Deep hedging & $4.45\mathrm{e}{-5}$ (2.7$\mathrm{e}{-6}$) & $4.69\mathrm{e}{-3}$ (1.4$\mathrm{e}{-4}$) & $1.01\mathrm{e}{-2}$ (4.7$\mathrm{e}{-4}$) & $1.45\mathrm{e}{-2}$ (1.0$\mathrm{e}{-3}$) \\ 
Deep kernel hedging & $\mathbf{3.82{e}{-5} \, (7.4{e}{-7}})$ & $\mathbf{4.29{e}{-3} \,  (3.1{e}{-5}})$ & $\mathbf{9.08{e}{-3} \, (2.1{e}{-4}})$ & $\mathbf{1.25{e}{-2} \, (4.4{e}{-4}})$ \\ 
\bottomrule \end{tabular}  
\end{minipage}
\\ \vspace{0.25cm}

\begin{minipage}{1\textwidth}
    \centering 
    \caption*{$T_\text{days}=$ 40 days.}

\begin{tabular}{lcccc} 
\toprule & MSE & MAE & $q^{\text{upper}}_{0.95}$ & CVaR$^{\text{upper}}_{0.95}$ \\ 
\midrule 
BS & ${9.75\mathrm{e}{-5}}$ & $\mathbf{6.23{e}{-3}}$ & $\mathbf{1.36{e}{-2}}$ & $\mathbf{1.76{e}{-2}}$  \\ 
Kernel hedging & $1.13\mathrm{e}{-4}$ (2.8$\mathrm{e}{-6}$) & $7.71\mathrm{e}{-3}$ (1.1$\mathrm{e}{-4}$) & $1.97\mathrm{e}{-2}$ (8.6$\mathrm{e}{-4}$) & $2.68\mathrm{e}{-2}$ (7.1$\mathrm{e}{-4}$)\\ 
Deep hedging & $1.01\mathrm{e}{-4}$ (5.5$\mathrm{e}{-6}$) & $7.21\mathrm{e}{-3}$ (2.2$\mathrm{e}{-4}$) & $1.82\mathrm{e}{-2}$ (7.5$\mathrm{e}{-4}$) & $2.49\mathrm{e}{-2}$ (1.3$\mathrm{e}{-3}$) \\ 
Deep kernel hedging & $\mathbf{8.93{e}{-5} \, (3.5{e}{-6}})$ & ${6.58\mathrm{e}{-3}  \, (1.1\mathrm{e}{-4}})$  & ${1.68\mathrm{e}{-2} \, (6.1\mathrm{e}{-4}})$ & ${2.35\mathrm{e}{-2} \, (1.2\mathrm{e}{-3}})$ \\ 
\bottomrule \end{tabular}
\end{minipage}
\end{table}

\paragraph{Asian option hedging.}
We consider the quadratic hedging of ATM arithmetic Asian call options of the form \eqref{eq: arithm_asian_payoff}. Based on real data, for each trajectory, we set the initial price of the Asian call used for hedging as
\[
v_0^i = \mathrm{AMM}(S_0=S_0^i,K=S_0^i,T,\sigma=\sqrt{\nu_0^i}),
\]
where $\nu_0^i$ denotes the initial realized variance of the $i$-th path and $\mathrm{AMM}(S,K,T,\sigma)$ is the closed-form moment-matching approximation of the arithmetic Asian call price given in \cite[Proposition 6]{lo2014moment}. Table \ref{tab:loss_SP500_asian} reports the results. Across all three maturities, the deep kernel hedging achieves strong performances, with performance gaps widening as maturity increases from 10 to 40 days. 
\begin{comment}
Across all three maturities, deep kernel hedging achieves the lowest MSE and MAE, with the performance gap widening as maturity increases from 10 to 40 days. 
Although tail-risk metrics are not explicitly optimized, deep kernel hedging generally provides the most favorable downside-risk P\&L profile across all maturities. 
Kernel hedging remains less accurate than the two deep-learning approaches. Overall, these results further highlight the effectiveness of deep kernel hedging, including for path-dependent payoffs.
\end{comment}

\begin{table}[!ht]
\small
\centering
\caption{Out-of-sample metrics for the quadratic hedging of an ATM Asian call on the real dataset with
time-augmented signature features. Reported values are averages over ten seeds, with standard deviations in parentheses.}
\label{tab:loss_SP500_asian}

\begin{minipage}{1\textwidth}
\centering
\caption*{$T_\text{days}=$ 10 days.}
\begin{tabular}{lcccc} 
\toprule 
& MSE & MAE & $q^{\text{upper}}_{0.95}$ & CVaR$^{\text{upper}}_{0.95}$\\ 
\midrule 
Kernel hedging & $1.81\mathrm{e}{-5}$ (5.1$\mathrm{e}{-7}$) & $2.84\mathrm{e}{-3}$ (3.2$\mathrm{e}{-5}$) & $5.94\mathrm{e}{-3}$ (1.9$\mathrm{e}{-4}$) & $1.02\mathrm{e}{-2}$ (2.1$\mathrm{e}{-4}$) \\ 
Deep hedging & $1.59\mathrm{e}{-5}$ (9.0$\mathrm{e}{-7}$) & $2.77\mathrm{e}{-3}$ (6.2$\mathrm{e}{-5}$) & $5.30\mathrm{e}{-3}$ (2.7$\mathrm{e}{-4}$) & $8.61\mathrm{e}{-3}$ (4.3$\mathrm{e}{-4}$) \\ 
Deep kernel hedging & $\mathbf{1.30{e}{-5} \, (5.4{e}{-7})}$ & $\mathbf{2.51{e}{-3} \, (2.5{e}{-5})}$ & $\mathbf{4.71{e}{-3} \, (6.2{e}{-5})}$ & $\mathbf{7.90{e}{-3} \, (2.1{e}{-4})}$\\ 
\bottomrule 
\end{tabular}

\end{minipage}
\\ \vspace{0.25cm}

\begin{minipage}{1\textwidth}
\caption*{$T_\text{days}=$ 20 days.}

  \centering 
\begin{tabular}{lcccc} 
\toprule 
& MSE & MAE & $q^{\text{upper}}_{0.95}$ & CVaR$^{\text{upper}}_{0.95}$\\ 
\midrule 
Kernel hedging & $2.31\mathrm{e}{-5}$ (6.5$\mathrm{e}{-7}$) & $3.34\mathrm{e}{-3}$ (5.1$\mathrm{e}{-5}$) & $6.82\mathrm{e}{-3}$ (2.3$\mathrm{e}{-4}$) & $1.06\mathrm{e}{-2}$ (2.7$\mathrm{e}{-4}$) \\ 
Deep hedging & $2.15\mathrm{e}{-5}$ (1.5$\mathrm{e}{-6}$) & $3.33\mathrm{e}{-3}$ (9.8$\mathrm{e}{-5}$) & $6.69\mathrm{e}{-3}$ (4.3$\mathrm{e}{-4}$) & $9.80\mathrm{e}{-3}$ (9.4$\mathrm{e}{-4}$) \\ 
Deep kernel hedging & $\mathbf{1.76{e}{-5} \, (5.5{e}{-7})}$ & $\mathbf{2.95{e}{-3} \, (3.1{e}{-5})}$ & $\mathbf{5.92{e}{-3} \, (1.2{e}{-4})}$ & $\mathbf{8.53{e}{-3} \, (3.8{e}{-4})}$ \\ 
\bottomrule 
\end{tabular}

\end{minipage}
\\ \vspace{0.25cm}

\begin{minipage}{1\textwidth}
    \centering 
    \caption*{$T_\text{days}=$ 40 days.}
\begin{tabular}{lcccc} 
\toprule 
& MSE & MAE & $q^{\text{upper}}_{0.95}$ & CVaR$^{\text{upper}}_{0.95}$ \\ 
\midrule 
Kernel hedging & $4.11\mathrm{e}{-5}$ (1.1$\mathrm{e}{-6}$) & $4.62\mathrm{e}{-3}$ (6.5$\mathrm{e}{-5}$) & $9.56\mathrm{e}{-3}$ (2.5$\mathrm{e}{-4}$) & $1.41\mathrm{e}{-2}$ (3.2$\mathrm{e}{-4}$) \\ 
Deep hedging & $3.38\mathrm{e}{-5}$ (2.8$\mathrm{e}{-6}$) & $4.07\mathrm{e}{-3}$ (7.6$\mathrm{e}{-5}$) & $\mathbf{7.58{e}{-3} (2.4{e}{-4})}$ & ${1.27\mathrm{e}{-2}}$ (1.6$\mathrm{e}{-3}$) \\ 
Deep kernel hedging & $\mathbf{2.96{e}{-5} \, (8.3{e}{-7})}$ & $\mathbf{3.95{e}{-3} \, (4.6{e}{-5})}$ & $7.72\mathrm{e}{-3} \, (1.6\mathrm{e}-4)$ & $\mathbf{1.13{e}{-2} \,  (3.7{e}{-4})}$ \\ 
\bottomrule 
\end{tabular}

\end{minipage}
\end{table}

{\small

\section*{Disclosure of interest}
No potential conflict of interest was reported by the authors. 

\section*{Funding}
\textbf{E. Motte} is  grateful for the financial support from the Fonds de la Recherche Scientifique (F.R.S. - FNRS) through a FRIA grant.

\section*{Data Availability Statement}
The S\&P500  historical data (January 1st, 2010 to May 31st, 2026) used in the numerical experiments are publicly available and can be accessed directly through the \texttt{yfinance} Python package. No proprietary or newly collected data were used. 

\section*{Declaration of generative AI use}
 Generative AI was used to improve the spelling, grammar and flow of the article. It was also used to generate the flow chart in Figure \ref{fig: DKH_architecture}.
 
 \section*{Author contributions}
CRediT: \textbf{J. Dupret:} Methodology, Software, Validation, Formal analysis, Writing; \textbf{D. Hainaut:} Conceptualization, Validation, Supervision; \textbf{E. Motte:} Conceptualization, Methodology, Software, Validation, Formal analysis, Writing.}

\bibliographystyle{plainnat}
\bibliography{bibl}

\end{document}